\documentclass[journal]{IEEEtran}  
\IEEEoverridecommandlockouts                              

\newif\ifanonymouss
\anonymoussfalse  

\usepackage{array}          
\usepackage{textcomp}       
\usepackage{upgreek}        
\usepackage{stfloats}       
\usepackage{url}            
\usepackage{verbatim}       
\usepackage[normalem]{ulem} 
\useunder{\uline}{\ul}{}
\usepackage[english]{babel} 

\usepackage[space]{cite}    

\makeatletter
\let\NAT@parse\undefined
\makeatother

\usepackage{hyperref}       
\hypersetup{
    colorlinks=true,
    linkcolor=red,
    filecolor=magenta,      
    urlcolor=blue,
    citecolor=blue,
}

\usepackage{graphicx}       

\usepackage[table,xcdraw,dvipsnames]{xcolor} 
\usepackage{colortbl}       

\usepackage{subcaption}     

\usepackage{tabularx}       
\usepackage{threeparttablex}
\usepackage{multirow}       
\usepackage{makecell}       
\usepackage{booktabs}       

\usepackage{amsmath}        
\usepackage{amsfonts}       
\usepackage{amssymb}        
\usepackage{amsthm}       
\newtheorem{lemma}{Lemma}
\theoremstyle{definition}

\theoremstyle{remark}

\usepackage[ruled,vlined,linesnumbered]{algorithm2e} 
\usepackage{algorithmic}    
\DeclareMathOperator*{\argmin}{argmin} 

\usepackage{CJKutf8}

\newcommand{\appref}[1]{the Supplementary Material (Appendix~\ref{#1})}

\newcommand{\mainref}[1]{the main paper (Section~\ref{#1})}

\newcommand{\maineqref}[1]{Eq.~\ref{#1} in the main paper}

\newcommand{\x}{\mathbf{x}}

\title{\LARGE \bf
CSMapping: Scalable Crowdsourced Semantic Mapping and Topology Inference for Autonomous Driving
}

\ifanonymouss
\else
\author{Zhijian Qiao$^*$, Zehuan Yu$^*$, Tong Li, Chih-Chung Chou, Wenchao Ding and Shaojie Shen$^{\dag}$            
\thanks{This work was supported in part by the HKUST Postgraduate Studentship, and in part by the HKUST-DJI Joint Innovation Laboratory.}
\thanks{$^*$ Zhijian Qiao and Zehuan Yu are co-first authors. Zhijian Qiao, Zehuan Yu, Tong Li, and Shaojie Shen are with the Department of Electronic and Computer Engineering, the Hong Kong University of Science and Technology, Hong Kong, China. Wenchao Ding is with the Academy for Engineering and Technology, Fudan University, Shanghai, China. Chih-Chung Chou is with Zhuoyu Technology Co., Ltd., Shenzhen, China.}
\thanks{$^{\dag}$Corresponding author: Shaojie Shen}
}
\fi

\begin{document}
\maketitle
\thispagestyle{empty}
\pagestyle{empty}

\begin{abstract}
    Crowdsourcing offers a scalable approach to autonomous driving map construction by aggregating perception outputs and trajectories from large fleets of production vehicles. However, inherent observation noise from low-cost sensors prevents map quality from automatically scaling with data volume. We present CSMapping, a scalable system that constructs accurate semantic maps and topological road centerlines whose quality improves as more crowdsourced data becomes available. For semantic mapping, we learn a generative map prior using a latent diffusion model trained on high-definition (HD) maps---optionally with standard-definition (SD) map conditioning---without requiring paired crowdsourced/HD-map supervision. This prior captures structural regularities in real-world maps, enabling robustness to extreme observation noise and plausible completion in unobserved regions. We incorporate this prior through a constrained maximum a posteriori (MAP) formulation in latent space, optimizing the map within a low-dimensional manifold that preserves the prior. To initialize the latent space, we develop a robust curve-mapping module that produces vectorized maps, then apply diffusion inversion to transfer them into latent space. For latent optimization, we introduce an efficient Gaussian-basis reparameterization with projected gradient search and a multi-start strategy using posterior scoring. We further implement latent-space factor-graph optimization for global map consistency. For topological mapping, we cluster vehicle trajectories using confidence-weighted k-medoids and apply kinematic refinement, yielding smooth, human-pattern road centerlines robust to diverse trajectories and dispersed endpoints. Experiments on nuScenes, Argoverse~2, and a large proprietary dataset demonstrate state-of-the-art performance on both semantic and topological mapping, with comprehensive ablations and scalability analyses along training and inference axes. Code will be released to support the research community.
\end{abstract}

\section{Introduction}
\label{sec:introduction}

High-definition (HD) maps~\cite{elghazaly2023high} are a core infrastructure for autonomous driving, providing lane-level semantics~\cite{qiao2023online,chen2025online} and topology~\cite{ding2023flowmap,qin2023traffic} that enable reliable localization and long-horizon planning. Current production pipelines rely heavily on professional survey fleets equipped with high-precision sensors, making large-scale deployment costly and slow. Crowdsourced mapping~\cite{guo2016low,qin2021light} aggregates perception outputs and trajectories from large numbers of production vehicles with low-cost sensors, offering a scalable alternative.

In this work, as shown in Fig.~\ref{fig:challenge}, we present CSMapping, a scalable crowdsourced mapping system that leverages perceived road semantics (e.g., lane dividers and boundaries) and human driving trajectories to construct semantic maps and topological centerlines, respectively. The key scalability property is that the performance of both subsystems improves and eventually saturates at a high-quality level as crowdsourced data increases, despite extremely noisy observations (Fig.~\ref{fig:challenge}\textcolor{red}{a}).

\begin{figure}[t]
    \centering
    \includegraphics[width=0.9\linewidth]{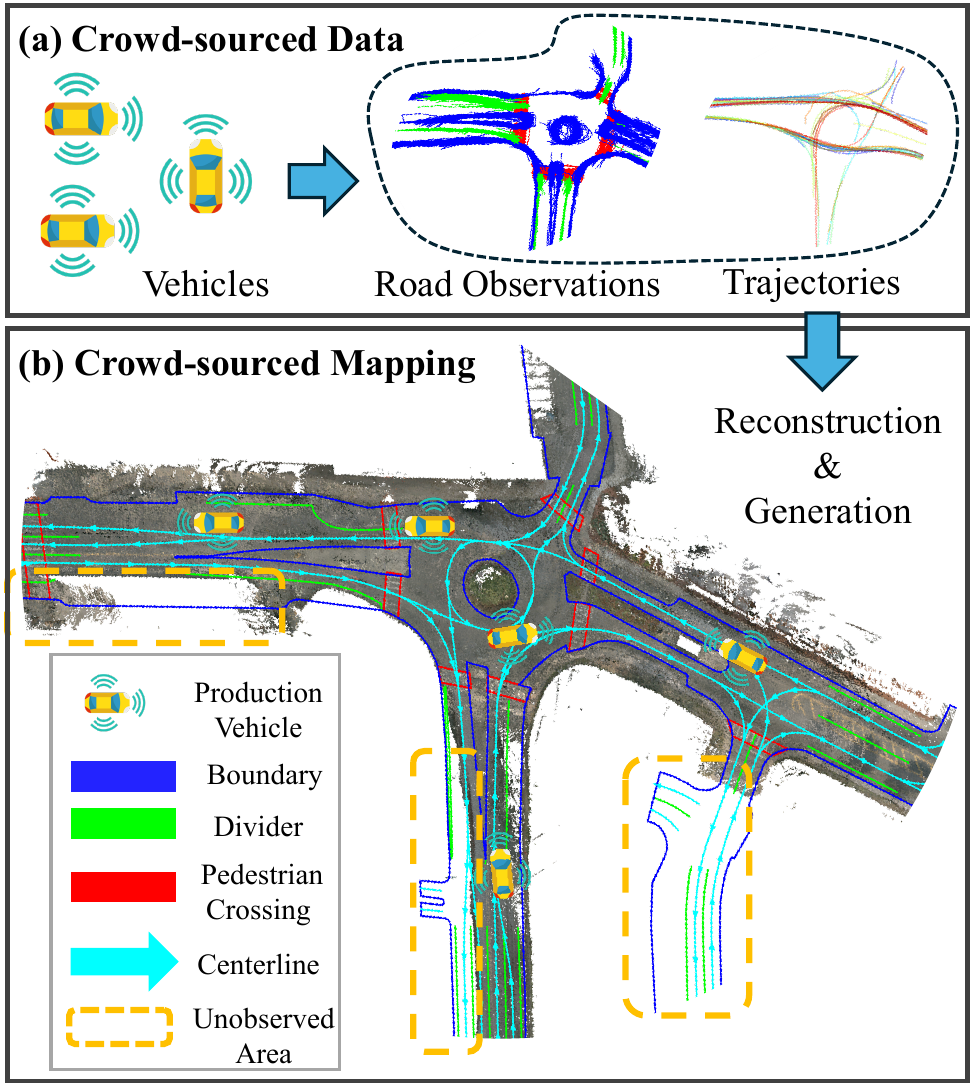}
    \caption{CSMapping leverages crowdsourced vectorized observations, including road semantic detections and vehicle trajectories collected at scale, to respectively produce semantic maps and topological centerlines. Despite noisy observations and incomplete coverage, the system robustly reconstructs observed regions and provides plausible generation of unobserved regions (\textcolor{orange}{orange dashed regions}). A top-down RGB basemap visualizes the coverage of crowdsourced observations.}
    \label{fig:challenge}
\end{figure}

For semantic mapping, while prior works achieve impressive quality via semantic grid mapping~\cite{qin2021light,qin2020avp,guo2016low} or point-cloud clustering followed by vectorization~\cite{kim2021hd,munoz2022robust,zhou2022visual}, we consider a more challenging setting with extremely high observation noise, where traditional robust estimators (e.g., robust cost functions~\cite{mactavish2015all} and Maximum Consensus~\cite{chin2017maximum}) for accurate mapping may fail. Moreover, the coverage of crowdsourced data may be incomplete, unlike professional survey fleets that are controlled to cover entire areas. This limited observability~\cite{yang2019observability} leads to incomplete maps. We attribute both inaccuracies and incompleteness to the lack of a strong map prior that captures structural regularities (element shapes, spatial layouts, connectivity) from large-scale HD map data. To address this, we employ a latent diffusion model~\cite{rombach2022high} to learn a mapping from a Gaussian distribution to the HD-map distribution, enabling any Gaussian latent to be transformed into a plausible, complete HD map via multiple differentiable denoising steps~\cite{ho2020denoising}. The training data contain only HD maps (optionally SD maps for conditional generation). We then formulate semantic mapping as a constrained maximum a posteriori (MAP) estimation in latent space. By constraining the latent to the Gaussian distribution, the estimated map preserves the prior. Intuitively, we convert the original problem—recovering the semantic map from noisy and incomplete observations—into finding a Gaussian latent whose generated HD map best agrees with the observations. To this end, we propose an efficient Gaussian-basis latent reparameterization to preserve Gaussian moment constraints during optimization, and apply projected-gradient optimization on the unit sphere to search for the optimal latent. This approach aligns with the manifold assumption that real map data lies on a low-dimensional manifold~\cite{li2025basicsletdenoisinggenerative}.

Another issue in latent-space optimization is latent initialization. Prior works~\cite{qi2024not,menon2020pulse,ma2025inference} often rely on random sampling, which can be inefficient and prone to poor local optima. We instead first construct a semantic map using classical mapping, and then apply diffusion inversion~\cite{song2020denoising,hong2024exact} to transfer the constructed map into latent space as initialization. To this end, a vectorized mapping system tailored to curve features is designed: it inputs curve observations from online perception~\cite{liao2024maptrv2,chen2024maptracker} and outputs a vectorized map, yielding higher accuracy than grid-based methods. Specifically, we introduce voxel-connectivity association (VCA) for curve observation–landmark association. Curve landmarks are represented with Chebyshev polynomials and parameterized by arc length. We use continuous dynamic time warping (CDTW) for order-preserving curve matching to estimate arc-length parameters, and employ Graduated Non-Convexity (GNC) to robustly estimate landmark states. Beyond providing initialization, we handle extreme cases where vectorized mapping fails (e.g., excessive noise or sparse observations) by adopting a multi-start strategy that samples multiple initializations around the base initialization, together with a posterior scoring method combining observation consistency and map plausibility to select the best estimation.

The proposed semantic mapping system scales not only with increasing crowdsourced data, but also along training and inference axes. During training, generative learning yields mapping performance gains with more data, more training steps, stronger conditioning, and larger model capacity. At inference, quality scales with the number of gradient-based optimization steps and the number of initial latents. Furthermore, we integrate latent-space optimization with factor-graph optimization to produce a globally consistent map.

For topological centerline mapping, prior works~\cite{ding2023flowmap,qin2023traffic} first group trajectories via endpoint clustering, then select representative trajectories in each group to obtain human-pattern centerlines. However, as trajectories accumulate, inter-driver variability causes endpoints to become broadly distributed rather than concentrated near lane centers, making endpoint clustering difficult and harming scalability. We therefore introduce a single-stage approach that clusters trajectories directly using confidence-weighted k-medoids~\cite{rdusseeun1987clustering}, avoiding endpoint clustering altogether. The k-medoids algorithm selects a representative trajectory (the medoid) per cluster and minimizes a sum of pairwise dissimilarities rather than the sum of squared Euclidean distances as in k-means, making it more robust to trajectory noise and outliers. Finally, Non-Maximum Suppression (NMS) and kinematic refinement remove redundant medoids and ensure the smoothness and feasibility of the resulting centerlines. Our contributions are summarized as follows.

\begin{itemize}
    \item We present a scalable crowdsourced system for semantic and topological centerline mapping. The quality of both subsystems improves and eventually saturates at a high-quality level as crowdsourced data increases.
    
    \item We propose a hybrid mapping system integrating classical and learning-based approaches. In learning-based mapping, we develop a latent-space optimization pipeline that integrates a learned generative prior to construct accurate and complete HD maps (we refer to this as generative mapping). In classical mapping, we build a customized vectorized mapping system for curve features, providing a high-quality initialization for latent-space optimization. Factor-graph optimization in latent space yields a globally consistent map. The hybrid system exhibits scalability in both training and inference.
    
    \item For topological mapping, we propose a single-stage approach that directly clusters trajectories using confidence-weighted k-medoids and kinematic refinement to produce human-pattern centerlines.
    
    \item Evaluation on nuScenes~\cite{caesar2020nuscenes}, Argoverse~2~\cite{wilson2023argoverse}, and a large proprietary dataset demonstrates state-of-the-art performance on both semantic and topological mapping tasks, with comprehensive ablations, scalability analyses, and application cases.
\end{itemize}

\section{Related Work}

\subsection{Crowdsourced Mapping}
\label{subsec:rw_crowdsourced_mapping}

Crowdsourced mapping shifts data collection from professional survey fleets to production vehicles, which upload compact semantics rather than raw sensor streams to reduce bandwidth. Typical uploads include semantic point clouds~\cite{zhang2021real,pannen2019lane,greve2024collaborative,pannen2020keep,qin2020avp}, vectorized 3D road markings~\cite{kim2021hd,qin2023traffic,zhou2022visual,chen2024mapcvv}, and 2D image-space contours~\cite{qin2023traffic,ding2023flowmap,ruhhammer2016automated}. On the cloud side, observations are aligned by transforming each vehicle's pose into a common frame using techniques from multi-robot and multi-session Simultaneous Localization and Mapping (SLAM)~\cite{xu2022omni,tian2022kimera,slim2025}. The system then constructs semantic (Section~\ref{related_work:semantic}) and topological (Section~\ref{related_work:topological}) maps.

\subsubsection{Semantic Mapping}
\label{related_work:semantic}
Cloud-side reconstruction typically follows two paradigms: vectorization and rasterization.

Vectorization methods~\cite{kim2021hd,munoz2022robust,chen2024mapcvv,zhou2022visual} adopt a feature-based SLAM-like scheme: they associate, across time and sessions, all observations that correspond to the same curve-landmark and then estimate the landmark state; unlike classical SLAM with point/line landmarks, the landmarks here are curves. For association, some~\cite{zhang2021real,doer2020hd} discretize curves into points and cluster them using Density-Based Spatial Clustering of Applications with Noise (DBSCAN), while filtering outliers; others~\cite{kim2021hd,pannen2019lane,qiao2023online,munoz2022robust} define curve-to-curve similarities to track curve instances. To handle observation uncertainty, shell-based similarities are employed~\cite{kim2021hd,pannen2019lane}; pose uncertainty is incorporated in Chamfer distance-based similarity~\cite{qiao2023online}; curvature-angle descriptors further improve robustness~\cite{munoz2022robust}. After association, clustered curves are fitted with B\'ezier curves~\cite{kim2021hd}, cubic polynomials~\cite{zhou2022visual}, B-splines~\cite{chen2024mapcvv}, or Catmull--Rom splines~\cite{qiao2023online}.

Rasterization methods~\cite{qin2021light,qin2023traffic} maintain probabilistic semantic grids for flexible, real-time updates. Beyond foreground classes emphasized by vectorization methods, they explicitly model background (e.g., road surface) to suppress outliers~\cite{qin2023traffic,qin2021light}; learned uncertainty further improves grid estimation~\cite{xie2023mv}. Learned feature grids~\cite{you2022hindsight,yuan2024presight,zhu2023nemo,xiong2023neural} encode richer per-cell embeddings at the cost of extra storage.

Both representations are usable for downstream planning~\cite{jiang2023vad,hu2023planning}. Vectorized mapping offers high precision but can be brittle under high outlier ratios and complex geometries (e.g., U-shapes), largely due to association and parameterization errors. Rasterized mapping is more tolerant to noise but discretizes the scene into independent cells, weakening global layout cues. These trade-offs motivate our approach: starting from a precise vectorized map initialization (Section~\ref{sec:vec_mapping}), we perform latent-space optimization (Section~\ref{sec:diff_mapping}) under a learned generative prior to produce an accurate and plausible semantic grid map.

\subsubsection{Topological Mapping}
\label{related_work:topological}
Topological mapping estimates drivable centerlines that provide long-horizon structure for navigation and planning, especially in complex intersections where field-of-view limits make online perception models~\cite{wu2023topomlp,fu2024topologic} less reliable. Crowdsourced human driving trajectories are a natural signal: they encode connectivity and common driving patterns without additional labeling. Autograph~\cite{zurn2023autograph,buchner2023learning} learns successor graphs from satellite imagery and trajectories, aggregates them into lane graphs, and relies on post-processing to extract centerlines. Alternatively, Zhou et al.~\cite{zhou2021automatic} infer entry/exit pairs from vehicle flows and then generate B\'ezier trajectories. Methods such as \cite{qin2023traffic,ding2023flowmap,ruhhammer2016automated} cluster trajectories given known entry/exit points to obtain human-like centerlines. We follow the clustering route but avoid explicit entry/exit detection to reduce brittleness, directly clustering raw trajectories and then refining with kinematic constraints.

\subsection{Prior Construction in Mapping}
\label{related_work:priors}

\subsubsection{Hand-crafted Priors}
Hand-crafted priors encode designer knowledge about scene structure. Structured SLAM leverages Manhattan/Atlanta/tilted-world assumptions~\cite{coughlan1999manhattan,schindler2004atlanta,straub2017manhattan,li2023hong}, adding angular constraints between lines and principal directions to stabilize pose estimation. When scenes consist of parameterizable primitives—lines~\cite{xu2025airslam,wen2022roadside,schaefer2019long}, planes~\cite{bao2021utilization}, cylinders~\cite{zhou2022plc}—feature-level parameterization reduces reliance on point correspondences and improves accuracy~\cite{liu2023efficient} and long-term maintainability~\cite{slim2025}. Regularization also appears in NeRF~\cite{mildenhall2021nerf}: RegNeRF~\cite{niemeyer2022regnerf} imposes depth smoothness, and SimpleNeRF~\cite{somraj2023simplenerf} limits high-frequency position encoding to avoid spurious discontinuities. These priors are effective but limited in expressiveness and can fail when assumptions are violated.

\subsubsection{Generative Priors}
Generative modeling learns priors from large-scale data in an unsupervised fashion; in conditional variants, it can incorporate weak or easily obtainable supervision (e.g., text~\cite{rombach2022high} or edge maps~\cite{zhang2023adding}), thereby overcoming the limitations of hand-crafted assumptions.

In image generation, Stable Diffusion~\cite{rombach2022high} and its distilled variants SD-Turbo~\cite{sauer2024adversarial} learn strong priors from large-scale text--image data. Difix3D++~\cite{zhangjie2025difix3d+} fine-tunes SD-Turbo~\cite{sauer2024adversarial} with reference images to reduce artifacts for novel-view synthesis, then uses the deartifacted images to fine-tune NeRF~\cite{mildenhall2021nerf} or 3DGS~\cite{kerbl20233d}. Similarly, 3DGS-Enhancer~\cite{liu20243dgs} leverages Stable Video Diffusion (SVD)~\cite{blattmann2023stable} to interpolate frames, improving 3DGS. Beyond pure synthesis, generators can act as evaluators~\cite{gu2023nerfdiff,warburg2023nerfbusters,wu2024reconfusion,poole2022dreamfusion}, integrating diffusion priors into loss functions to favor plausible reconstructions.

For crowdsourced mapping, generic SVD~\cite{blattmann2023stable}/Stable Diffusion~\cite{rombach2022high} priors are not directly applicable. We therefore train a specialized HD map diffusion model, adapting Stable Diffusion-style training to HD/SD map datasets rather than Internet images, yielding a prior tailored to road semantics and layouts.

\subsection{Latent Optimization Methods}

While Gaussian noise can be used to generate images~\cite{rombach2022high} or music~\cite{ben2024d} by diffusion models~\cite{ho2020denoising}, not all noise seeds are equally effective. Recent work explores selecting seeds and metrics for quality evaluation. Qi et al.~\cite{qi2024not} randomly enumerate seeds and select by inversion stability; Zhou et al.~\cite{zhou2024golden} train a network to predict ``golden'' noise directly. Zero-order search over denoising paths has been investigated~\cite{ma2025inference}, and pre-trained verifiers such as DINO~\cite{oquab2023dinov2} and CLIP~\cite{radford2021learning} guide selection. These efforts fall under inference-time scaling for diffusion models.

Because the denoising process is differentiable, first-order optimization is feasible and often more efficient. CodeSLAM~\cite{bloesch2018codeslam} and follow-ups~\cite{matsuki2021codemapping,zuo2021codevio} encode dense geometry into VAE codes and jointly optimize codes and poses without constraining the code distribution. To constrain the latent to the Gaussian distribution, projected gradient descent on the unit sphere $S^{d-1}$, where $d$ is the dimension of the Gaussian latent space, can be used to enforce Gaussian geometry in high dimensions~\cite{Papaspiliopoulos02102020}, and has been applied in audio and image synthesis~\cite{ben2024d,menon2020pulse}. We similarly use projected gradients but introduce a Gaussian-basis reparameterization to preserve first and second moments and reduce dimensionality, coupled with a multi-start strategy to mitigate sensitivity to initialization.
\section{Semantic Mapping Overview and Problem Formulation}

\begin{figure}[t]
    \centering
    \includegraphics[width=0.8\linewidth]{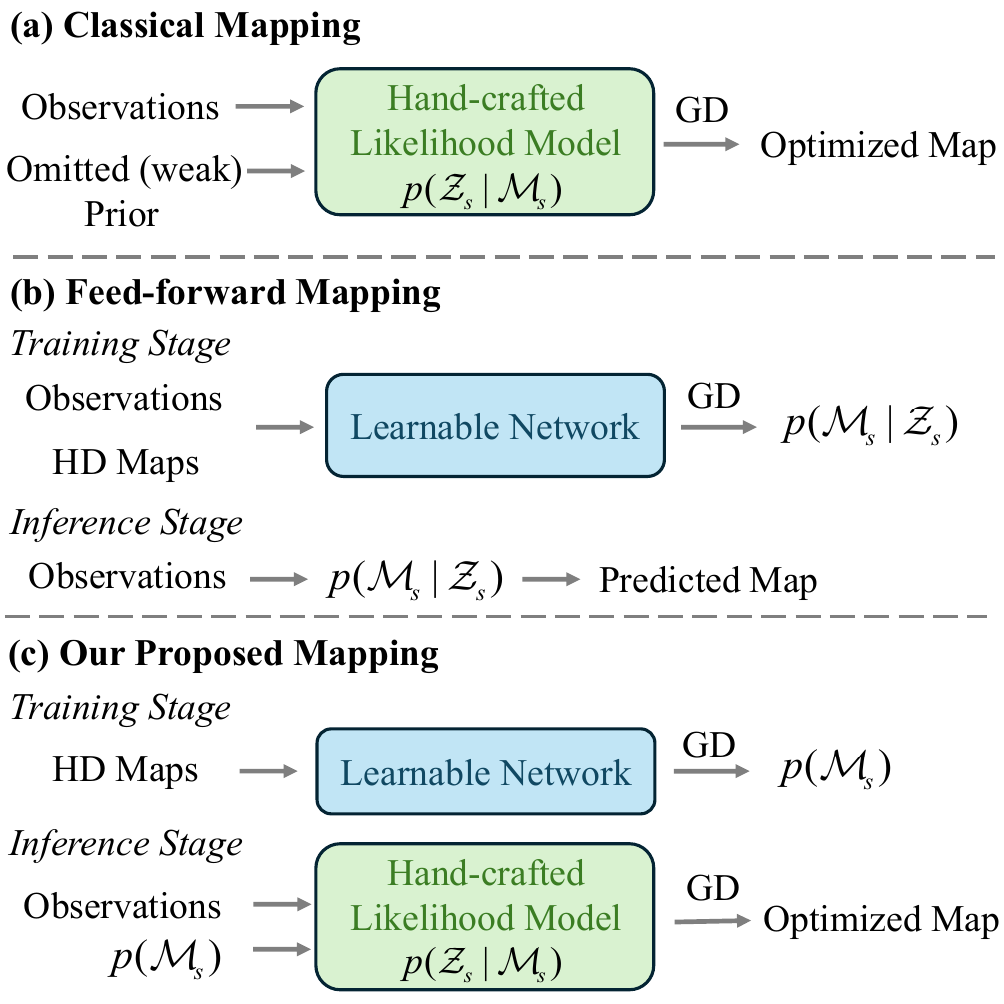}
    \caption{Comparison of mapping paradigms. (a) Classical mapping employs gradient descent (GD) optimization with hand-crafted likelihoods for accurate results, though it relies on weak or absent priors. (b) Feed-forward mapping learns the posterior directly, enabling fast inference but demanding paired data supervision. (c) Our approach learns a prior on HD maps without requiring paired supervision and conducts scene-specific MAP estimation to construct accurate and complete maps.}
    \label{fig:pipe_comp}
  \end{figure}

We consider semantic mapping for a fixed-size scene (e.g., 100 m × 100 m). Crowdsourced vehicles upload vectorized observations $\mathcal{Z}_s$ from online perception models~\cite{liao2024maptrv2,chen2024maptracker}, containing polyline coordinates and $C$ semantic categories. The target semantic map $\mathcal{M}_s \in \{0,1\}^{H \times W \times C}$ is a binary multi-channel raster representing road elements. Unlike conventional formulations that implicitly assume full coverage, our fixed-size map is only partially observed in practice due to limited routes, viewpoints, and traversals; thus, some regions may remain unobserved.

\subsubsection{MAP Formulation in Latent Space}

We formulate semantic mapping as a MAP estimation:
\begin{equation}
\label{eq:map_formulation}
\begin{aligned}
\mathcal{M}_s^* &= \arg\max_{\mathcal{M}_s} p(\mathcal{M}_s|\mathcal{Z}_s) \\
&= \arg\max_{\mathcal{M}_s} p(\mathcal{Z}_s|\mathcal{M}_s)p(\mathcal{M}_s)
\end{aligned}
\end{equation}
where $p(\mathcal{Z}_s|\mathcal{M}_s)$ is the observation likelihood and $p(\mathcal{M}_s)$ is the map prior.

Classical mapping methods typically omit the prior term $p(\mathcal{M}_s)$ because valid map distributions are difficult to hand-craft—reducing Eq.~\ref{eq:map_formulation} to maximum likelihood estimation. We instead use Denoising Diffusion Probabilistic Models \cite{ho2020denoising} (DDPM) to explicitly learn the prior: we train an HD map generator $G$ that learns a mapping from standard Gaussian noise $\mathbf{x}_T \sim \mathcal{N}(0,I)$ to plausible and complete maps through $G := D \circ g_\theta: \mathbb{R}^d \to \{0,1\}^{H \times W \times C}$, where $T$ is the timestep in the diffusion process, $g_\theta$ denotes the denoising process that produces a clean latent $\mathbf{x}_0$, and $D$ decodes it to a map. By constraining $\mathbf{x}_T$ within a Gaussian latent space, we enforce that generated maps follow the learned prior distribution. Eq.~\ref{eq:map_formulation} transforms to latent space optimization in a constrained MAP formulation:
\begin{equation}
\label{eq:latent_opt}
\begin{aligned}
\mathbf{x}_T^* &= \arg\max_{\mathbf{x}_T \sim \mathcal{N}(0,I)} p(\mathcal{Z}_s|G(\mathbf{x}_T)) \\
&= \arg\min_{\mathbf{x}_T} -\log p(\mathcal{Z}_s|G(\mathbf{x}_T)) \quad \text{s.t.} \quad \mathbf{x}_T \sim \mathcal{N}(0,I)
\end{aligned}
\end{equation}

To solve this optimization problem, vectorized mapping (Section~\ref{sec:vec_mapping}) first constructs an initial semantic map from noisy crowdsourced vectorized observations, which is then transformed into latent space via diffusion inversion (Section~\ref{subsec:latent_init}) for initialization. Subsequently, Section~\ref{sec:diff_mapping} details the latent-space optimization process to refine the initial map into an accurate and complete semantic map.


\subsubsection{Relation to Existing Paradigms}
\label{subsec:paradigms_comp}

Fig.~\ref{fig:pipe_comp} contrasts classical and feed-forward mapping paradigms with ours. Classical mapping employs hand-crafted likelihoods \( p(\mathcal{Z}_s|\mathcal{M}_s) \) with weak or omitted priors to recover maps from observations, often suffering from increasing noise and incomplete coverage. When noise is moderate, however, mapping quality can be improved through iterative gradient descent optimization. The feed-forward paradigm~\cite{wang2025vggt,keetha2025mapanything} learns the posterior \( p(\mathcal{M}_s|\mathcal{Z}_s) \) from paired data and predicts maps in a single inference step. It offers fast inference but requires expensive paired crowdsourced-HD map data for training and lacks the ability to iteratively refine maps for higher accuracy. Our method essentially enhances the classical paradigm by incorporating a learned prior on HD maps via diffusion models, eliminating the need for paired training data. Crowdsourced data is utilized during inference to identify the HD map in the learned prior distribution that best matches the observations.

\section{Vectorized Mapping}
\label{sec:vec_mapping}

Vectorized mapping targets curve-based observations and produces an initial semantic map from noisy, crowdsourced vector inputs. This map provides the latent initialization in Eq.~\ref{eq:latent_opt}. Section~\ref{subsec:data_association} describes observation–landmark association; Section~\ref{subsec:curve_param} presents the representation and parameterization; Section~\ref{subsec:robust_estimation} estimates landmark states robustly.

\subsection{Observation-Landmark Association}
\label{subsec:data_association}

Given vectorized observations (2D polylines), association assigns observed curve segments to landmark instances. Unlike points or lines, a curve observation is not atomic, which raises three practical issues: (1) Confidence varies along a curve due to viewpoint, occlusion, and distance; an observation can be partly correct rather than a pure inlier or outlier. Hence, keeping or discarding the entire curve is inappropriate (see the circled regions in Fig.~\ref{fig:data_association}\textcolor{red}{c}). (2) When landmarks are spatially close, a single observation may overlap multiple landmarks, precluding one-to-one assignment (Fig.~\ref{fig:data_association}\textcolor{red}{c}). (3) Curve landmarks are typically long; most observations cover only fragments (Fig.~\ref{fig:data_association}\textcolor{red}{b}), making curve-level similarity measures~\cite{qiao2023online,munoz2022robust} unreliable for association.

We therefore propose a voxel-connectivity association (VCA) method for curve features. We maintain a hashed voxel observation map with $(C+1)$ channels (foreground classes plus background) that accumulates counts over time. For each frame, we rasterize~\cite{bresenham1977} detected foreground polylines to increment per-class counts along traversed voxels, and increment background counts within observed regions that lack foreground detections. We then binarize each voxel: a foreground class is assigned if its count exceeds $\alpha$ times the background count ($\alpha<1$ favors recall). Connected components of same-class voxels define landmark instances. For each instance, we clip each curve to the voxels of that instance and retain only the in-region segments as observations. 

\begin{figure}[tbp]
    \centering
    \includegraphics[width=0.8\linewidth]{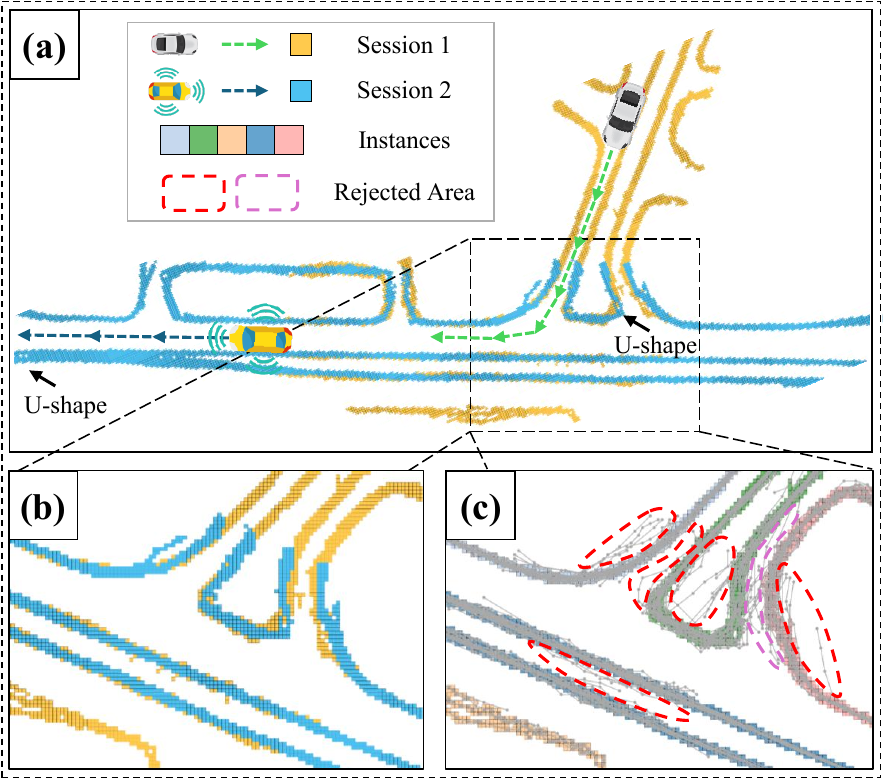}
    \caption{Data association for vectorized observations. (a) Two local semantic maps from different sessions (\textcolor{cyan}{cyan} and \textcolor{orange}{orange}). (b) Limited overlap between observations from different sessions. (c) Instance segmentation with different colors; rejected regions (\textcolor{red}{red} and \textcolor{violet}{violet} dashed circles represent issue 1 and 2 in Section \ref{subsec:data_association}, respectively) indicate portions of vectorized observations to be excluded.}
    \label{fig:data_association}
\end{figure}

\subsection{Landmark Representation and Parameterization}
\label{subsec:curve_param}

We represent 2D landmark curves using an arc-length–parameterized Chebyshev expansion. The landmark state is denoted as $\mathbf{a} = [a_0, \ldots, a_n]$, where $a_i \in \mathbb{R}^2$ are the Chebyshev coefficients and $n$ (the expansion order) is chosen adaptively based on the number of associated voxels. For a point $p_j$ with parameter $t_j$ (arc length affinely normalized to $[-1,1]$), we compute
\begin{equation}
        p_j = \sum_{i=0}^n a_i \cos(i \arccos(t_j)) \quad \text{for} \quad t_j \in [-1, 1]
\end{equation}
Once the paired observations $(p_j, t_j)$ are available for all curves, the coefficients $\mathbf{a}$ are obtained via linear least squares. To obtain the parameters $t_j$, we initialize from one observation and propagate a unified parameter space to others through curve-to-curve matching. Arc-length parameterization requires an order-preserving curve matching method to ensure parameter space consistency, especially for complex shapes (e.g., U-turns in Figs.~\ref{fig:data_association}\textcolor{red}{a} and~\ref{fig:cdtw}), where nearest-neighbor-based matching tends to produce ambiguities. Thus, we employ continuous dynamic time warping (CDTW)~\cite{brankovic2020} to enforce order preservation.

CDTW seeks a monotone matching path $\pi(\cdot)$ that minimizes the cumulative cost:
\begin{equation}
\label{eq:cdtw}
\pi^* = \argmin_{\pi} \int_0^1 \| c_r(\pi(t)) - c_o(\pi(t)) \|^2_2 \cdot \|\pi'(t)\|_1 \, dt
\end{equation}
where $\pi(0) = (0,0)$ and $\pi(1) = (L_r, L_o)$, given two curves $c_r$ and $c_o$ with arc lengths $L_r$ and $L_o$. Intuitively, CDTW computes a shortest path on the discretized cost field from the bottom-left to the top-right corner (Fig.~\ref{fig:cdtw}\textcolor{red}{b}), using only nondecreasing moves (right, up, or right-up on the grid). It is solved via an improved shortest-path algorithm~\cite{brankovic2020}. Additional details about the CDTW algorithm and non-overlapping portion recognition are provided in \appref{app:cdtw}.

Finally, we build a graph whose nodes are observation curves and whose edges encode truncated Chamfer distances. We extract a minimum spanning tree and perform a breadth-first traversal, starting from the highest-degree node as the reference, to propagate arc-length parameters along the tree (Fig.~\ref{fig:cdtw}\textcolor{red}{a}) based on the CDTW matching.


\begin{figure}[t]
    \centering
    \includegraphics[width=0.8\linewidth]{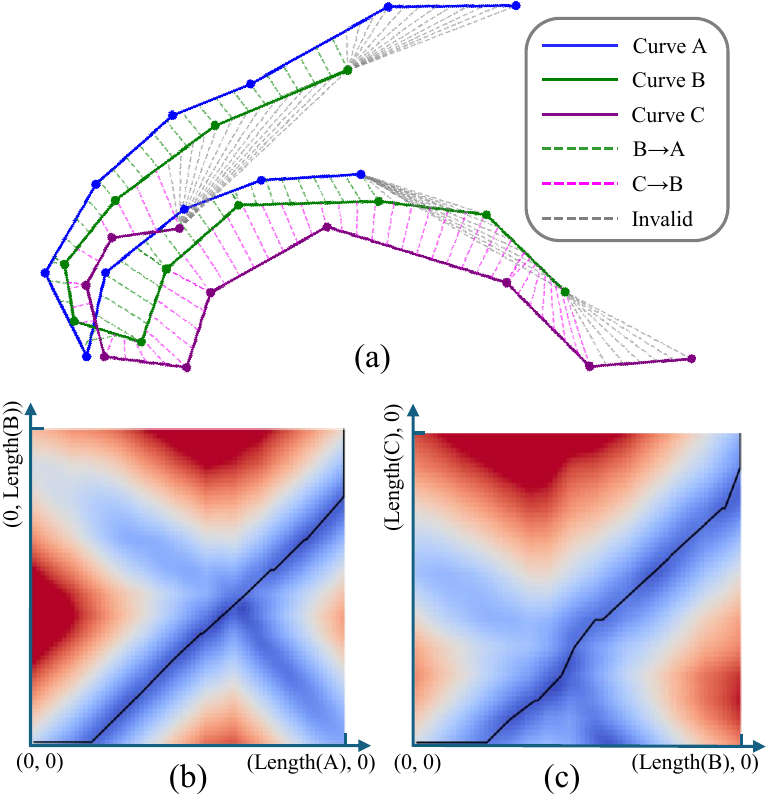}
    \caption{(a) Continuous curve matching via CDTW, where the C$\rightarrow$A correspondence is obtained through sequential C$\rightarrow$B and B$\rightarrow$A propagation. Non-overlapping portions are marked as invalid matches (gray dashed lines). (b) Cost fields for B$\rightarrow$A and C$\rightarrow$B matching, with black lines denoting optimal warping paths. The cost is defined as the Euclidean distance between corresponding points on the two curves.}
    \label{fig:cdtw}
\end{figure}

\subsection{Landmark State Estimation}
\label{subsec:robust_estimation}

After parameterizing all observation curves, we estimate the Chebyshev coefficients $\mathbf{a}$ by solving the following truncated least-squares (TLS) problem:
\begin{equation}
\label{eq:tls_objective}
\mathbf{a}^* = \argmin_{\mathbf{a}} \sum_{j=1}^N \min \left( \|p_j - \sum_{i=0}^n a_i \cos(i \arccos(t_j))\|^2_2, c^2 \right)
\end{equation}
where $c^2$ is the truncation threshold and $N$ is the number of point–parameter pairs. This non-convex objective is solved using graduated non-convexity (GNC)~\cite{yang2020graduated}; see \appref{app:gnc} for details.

\section{Generative Mapping}
\label{sec:diff_mapping}

\begin{figure}[t]
  \centering
  \begin{subfigure}[t]{\linewidth}
    \includegraphics[width=\linewidth]{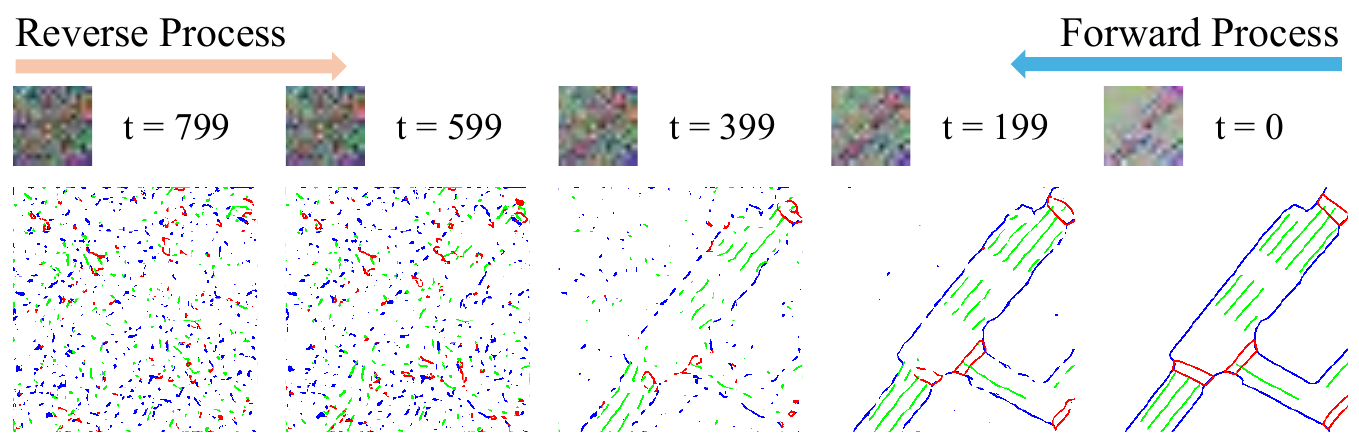}
    \caption{}\label{fig:denoising_ddpm_example}
  \end{subfigure}
  \hfill
  \begin{subfigure}[t]{\linewidth}
    \includegraphics[width=\linewidth]{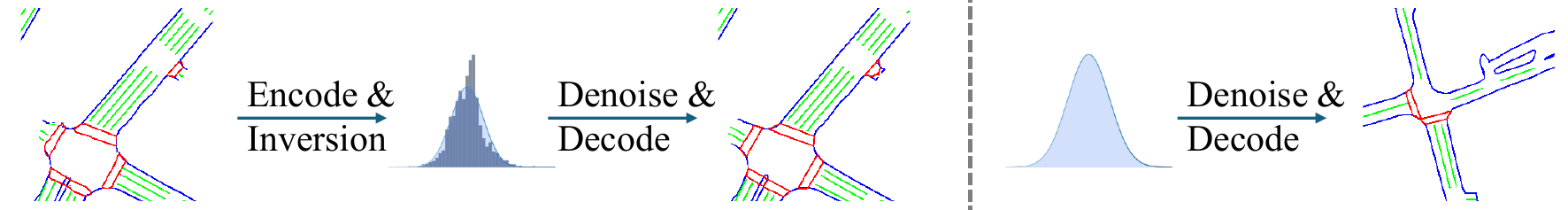}
    \caption{}\label{fig:denoising_dpm_inversion}
  \end{subfigure}
  \caption{Illustration of latent diffusion and inversion. (a) Latent diffusion and denoising process with decoded maps along the trajectory. (b) From left to right: initial map from vectorized mapping, inverted latent from diffusion inversion and its generated map, and a random Gaussian latent with its generated map.}\label{fig:denoising}
\end{figure}

After obtaining the initial semantic map from vectorized mapping (Section~\ref{sec:vec_mapping}), this section presents our approach to solving the constrained MAP problem in latent space (Eq.~\ref{eq:latent_opt}). Section~\ref{subsec:diffusion} introduces the map latent diffusion model and its training. Section~\ref{subsec:latent_init} describes the initialization of the latent variable via diffusion inversion. Section~\ref{subsec:latent_reparam} proposes a Gaussian-basis reparameterization and optimization strategy. Section~\ref{subsec:obs_likelihood} defines the observation likelihood as a masked squared error. Section~\ref{subsec:multistart} introduces a multi-start strategy for more robust optimization. Section~\ref{subsec:global_consistency} extends to global multi-submap scenarios via factor graph optimization.

\subsection{Map Latent Diffusion Model}
\label{subsec:diffusion}
$G(\cdot)$ in Eq.~\ref{eq:latent_opt} denotes a generative process that maps a noise latent $\mathbf{x}_T$ to a map $\mathcal{M}_s$ based on Denoising Diffusion Probabilistic Models (DDPM) \cite{ho2020denoising}. Following latent diffusion practice in image generation \cite{rombach2022high}, the map $\mathcal{M}_s \in \{0,1\}^{H \times W \times C}$ is first encoded by a Variational Autoencoder (VAE) \cite{kingma2013auto} into a compact latent $\mathbf{x}_0 = E(\mathcal{M}_s) \in \mathbb{R}^{H' \times W' \times C'}$, and we denote $d = H' W' C'$ for brevity. In the forward diffusion process, we gradually add Gaussian noise $\epsilon \sim \mathcal{N}(0,I)$ to obtain the noised latent $\mathbf{x}_t$ over $T$ timesteps, where $\mathbf{x}_t = \sqrt{\bar{\alpha}_t} \, \mathbf{x}_0 + \sqrt{1-\bar{\alpha}_t} \, \epsilon$ (see \appref{app:forward_process} for details); when $T$ is large, $\mathbf{x}_T$ approaches a standard Gaussian. In the reverse process, $\mathbf{x}_T$ is iteratively denoised to $\mathbf{x}_0$ using DPM-Solver \cite{lu2022dpm++} with a noise-prediction network $\epsilon_\theta$. DDIM \cite{song2020denoising} is provided as an alternative and a comparison is reported in Section~\ref{subsec:ab_scheduler}. In this work, we omit the stochastic term in DPM-Solver, yielding a deterministic mapping from $\mathbf{x}_T$ to $\mathbf{x}_0$. Fig.~\ref{fig:denoising_ddpm_example} visualizes the trajectories of $\mathbf{x}_t$ and the corresponding decoded maps.

\begin{figure}[t]
    \centering
    \includegraphics[width=\linewidth]{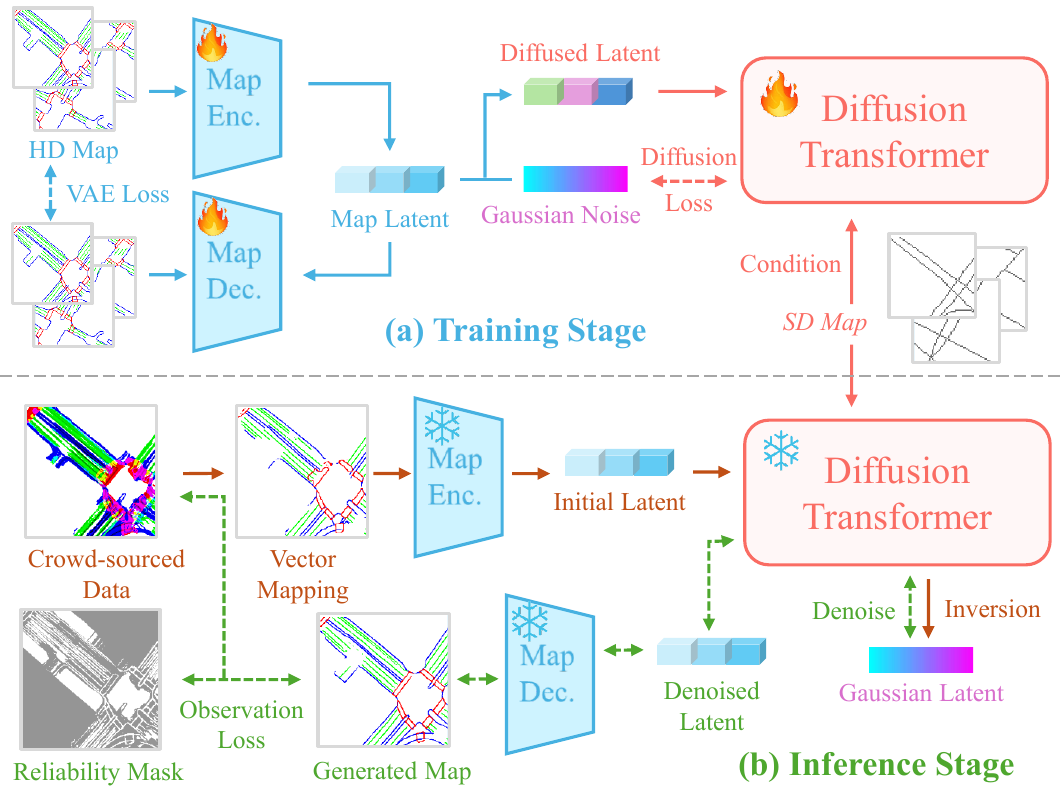}
    \caption{Generative mapping pipeline. (a) Training the HD-map latent diffusion model: a VAE encodes HD maps into latent space and decodes them back to reconstruct the map (\textcolor{CornflowerBlue}{blue} paths), while a diffusion transformer predicts noise at each timestep, optionally conditioned on an SD map (\textcolor{Salmon}{pink} paths). (b) \textcolor{RawSienna}{Brown} paths illustrate latent initialization from vectorized mapping via diffusion inversion; \textcolor{OliveGreen}{green} paths show gradients backpropagated through the denoising network to the noise latent $\mathbf{x}_T$ using a masked loss.}
    \label{fig:generative_pipe}
\end{figure}

The training pipeline is shown in Fig.~\ref{fig:generative_pipe}\textcolor{red}{a} and comprises two parts: (1) VAE training and (2) denoising network training. The encoder $E$ and decoder $D$ of the VAE are implemented as convolutional networks and are trained with a Kullback--Leibler (KL) divergence term and a reconstruction term on HD map samples $\mathcal{M}_s \in \{0,1\}^{H \times W \times C}$. The denoising network $\epsilon_\theta$ is a skip-connected DiT \cite{bao2023all} that takes the latent $\mathbf{x}_t$, timestep $t$, and the SD map $c_{\text{sd}}$ (optional) as inputs to predict the noise. The timestep $t$ is injected via adaptive layer normalization (adaLN~\cite{peebles2023}). The SD map $c_{\text{sd}}$ encodes a rasterized skeleton of road-level geometry, which has the same spatial resolution as the HD map; for every pixel belonging to the SD skeleton, all semantic channels are set to $1$, and $0$ otherwise. The SD map is encoded by $E$ and injected into the DiT through zero-initialized layers at each DiT block, following ControlNet \cite{zhang2023adding}. The complete architecture is illustrated in \appref{app:diffusion_architecture}. Given the latent $\mathbf{x}_0$ and its corresponding SD map $c_{\text{sd}}$, we sample $\epsilon \sim \mathcal{N}(0,I)$ and a timestep $t$ uniformly from $\{1,\dots,T\}$ to train the denoising network. The training objective maximizes the log-likelihood of the data, $\log p_\theta(\mathbf{x}_0)$, and simplifies to minimizing the mean-squared error between the true and predicted noise as follows (see \appref{app:ddpm_loss_derivation}):
\begin{equation}
\label{eq:masked_loss}
    \mathcal{L} = \mathbb{E}_{\mathbf{x}_0, \, \epsilon \sim \mathcal{N}(0,I), \, t \sim \mathcal{U}\{1,\dots,T\}} \left\| \epsilon - \epsilon_\theta(\mathbf{x}_t, t, c_{\text{sd}}) \right\|_2^2 
\end{equation}

During training, the SD map is randomly replaced with a zero map with 20\% probability to enable unconditional generation. Given the learned $\epsilon_\theta$, $\mathbf{x}_T$ is gradually denoised to $\mathbf{x}_0$ through DPM-Solver \cite{lu2022dpm++}, which defines the mapping $\mathbf{x}_0 = g_\theta(\mathbf{x}_T)$ (we omit $c_{\text{sd}}$ for brevity). Problem~\ref{eq:latent_opt} becomes:
\begin{equation}
    \label{eq:latent_opt_3}
    \begin{aligned}
        &\arg\min_{\mathbf{x}_T} -\log p(\mathcal{Z}_s|D(g_\theta(\mathbf{x}_T))) \\
        &\text{ subject to } \mathbf{x}_T \sim \mathcal{N}(0,I)
    \end{aligned}
\end{equation}

\subsection{Latent Initialization}
\label{subsec:latent_init}

To solve Eq.~\ref{eq:latent_opt_3}, directly searching in latent space from a random initialization is challenging. We therefore initialize from the vectorized mapping result (Section~\ref{sec:vec_mapping}) and then perform diffusion inversion, as shown in Fig.~\ref{fig:generative_pipe}\textcolor{red}{b}. Although the vectorized map may be noisy and incomplete, it remains close to the ground truth and thus provides a strong initialization. Specifically, we rasterize the vectorized map to obtain $\mathcal{M}_s^{\text{init}}$ and encode it as $\mathbf{x}_0^{\text{init}} = E(\mathcal{M}_s^{\text{init}})$. Inversion reverses the generative direction by tracing back along the denoising trajectory to the noise latent. Ideally, $\mathbf{x}_0$ can be recovered by first inverting to $\mathbf{x}_T$ and then denoising back to $\mathbf{x}_0$ using the same denoising network. For the implementation of DPM-Solver++ and DDIM as well as their inversion, we used the diffusers library~\cite{diffusers}. Fig.~\ref{fig:denoising_dpm_inversion} demonstrates the importance of inversion: starting from the inverted latent $\mathbf{x}_T^{\text{init}}$ reconstructs the initial map closely, whereas a random Gaussian latent yields a significantly different result. The quantitative comparison can be found in Section~\ref{subsec:ab_initialization}.

\begin{figure*}[htbp]
    \centering
    \includegraphics[width=\linewidth]{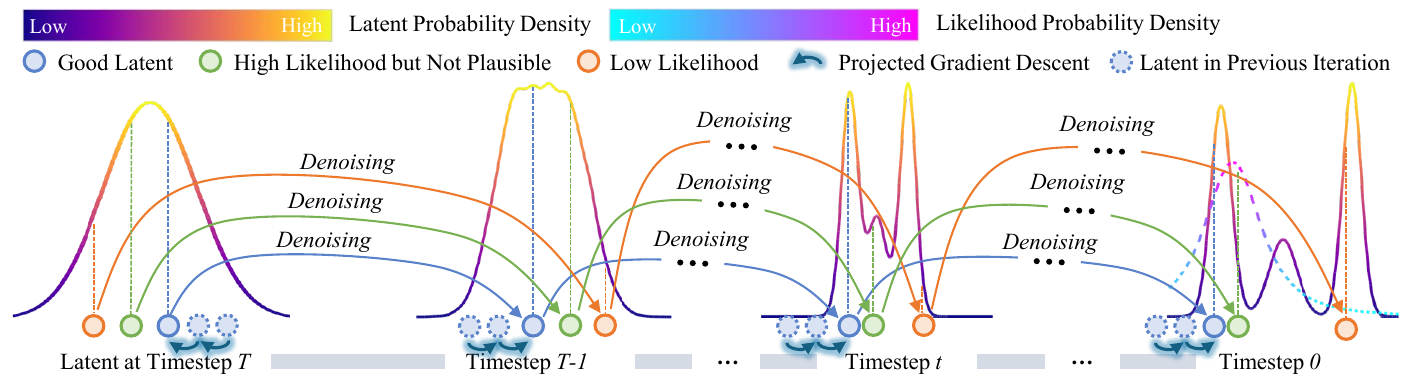}
    \caption{Latent-space optimization. \textcolor[RGB]{90, 130, 203}{Blue}, \textcolor[RGB]{236, 121, 44}{orange}, and \textcolor[RGB]{120, 177, 81}{green} balls denote three typical initial latents. Initial latents ($t=T$) are optimized by projected gradient descent to maximize the observation likelihood while preserving the Gaussian prior. Latent probability density over timesteps is shown in solid color-changing curves; the likelihood density at $t=0$ is shown in a dashed color-changing curve. The three latent candidates are evaluated at timestep $t=0$: the orange latent has a high prior score but low likelihood; the blue and green latents both achieve high observation likelihood, but the blue latent lies in a region of higher prior probability and is therefore more likely to produce a plausible map.}
    \label{fig:latent_opt}
\end{figure*}

\subsection{Latent Reparameterization and Optimization}
\label{subsec:latent_reparam}

Directly optimizing $\mathbf{x}_T$ from the inverted latent $\mathbf{x}_T^{\text{init}}$ while maintaining the Gaussian constraint in Eq.~\ref{eq:latent_opt_3} is challenging. Therefore, we propose a Gaussian-basis reparameterization and an accompanying optimization strategy to explore the Gaussian latent space. Specifically, we reparameterize $\mathbf{x}_T$ as a linear combination of orthogonal Gaussian basis vectors:
\begin{equation}
\label{eq:basis_decomp}
    \mathbf{x}_T = \sum_{i=0}^{K} w_i \mathbf{b}_i, \quad \text{s.t.} \quad \|\mathbf{w}\|_2 = 1
\end{equation}
where $\mathbf{b}_0 = \mathbf{x}_T^{\text{init}}$ serving as the primary anchor, and $\{\mathbf{b}_i\}_{i=1}^K$ are pre-computed orthogonal vectors sampled from a Gaussian distribution, providing diverse directions for exploration around the initial anchor $\mathbf{b}_0$. Specifically, these basis vectors are generated by QR decomposition of a Gaussian matrix $\mathbf{A} \sim \mathcal{N}(\mathbf{0}_{d \times d}, \mathbf{I}_{d \times d})$:
\begin{equation}
\label{eq:basis_generation}
[\mathbf{b}_1, \dots, \mathbf{b}_K] = \sqrt{d} \, \mathbf{Q}[:,1:K] \quad \text{where } \mathbf{A} = \mathbf{QR}
\end{equation}

The key insight is that the linear combination of orthogonal Gaussian basis vectors approximately preserves the first and second moments of the distribution when the weight vector lies on the unit sphere $\|\mathbf{w}\|_2 = 1$. Although $\mathbf{x}_T^{\text{init}}$ may deviate from a perfect Gaussian sample due to vectorized mapping errors and inversion imperfections, with sufficiently large $K$, the resulting $\mathbf{x}_T$ remains close to $\mathcal{N}(0,I)$.

With this reparameterization, our final optimization problem becomes:
\begin{equation}
    \label{eq:latent_opt_4}
    \begin{aligned}
        &\arg\min_{\mathbf{w}} -\log p(\mathcal{Z}_s|D(g_\theta(\sum_{i=0}^{K} w_i \mathbf{b}_i))) \\
        &\text{ subject to } \|\mathbf{w}\|_2 = 1
    \end{aligned}
\end{equation}

Given the initial guess $\mathbf{x}_T^{\text{init}}$, $\mathbf{w}$ can be initialized as $\mathbf{w} = [1, 0, \dots, 0]$ and we solve the spherical optimization problem by projected gradient descent. The dimension of the optimization problem is reduced from $d$ to $K$, with $K \ll d$ in practice. The proposed Gaussian-basis reparameterization outperforms other parameterization methods in both reconstruction quality in observed areas and plausible generation in unobserved areas, which is validated in Section~\ref{subsec:ab_opt_state}. The influence of $K$ and the Gaussian bases $\mathbf{b}_{1}, \ldots, \mathbf{b}_{K}$ is analyzed in Section~\ref{subsec:ab_num_basis}.

\subsection{Observation Likelihood}
\label{subsec:obs_likelihood}

To realize the constrained MAP in Eq.~\ref{eq:latent_opt_4}, we define the objective as a masked $\ell_2$ loss. Let $\mathcal{M}_s = D(g_\theta(\sum_{i=0}^{K} w_i \mathbf{b}_i)) \in \{0,1\}^{H \times W \times C}$ denote the generated map from the reparameterized latent. We construct the target map $\mathcal{M}_{tgt} \in \{0,1\}^{H \times W \times C}$ and a reliability mask $\mathcal{M}_{mask} \in \{0,1\}^{H \times W \times C}$ from observations.

To construct $\mathcal{M}_{tgt}$ and $\mathcal{M}_{mask}$, we maintain an observation-count tensor $\mathbf{N} \in \mathbb{N}^{H \times W \times (C+1)}$, where the first $C$ channels store foreground class counts and the last channel stores the total observation count $N_p$. This requires only $O(HWC)$ memory, instead of $O(NHWC)$ for storing $N$ individual observations. We normalize the first $C$ channels by $N_p$ to obtain a probability map $\mathcal{M}_{prob} \in [0,1]^{H \times W \times C}$. As described in \appref{app:obs_aggregation}, computing the $\ell_2$ loss between $\mathcal{M}_s$ and $\mathcal{M}_{prob}$ is equivalent (up to a constant) to summing over all individual observations treated as inliers.

However, observations vary in reliability, so we only use high-confidence ones to construct the observation likelihood. We first binarize $\mathcal{M}_{prob}$ with a threshold $\tau_{\text{fg}}$ to obtain the target map $\mathcal{M}_{tgt}$, and then build a reliability mask $\mathcal{M}_{mask}$ to keep only trustworthy foreground and background supervision. The problem~\ref{eq:latent_opt_4} becomes:
\begin{equation}
    \label{eq:latent_opt_5}
    \begin{aligned}
        &\arg\min_{\mathbf{w}} \mathcal{L}_{\text{mask}}(\mathcal{M}_s, \mathcal{M}_{tgt}, \mathcal{M}_{mask}) \\
        &\quad := \sum_{p \in \Omega} \sum_{c=1}^C m_{mask}^c(p) \cdot \left\| \mathcal{M}_s^c(p) - m_{tgt}^c(p) \right\|^2 \\
        &\text{ subject to } \|\mathbf{w}\|_2 = 1
    \end{aligned}
\end{equation}
where $\Omega$ is the pixel domain; $m_{mask}^c(p)$ is true if $m_{tgt}^c(p)$ is true, or the pixel is observed but only as background with sufficient total observation count (such as the 75th percentile of $N_p$), and false otherwise. Once the observation likelihood is defined, the gradient can be backpropagated through the denoising network to the noise latent $\mathbf{x}_T$, as shown in Fig.~\ref{fig:generative_pipe}\textcolor{red}{b}.

\subsection{Multi-Start Strategy and Posterior-Based Scoring}
\label{subsec:multistart}

While projected gradient descent is efficient, it can be sensitive to the initial guess $\mathbf{x}_T^{\text{init}}$, especially when the observations $\mathcal{Z}_s$ are extremely sparse or noisy (beyond the capability of vectorized mapping). To mitigate this, we introduce a multi-start strategy. We generate $N_s$ starting points by perturbing the initial latent with Gaussian noise (Fig.~\ref{fig:latent_opt}). For each perturbed starting point $\mathbf{x}_T^{\text{init}, i}$, we execute the full spherical optimization, yielding a set of $N_s$ candidate maps $\{\mathcal{M}_{s,i}\}$ with corresponding denoised latents $\{\mathbf{x}_{0,i}\}$.

To evaluate the quality of each $\mathcal{M}_{s,i}$, we adopt a heuristic posterior-based scoring. The likelihood term is computed as the Intersection over Union (IoU) between the candidate map $\mathcal{M}_{s,i}$ and $\mathcal{M}_{tgt}$ under the reliability mask $\mathcal{M}_{mask}$. Although $\mathcal{M}_{mask}$ indicates the high-confidence pixels, some of them may still be misclassified. Meanwhile, $\mathcal{M}_{mask}$ is sparse, so relying only on the likelihood term may produce candidates that achieve high likelihood yet are implausible. We therefore introduce a prior term $p(\mathcal{M}_s) \propto p(\mathbf{x}_0)$ to measure plausibility. As described in \appref{app:ddpm_loss_derivation}, the DDPM training objective maximizes the data log-likelihood $\log p_\theta(\mathbf{x}_0)$ and reduces to minimizing the noise-prediction loss. This motivates using the cosine similarity between the predicted noise and the true noise to approximate the prior term, i.e.,   

\begin{equation}
\label{eq:prior_score}
s_{\text{prior}}(\mathbf{x}_0) = \frac{1}{N_t} \sum_{j=1}^{N_t} \mathbb{E}_{\epsilon \sim \mathcal{N}(0,I)} \left[ \frac{\epsilon \cdot \epsilon_\theta(\mathbf{x}_{t_j}, t_j, c_{\text{sd}})}{\|\epsilon\| \|\epsilon_\theta(\mathbf{x}_{t_j}, t_j, c_{\text{sd}})\|} \right],
\end{equation}
where $\mathbf{x}_{t_j} = \sqrt{\bar{\alpha}_{t_j}} \mathbf{x}_0 + \sqrt{1-\bar{\alpha}_{t_j}} \epsilon$ and $t_j = j \cdot \frac{T}{N_t}$. In implementation, we use $N_t=5$ uniformly spaced timesteps and add random noise to each $\{\mathbf{x}_{0,i}\}$ to compute the prior score.

Both terms are naturally bounded in $[0,1]$ for straightforward combination. We first discard candidates with low likelihood to ensure basic consistency with the crowdsourced observations, and then select the best candidate using a weighted sum of the two terms. The advantage of the proposed posterior score over a pure likelihood score is validated in Section~\ref{subsec:scalability_inference}.

\subsection{Global Consistency via Factor Graph Optimization}
\label{subsec:global_consistency}

\begin{figure}[t]
    \centering

    \begin{subfigure}[b]{\linewidth}
        \centering
        \includegraphics[width=0.8\linewidth]{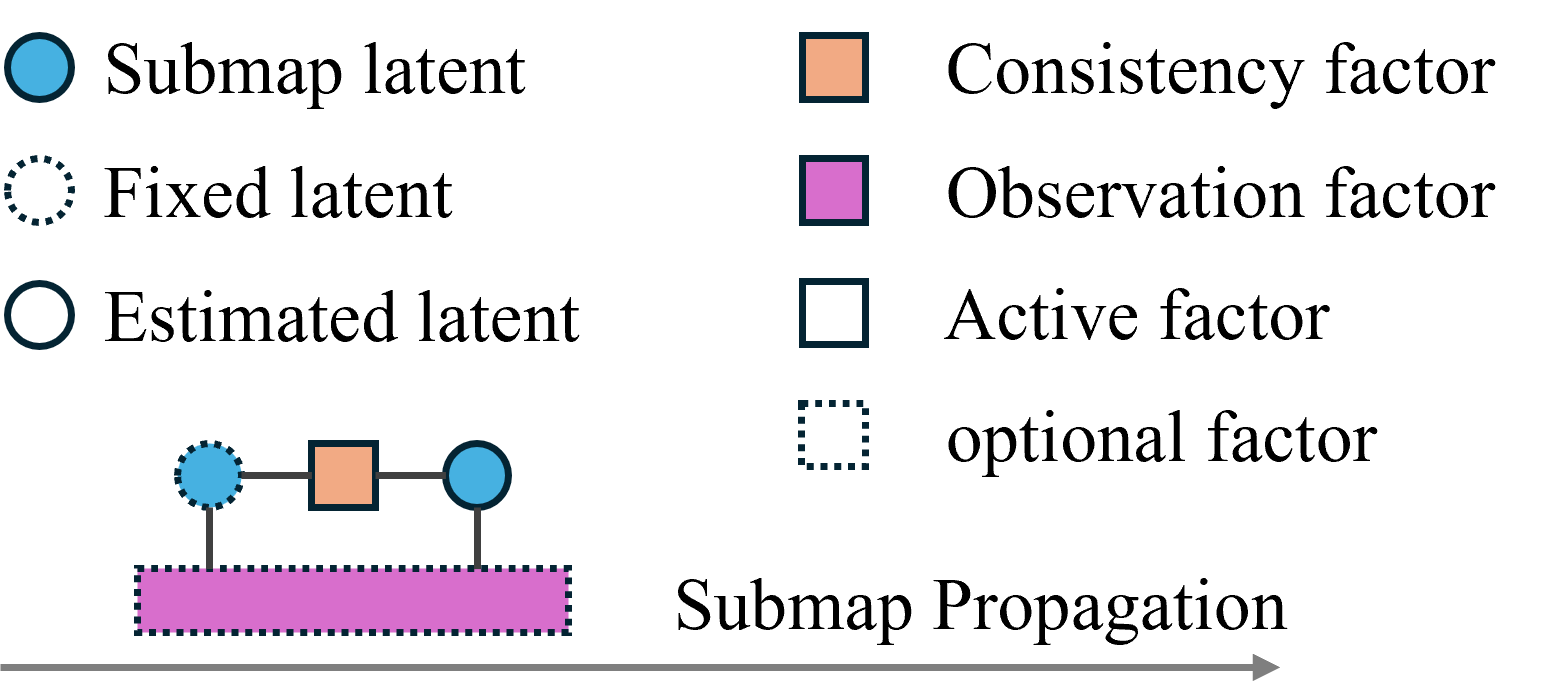}
        \caption{}
        \label{fig:submap_propagation}
    \end{subfigure}

    \vspace{0.5em}

    \begin{subfigure}[b]{\linewidth}
        \centering
        \includegraphics[width=0.8\linewidth]{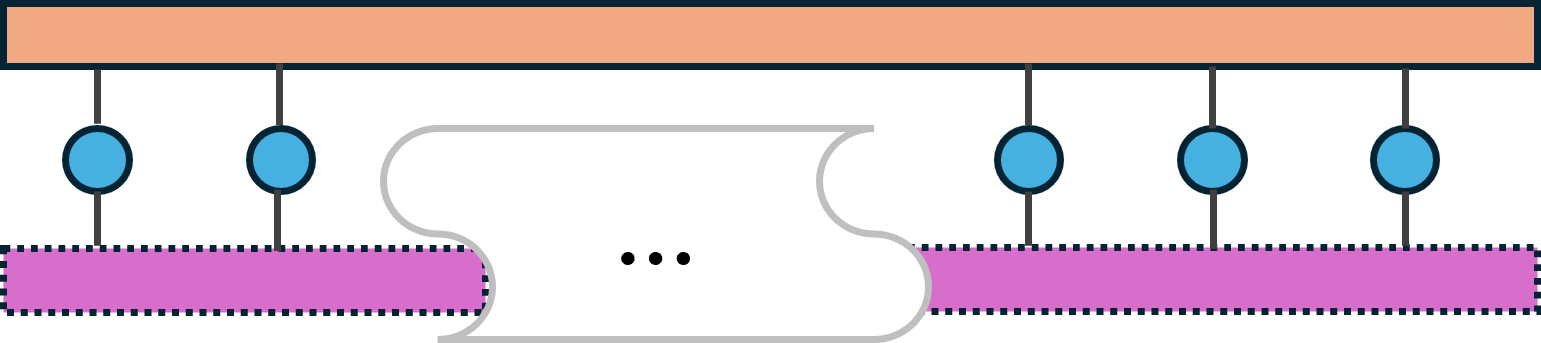}
        \caption{}
        \label{fig:global_graph}
    \end{subfigure}

    \caption{(a) Sequential submap propagation: the new latent is initialized by enforcing consistency with the previous submap. (b) Joint factor-graph optimization: all overlapping submap pairs are connected via consistency factors, yielding a globally coherent map. In both stages, observation factors are applied only when observations are available for the corresponding submap.}
    \label{fig:global_consistency}
\end{figure}

While the previous sections focus on optimizing individual fixed-size submaps, constructing globally consistent maps across large-scale scenes requires merging multiple overlapping submaps. We achieve this through factor graph optimization, which operates in two stages: sequential propagation for initialization (Fig.~\ref{fig:submap_propagation}) followed by joint refinement (Fig.~\ref{fig:global_graph}).

\subsubsection{Sequential Submap Propagation}
Given an optimized submap latent $\mathbf{x}_{\text{prev}}$ ($T$ is omitted for simplicity), we optimize the next submap latent $\mathbf{x}_{\text{next}}$ along a predefined path (e.g., following vehicle trajectories) using the consistency and observation factors, both defined as masked $\ell_2$ losses. The consistency factor is:
\begin{equation}
\label{eq:consistency_factor}
r_{\text{cons}}(\mathbf{x}_{\text{prev}}, \mathbf{x}_{\text{next}}) = \mathcal{L}_{\text{mask}}(\mathcal{M}_{\text{next}}, \mathcal{M}_{\text{prev}}^{\prime}, \mathcal{M}_{mask,\text{prev}})
\end{equation}
where $\mathcal{M}_{\text{prev}}^{\prime}$ is the warped map of $\mathcal{M}_{\text{prev}}$ by the relative pose $T_{\text{prev,next}}$, and $\mathcal{M}_{mask,\text{prev}}$ is a binary mask indicating the overlapping region between the two submaps.

The observation factor is defined as:
\begin{equation}
\label{eq:observation_factor}
r_{\text{obs}}(\mathbf{x}_{\text{next}}) = \mathcal{L}_{\text{mask}}(\mathcal{M}_{\text{next}}, \mathcal{M}_{tgt,\text{next}}, \mathcal{M}_{mask,\text{next}})
\end{equation}
where $\mathcal{M}_{tgt,\text{next}}$ and $\mathcal{M}_{mask,\text{next}}$ are the binarized target and reliability mask constructed from observations $\mathcal{Z}_{s,\text{next}}$ as described in Section~\ref{subsec:obs_likelihood}. The observation factor is enforced only when observations are available for the submap; otherwise the next submap is ``imagined'' subject to overlap consistency with the previous submap.

\subsubsection{Joint Factor Graph Refinement}
To handle overlap between multiple submaps and loop closures, we refine the entire map by jointly optimizing all latents in a factor graph (Fig.~\ref{fig:global_graph}). The graph connects submaps via consistency factors for all overlapping pairs, while observation factors ground submaps with available data:
\begin{equation}
\label{eq:factor_graph}
\{\mathbf{x}^*\} = \arg\min_{\{\mathbf{x}\}} \left[ \sum_{(i,j) \in \mathcal{E}} r_{\text{cons}}(\mathbf{x}_i, \mathbf{x}_j) + \sum_{t \in \mathcal{O}} r_{\text{obs}}(\mathbf{x}_t) \right]
\end{equation}
where $\mathcal{E}$ denotes the set of overlapping submap pairs and $\mathcal{O}$ denotes submaps with observations. See Section~\ref{exp:global_map} for global map construction validation.

\section{Trajectory-based Topological Mapping}
\label{sec:topology}

\begin{figure}[t]
    \centering
    \includegraphics[width=0.95\linewidth]{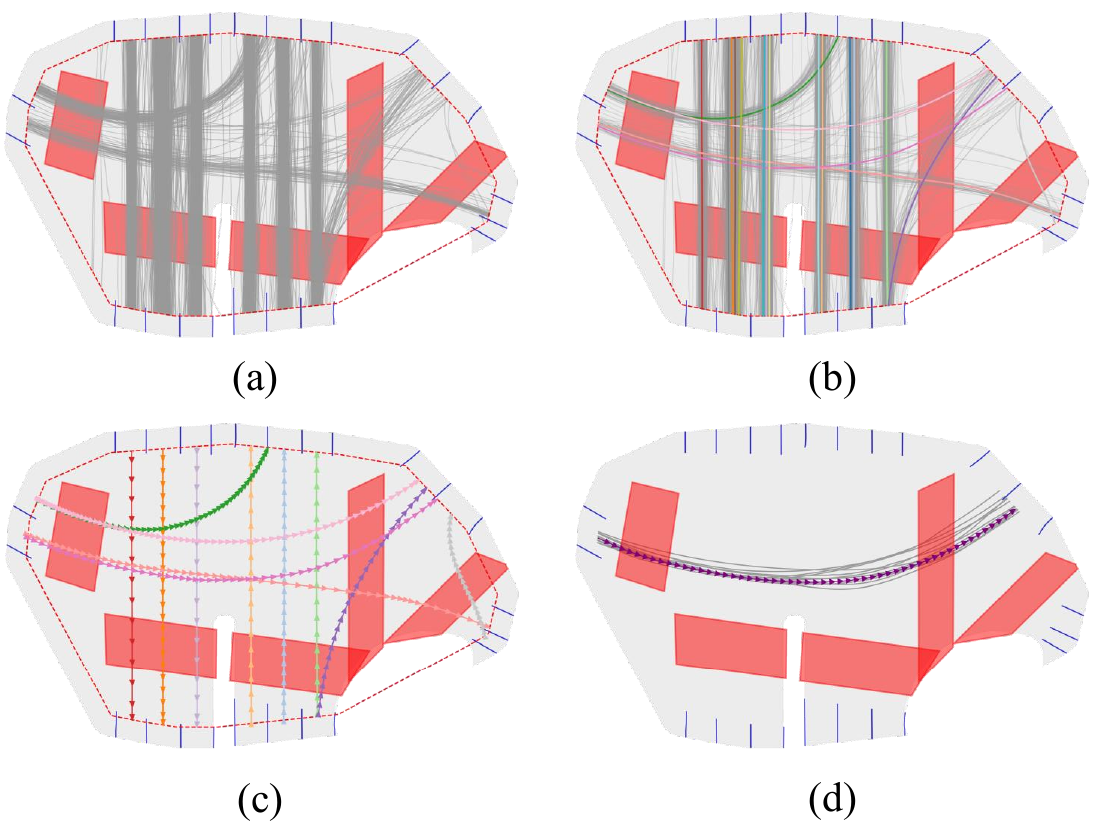}
    \caption{Topological mapping pipeline. (a) All raw trajectories (\textcolor[RGB]{169, 168, 170}{gray}) passing through the ROI (\textcolor[RGB]{220, 20, 60}{red} dashed box). (b) Clustered representative trajectories (indicated by different colors). (c) Representative trajectories after NMS. (d) Kinematically-optimized representative trajectory (medoid, \textcolor[RGB]{133,18,131}{purple}) and its associated non-medoid samples in a single cluster (\textcolor[RGB]{169, 168, 170}{gray}).}\label{fig:topo_pipeline}
\end{figure} 

Topological centerline maps \(\mathcal{M}_t\) provide long-horizon navigation structures and planning references, particularly in complex intersections. Human driving trajectories naturally capture connectivity and common driving patterns. We cluster raw trajectories to obtain representative ones (Section~\ref{subsec:traj_cluster}), then refine each medoid through kinematically constrained optimization to ensure smoothness and dynamic feasibility (Section~\ref{subsec:kinematic}), yielding drivable topological centerlines \(\mathcal{M}_t\).

\subsection{Trajectory Clustering}
\label{subsec:traj_cluster}

Given raw vehicle trajectories \(\mathcal{T} = \{\tau_i\}_{i=1}^N\) in a region of interest (ROI), as illustrated in Fig.~\ref{fig:topo_pipeline}\textcolor{red}{a}, we perform confidence-weighted k-medoids clustering to extract topological centerlines \(\mathcal{M}_t\) that summarize human-pattern driving routes. The optimization problem is formulated as:
\begin{equation}
\label{eq:topo_formulation}
    \mathcal{C}^* = \arg\min_{\mathcal{C} \subset \mathcal{T}} \sum_{i=1}^k \sum_{\tau_j \in C_i} w(\tau_j, c_i) \cdot d_\tau(\tau_j, c_i)
\end{equation}
where \(\mathcal{C} = \{c_1, \dots, c_k\} \subset \mathcal{T}\) is the set of \(k\) medoid trajectories, \(C_i\) denotes the cluster assigned to medoid \(c_i\), \(d_\tau\) measures trajectory distance (we use CDTW by default; Fréchet distance~\cite{buchin2019approximating} performs similarly under moderate noise), and \(w(\tau_j, c_i)\) reflects pairwise trajectory confidence. Unlike k-means, which averages trajectories to form centroids and may distort driving patterns, k-medoids selects actual trajectories as representatives. Additionally, it minimizes absolute distances rather than squared distances, making it more robust to outliers.

The pairwise confidence \(w(\tau_i, c_j)\) is computed as the harmonic mean of individual confidences \(w(\tau_i)\) and \(w(c_j)\), emphasizing reliable pairs:
\begin{equation}
\label{eq:dist_mat}
w(\tau_i, c_j) = \frac{2 w(\tau_i) w(c_j)}{w(\tau_i) + w(c_j)}.
\end{equation}

Each trajectory's confidence \(w(\tau_i)\) combines three weighted scores: \(\lambda_{\text{obs}} s_{\text{obs}} + \lambda_{\text{head}} s_{\text{head}} + \lambda_{\text{smooth}} s_{\text{smooth}}\), where \(\lambda\) are tunable weights. Specifically, \(s_{\text{obs}}\) is the proportion of observable waypoints across vehicles; \(s_{\text{head}}\) is the proportion of points with heading changes below \(\tau_{\theta}\) to suppress jitter; and \(s_{\text{smooth}} = \exp(-\bar{r}/\sigma_r)\) rewards geometric coherence, with \(\bar{r}\) as the average Chebyshev fitting error (see Section~\ref{subsec:curve_param}) and \(\sigma_r\) as a tunable parameter.

Given $k$, we solve Eq.~\ref{eq:topo_formulation} using FasterPAM~\cite{schubert2020fastpam} to obtain the medoid set $\mathcal{C}^*$. The dominant computational cost lies in constructing the distance and confidence matrices. In practice, when a new trajectory arrives, we update these matrices by appending its corresponding row and column rather than reconstructing the entire matrices. To avoid manually selecting $k$, we initialize with a larger value and then apply Non-Maximum Suppression (NMS), removing any medoid whose $d_{\tau}(\tau_i,\tau_j)$ to the medoid of a larger cluster falls below the threshold $\tau_{\text{nms}}$. The representatives after NMS are shown in Fig.~\ref{fig:topo_pipeline}\textcolor{red}{c}.

\subsection{Kinematic-Constrained Refinement}
\label{subsec:kinematic}

Medoids may retain observational artifacts. Fig.~\ref{fig:topo_pipeline}\textcolor{red}{d} illustrates kinematically constrained refinement for one cluster, producing an optimized medoid and associated samples. To obtain planning-ready paths similar to their cluster's trajectories while satisfying kinematic constraints, we refine each medoid $\hat{\tau}_k$ via optimization. To avoid ambiguity, we define the medoid as the collection of states $\hat{\tau}_k = \{x_i\}_{i=0}^n$ propagated by associated controls $\gamma =\{u_i\}_{i=0}^{n-1}$ with a time step $dt$:
\begin{align}
\label{eq:bspline_opt}
\hat{\tau}_k^* &= \argmin_{\gamma} \sum_{x \in \tau_k} \sum_{\tau_j \in C_k}  D_{\tau_j}(x) + \lambda_{\text{ctrl}} \sum_{u \in \gamma} \|u\|_2^2 \, dt  \nonumber \\
\text{s. t.  } & x_{i+1}= f(x_i, u_i) \quad \forall i\in \{0, \dots, n-1\}, \nonumber \\
& h_c(u) \leq 0 \quad \forall u\in \gamma, \quad h_s(x) \leq 0 \quad \forall x\in \hat{\tau}_k, 
\end{align}
where $C_k$ is the cluster of medoid $\hat{\tau}_k$, $D_{\tau_j}$ is the Euclidean distance field defined on $\tau_j$ and $\lambda_{\text{ctrl}}$ is the control regularization factor. The control constraint, state constraint, and bicycle transition function are denoted by $h_c(\cdot)$, $h_s(\cdot)$, and $f(\cdot)$. The optimization aligns the medoid to all cluster points via Euclidean distance fields and enforces the control smoothness while ensuring kinematic feasibility. We solve optimization with \textit{Augmented Lagrangian iterative linear quadratic regulator} (AL-iLQR)~\cite{mayne1966second} to obtain the refined medoid $\{\hat{\tau}_k^*\}$ constituting precise and drivable centerlines $\mathcal{M}_t$.

\begin{figure*}[t]
    \centering
    \includegraphics[width=0.96\linewidth]{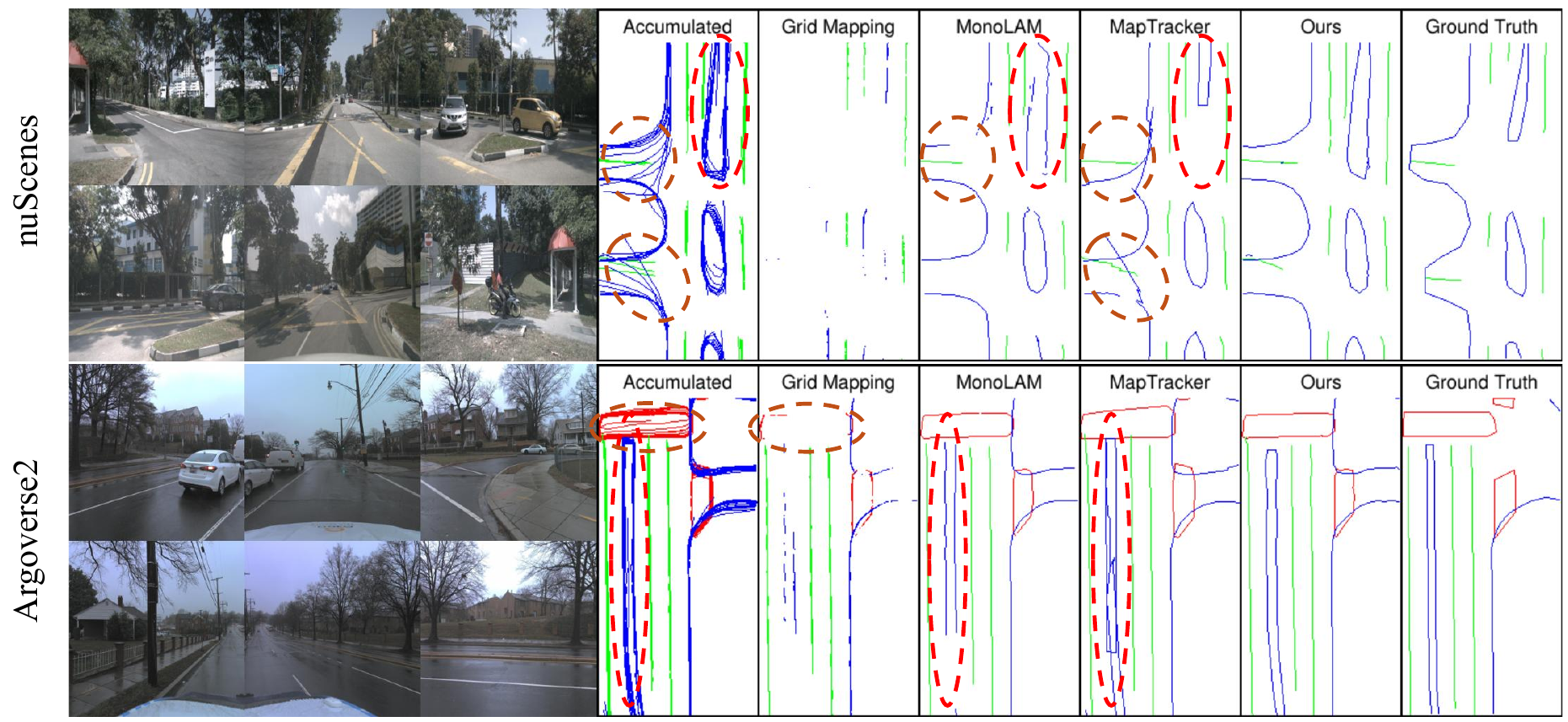}
    \caption{Visual comparison of single-session semantic mapping on nuScenes and Argoverse~2. Left to right: surround-view images, accumulated detection results~\cite{chen2024maptracker}, GridMap~\cite{qin2021light}, MonoLAM~\cite{qiao2023online}, MapTracker$^\dagger$~\cite{chen2024maptracker}, ours, and ground truth. \textcolor{brown}{Brown} dashed boxes highlight temporally inconsistent detection results. \textcolor{red}{Red} dashed boxes indicate performance under complex shapes, such as U-turns.}
    \label{fig:scene_mapping_vis}
    \vspace{-2em}
\end{figure*}

\section{Experiments}
\label{sec:experiments}
This section evaluates CSMapping on semantic and topological mapping. Section~\ref{subsec:single_run_mapping} studies single-session semantic mapping and demonstrates that our vectorized mapping (Section~\ref{sec:vec_mapping}) is robust to observation noise and complex road geometry. Section~\ref{subsec:crowdsourced_mapping} evaluates crowdsourced semantic mapping (Section~\ref{sec:diff_mapping}) on both public and proprietary datasets and shows applications to online detection, while Section~\ref{subsec:ablation} presents ablation studies and Section~\ref{subsec:scalability} analyzes scalability in both training and inference. Section~\ref{subsec:topo_exp} reports topological mapping results (Section~\ref{sec:topology}), including the impact of increasing trajectory data. Task setup, datasets, metrics, and implementation details are provided at the beginning of each subsection. All experiments are conducted on an Intel Core i7-13700K CPU and an NVIDIA GeForce RTX 4080 GPU. Throughout this section, we use consistent color coding: \textcolor{red}{red} for pedestrian crossings, \textcolor{blue}{blue} for boundaries, \textcolor{green}{green} for dividers, and \textcolor{gray}{gray} for SD maps. 

\subsection{Single-session Semantic Mapping}
\label{subsec:single_run_mapping}
\begin{table}[t]
    \centering
    \caption{Single-session semantic mapping on nuScenes and Argoverse~2.}
    \label{tab:scene_mapping}
    \setlength{\tabcolsep}{4pt}
    \resizebox{\linewidth}{!}{%
    \begin{tabular}{@{}ll ccc ccc ccc c@{}}
        \toprule
        \multicolumn{12}{c}{\textbf{nuScenes}} \\
        \addlinespace[2pt]
        \textbf{Model} & \textbf{Method} & \multicolumn{3}{c}{\textbf{Ped. Crossing}} & \multicolumn{3}{c}{\textbf{Divider}} & \multicolumn{3}{c}{\textbf{Boundary}} & \textbf{mIoU} \\
        \cmidrule(lr){3-5} \cmidrule(lr){6-8} \cmidrule(l){9-11} \cmidrule(l){12-12}
        & & IoU & Pre. & Rec. & IoU & Pre. & Rec. & IoU & Pre. & Rec. & \\
        \addlinespace[2pt] \midrule
        \multirow{4}{*}{MapTR} & MapTR~\cite{liao2024maptrv2}      & 9.16 & 16.08 & 14.47 & 27.89 & 51.77 & 41.23 & 30.58 & 55.42 & 47.59 & 22.54 \\
                               & GridMap~\cite{qin2021light}    & 3.92 & \cellcolor{blue!10}{22.02} & 5.46 & 20.25 & 60.64 & 26.57 & 21.68 & \cellcolor{blue!30}{73.32} & 29.87 & 15.28 \\
                               & MonoLAM~\cite{qiao2023online}    & \cellcolor{blue!10}{9.78} & 18.02 & \cellcolor{blue!10}{16.03} & \cellcolor{blue!10}{41.64} & \cellcolor{blue!30}{69.08} & \cellcolor{blue!10}{61.89} & \cellcolor{blue!10}{32.65} & 59.72 & \cellcolor{blue!10}{49.56} & \cellcolor{blue!10}{28.02} \\
                               & Ours       & \cellcolor{blue!30}{14.53} & \cellcolor{blue!30}{22.57} & \cellcolor{blue!30}{24.44} & \cellcolor{blue!30}{42.21} & \cellcolor{blue!10}{68.84} & \cellcolor{blue!30}{63.09} & \cellcolor{blue!30}{39.66} & \cellcolor{blue!10}{63.44} & \cellcolor{blue!30}{62.55} & \cellcolor{blue!30}{32.13} \\
        \midrule
        \multirow{5}{*}{MapTracker} & MapTracker~\cite{chen2024maptracker}   & \cellcolor{blue!10}{18.69} & 28.77 & \cellcolor{blue!10}{27.57} & 50.37 & 75.60 & 71.47 & \cellcolor{blue!10}{43.42} & 67.97 & 65.47 & \cellcolor{blue!10}{37.49} \\
                               & GridMap~\cite{qin2021light}    & 16.16 & \cellcolor{blue!30}{33.97} & 22.03 & 50.19 & \cellcolor{blue!30}{81.04} & 69.58 & 41.48 & \cellcolor{blue!30}{79.26} & 57.83 & 35.94 \\
                               & MonoLAM~\cite{qiao2023online}    & 16.42 & 29.21 & 23.40 & \cellcolor{blue!10}{53.09} & \cellcolor{blue!10}{78.88} & 75.44 & 41.44 & 68.62 & 61.52 & 36.98 \\
                               & MapTracker$^\dagger$~\cite{chen2024maptracker} & 12.63 & 20.22 & 20.25 & 51.36 & 73.50 & \cellcolor{blue!30}{78.10} & 42.24 & 64.56 & \cellcolor{blue!10}{67.13} & 35.41 \\
                               & Ours       & \cellcolor{blue!30}{20.30} & \cellcolor{blue!10}{31.42} & \cellcolor{blue!30}{29.43} & \cellcolor{blue!30}{53.19} & 78.23 & \cellcolor{blue!10}{76.26} & \cellcolor{blue!30}{47.42} & \cellcolor{blue!10}{72.96} & \cellcolor{blue!30}{70.61} & \cellcolor{blue!30}{40.30} \\
        \addlinespace[2pt] \midrule[.8pt] \addlinespace[2pt]
        \multicolumn{12}{c}{\textbf{Argoverse~2}} \\
        \addlinespace[2pt]
        \textbf{Model} & \textbf{Method} & \multicolumn{3}{c}{\textbf{Ped. Crossing}} & \multicolumn{3}{c}{\textbf{Divider}} & \multicolumn{3}{c}{\textbf{Boundary}} & \textbf{mIoU} \\
        \cmidrule(lr){3-5} \cmidrule(lr){6-8} \cmidrule(l){9-11} \cmidrule(l){12-12}
        & & IoU & Pre. & Rec. & IoU & Pre. & Rec. & IoU & Pre. & Rec. & \\
        \addlinespace[2pt] \midrule
        \multirow{4}{*}{MapTR} & MapTR~\cite{liao2024maptrv2}      & \cellcolor{blue!10}{13.64} & 22.40 & \cellcolor{blue!10}{21.41} & 35.29 & 54.13 & 50.41 & \cellcolor{blue!10}{37.80} & 63.34 & \cellcolor{blue!10}{57.37} & \cellcolor{blue!10}{28.91} \\
                               & GridMap~\cite{qin2021light}    & 10.16 & \cellcolor{blue!30}{28.04} & 14.50 & 35.55 & \cellcolor{blue!30}{58.74} & 50.73 & 37.31 & \cellcolor{blue!30}{73.69} & 54.62 & 27.67 \\
                               & MonoLAM~\cite{qiao2023online}    & 9.94 & 20.56 & 15.07 & \cellcolor{blue!10}{37.92} & \cellcolor{blue!10}{57.11} & \cellcolor{blue!10}{54.78} & 36.61 & 62.82 & 57.20 & 28.16 \\
                               & Ours       & \cellcolor{blue!30}{14.49} & \cellcolor{blue!10}{24.62} & \cellcolor{blue!30}{23.47} & \cellcolor{blue!30}{38.79} & 57.09 & \cellcolor{blue!30}{56.67} & \cellcolor{blue!30}{42.84} & \cellcolor{blue!10}{65.72} & \cellcolor{blue!30}{68.56} & \cellcolor{blue!30}{32.04} \\
        \midrule
        \multirow{5}{*}{MapTracker} & MapTracker~\cite{chen2024maptracker}   & \cellcolor{blue!30}{24.32} & 39.33 & \cellcolor{blue!30}{37.29} & 55.88 & 80.54 & 78.66 & 44.14 & 69.05 & 68.23 & \cellcolor{blue!10}{41.45} \\
                               & GridMap~\cite{qin2021light}    & 20.99 & \cellcolor{blue!30}{46.81} & 28.96 & 56.23 & \cellcolor{blue!30}{84.38} & 77.92 & 44.06 & \cellcolor{blue!30}{79.40} & 62.88 & 40.43 \\
                               & MonoLAM~\cite{qiao2023online}    & 16.73 & 34.56 & 24.41 & \cellcolor{blue!10}{57.55} & \cellcolor{blue!10}{82.56} & \cellcolor{blue!10}{81.43} & \cellcolor{blue!10}{45.56} & 73.10 & 68.17 & 39.95 \\
                               & MapTracker$^\dagger$~\cite{chen2024maptracker} & 17.05 & 27.85 & 28.18 & 50.85 & 73.08 & 78.84 & 42.83 & 65.65 & \cellcolor{blue!10}{68.92} & 36.91 \\
                               & Ours       & \cellcolor{blue!10}{23.03} & \cellcolor{blue!10}{39.95} & \cellcolor{blue!10}{34.29} & \cellcolor{blue!30}{57.83} & 82.21 & \cellcolor{blue!30}{82.07} & \cellcolor{blue!30}{47.17} & \cellcolor{blue!10}{74.02} & \cellcolor{blue!30}{70.89} & \cellcolor{blue!30}{42.68} \\
        \bottomrule
    \end{tabular}%
    }
    \begin{flushleft}
    \footnotesize{\textit{Note}: Best/second-best results are highlighted with \colorbox{blue!30}{blue}/\colorbox{blue!10}{light blue}. MapTracker$^\dagger$ refers to MapTracker's built-in mapping module. All metrics are in percentage (\%). Higher values are better ($\uparrow$). IoU, Pre., Rec., mIoU: IoU, Precision, Recall, mean IoU.}
    \end{flushleft}
\end{table}

This task evaluates single-session semantic mapping, validating the proposed vectorized mapping (Section~\ref{sec:vec_mapping}) for locally accurate maps under noisy observations and complex road elements. Unless otherwise stated, the hyperparameters are: VCA weighting \(\alpha=0.5\) and TLS threshold \(c=1.5\,\mathrm{m}\). Experiments are conducted on the \emph{nuScenes} validation set (150 scenes, 20 s each) and the \emph{Argoverse 2} sensor validation set (150 scenes, 15 s each). Following MonoLAM~\cite{qiao2023online}, we warp the constructed map to each keyframe local frame and rasterize to \(30\,\mathrm{m}\times60\,\mathrm{m}\) (matching the detector output size) at \(0.2\,\mathrm{m}\)/px. A \(3\times3\) morphological dilation is applied to mitigate discretization error. We report Intersection-over-Union (IoU), Precision, and Recall for pedestrian-crossing, divider, and boundary. Two map detection models are used: MapTR~\cite{liao2024maptrv2} and MapTracker~\cite{chen2024maptracker} (leveraging temporal cues).

Table~\ref{tab:scene_mapping} summarizes results on \emph{nuScenes} and \emph{Argoverse 2} with MapTR~\cite{liao2024maptrv2} and MapTracker~\cite{chen2024maptracker}. Across datasets and detectors, our vectorized mapping consistently increases IoU over frame-level detection and classical baselines while preserving a balanced precision–recall trade-off, demonstrating plug-and-play benefits. 

As highlighted by the brown dashed boxes in Fig.~\ref{fig:scene_mapping_vis}, the accumulation of frame-level detector outputs exhibits strong temporal inconsistency; consequently, GridMap's strict foreground voting attains high precision but sacrifices recall. For methods that treat curves as atomic instances (MonoLAM~\cite{qiao2023online}, MapTracker\(^\dagger\)~\cite{chen2024maptracker}), uneven confidence and instance ambiguity (Section~\ref{subsec:data_association}) lead to over-keeping or over-discarding, as shown in the brown dashed boxes in Fig.~\ref{fig:scene_mapping_vis}. Instead, our approach introduces VCA to perform grid mapping with a small background weight \(\alpha\) to preserve recall, then clips each curve using the grid map, keeping complete reliable observations while removing spurious segments. 

For complex curves (the red dashed boxes in Fig.~\ref{fig:scene_mapping_vis}), nearest-neighbor association (used in MonoLAM~\cite{qiao2023online} and MapTracker\(^\dagger\)~\cite{chen2024maptracker}) becomes ambiguous under arc-length parameterization, whereas our order-preserving CDTW matching resolves such cases. Furthermore, GNC-based TLS optimization rejects association outliers and produces smooth and accurate landmark estimates.

\subsection{Crowdsourced Semantic Mapping}
\label{subsec:crowdsourced_mapping}

\subsubsection{Task Setup and Evaluation}
This task evaluates generative mapping (Section~\ref{sec:diff_mapping}) using crowdsourced data from multiple vehicles. We use the public \emph{nuScenes} dataset for main evaluation and a larger proprietary dataset (Section~\ref{subsec:self_proprietary_datasets}) for practical validation. Since the two datasets differ substantially in scale and complexity, we adopt dataset-specific model configurations (detailed in \appref{app:diffusion_architecture}).

For latent-space optimization, default settings are: number of Gaussian bases \(K=256\) (the basis set is randomly instantiated once and reused across all experiments), DPM-Solver++ as the denoising scheduler, 3 denoising steps, no multi-start strategy (\(N_s=1\)), 60 optimization iterations, initialization of the latent by vectorized mapping with 5-step inversion, and foreground binarization threshold \(\tau_{\text{fg}}=\text{0.33}\).

\begin{table}[t]
\centering
\caption{Crowdsourced mapping benchmark. Comparison under \textbf{Visible} (masked) and \textbf{Full} areas.}
\label{tab:csm_bm}
\setlength{\tabcolsep}{4pt}
\resizebox{\linewidth}{!}{%
\begin{tabular}{@{}ll ccc ccc ccc c@{}}
\toprule
\multirow{2}{*}{\textbf{Area}} & \multirow{2}{*}{\textbf{Method}} & \multicolumn{3}{c}{\textbf{Ped. Crossing}} & \multicolumn{3}{c}{\textbf{Divider}} & \multicolumn{3}{c}{\textbf{Boundary}} & \textbf{mIoU} \\
\cmidrule(lr){3-5} \cmidrule(lr){6-8} \cmidrule(l){9-11} \cmidrule(l){12-12}
& & IoU & Pre. & Rec. & IoU & Pre. & Rec. & IoU & Pre. & Rec. & \\
\addlinespace[2pt] \midrule
\multirow{5}{*}{Visible} & GridMap~\cite{qin2021light} & 33.1 & \cellcolor{blue!30}{59.9} & 37.8 & 54.8 & \cellcolor{blue!30}{89.5} & 59.2 & 54.4 & \cellcolor{blue!30}{83.5} & 61.1 & 47.4 \\
& Vectorized Mapping & 40.2 & 53.1 & 51.6 & 66.2 & \cellcolor{blue!10}{81.6} & 77.4 & 54.3 & \cellcolor{blue!10}{78.2} & 64.0 & 53.6 \\
& Classifier Guidance~\cite{dhariwal2021diffusion} & 44.7 & 54.6 & 57.7 & 65.9 & 73.8 & \cellcolor{blue!30}{85.4} & 58.9 & 69.3 & \cellcolor{blue!30}{79.0} & 56.5 \\
& Ours (w/o SD Map) & \cellcolor{blue!10}{46.8} & 57.0 & \cellcolor{blue!10}{58.5} & \cellcolor{blue!10}{67.2} & 79.0 & 81.2 & \cellcolor{blue!10}{60.2} & 75.1 & 74.4 & \cellcolor{blue!10}{58.1} \\
& Ours & \cellcolor{blue!30}{47.9} & \cellcolor{blue!10}{58.2} & \cellcolor{blue!30}{59.6} & \cellcolor{blue!30}{68.8} & 79.3 & \cellcolor{blue!10}{83.3} & \cellcolor{blue!30}{61.3} & 75.2 & \cellcolor{blue!10}{76.0} & \cellcolor{blue!30}{59.3} \\
\addlinespace[2pt] \midrule[.8pt] \addlinespace[2pt]
\multirow{5}{*}{Full} & GridMap~\cite{qin2021light} & 32.3 & \cellcolor{blue!30}{59.9} & 36.6 & 51.2 & \cellcolor{blue!30}{89.5} & 54.9 & 49.0 & \cellcolor{blue!30}{83.5} & 54.4 & 44.2 \\
& Vectorized Mapping & 39.3 & 53.1 & 50.3 & 62.2 & \cellcolor{blue!10}{81.6} & 72.2 & 49.1 & \cellcolor{blue!10}{78.1} & 57.1 & 50.2 \\
& Classifier Guidance~\cite{dhariwal2021diffusion} & 43.7 & 54.4 & 56.3 & 62.4 & 73.1 & \cellcolor{blue!30}{80.5} & 53.7 & 67.9 & \cellcolor{blue!30}{71.5} & 53.3 \\
& Ours (w/o SD Map) & \cellcolor{blue!10}{45.6} & 56.5 & \cellcolor{blue!10}{57.2} & \cellcolor{blue!10}{63.4} & 78.6 & 76.2 & \cellcolor{blue!10}{54.7} & 73.9 & 67.2 & \cellcolor{blue!10}{54.6} \\
& Ours & \cellcolor{blue!30}{46.7} & \cellcolor{blue!10}{57.4} & \cellcolor{blue!30}{58.4} & \cellcolor{blue!30}{65.5} & 78.2 & \cellcolor{blue!10}{79.6} & \cellcolor{blue!30}{56.2} & 73.1 & 70.0 & \cellcolor{blue!30}{56.1} \\
\bottomrule
\end{tabular}}
\begin{flushleft}
\footnotesize{\textit{Note}: Best/second-best results are highlighted with \colorbox{blue!30}{blue}/\colorbox{blue!10}{light blue}. All metrics are in percentage (\%). Higher values are better ($\uparrow$). IoU, Pre., Rec., mIoU: IoU, Precision, Recall, mean IoU.}
\end{flushleft}
\end{table}

\subsubsection{Evaluation on nuScenes}
\label{exp:csm_quality}

For nuScenes, we split training and test sets by geographic regions. The test set comprises 224 samples, each covering \(100\,\mathrm{m}\times100\,\mathrm{m}\) with on average 12.8 trajectories per sample. The training set contains 190{,}000 samples with the same spatial extent as the test set. The distance between the centers of any training sample and any test sample is \(>10\,\mathrm{m}\) to avoid train–test leakage of ground-truth patterns. Each sample is rasterized to \(256\times256\times3\), where the three channels correspond to pedestrian-crossing, divider, and boundary. The latent's shape is $16\times16\times16$. Training uses 200{,}000 steps with batch size 192 on 4 NVIDIA RTX 3090 GPUs.

We evaluate both reconstruction and generation by partitioning the area into: \emph{Visible}, the sensor-observed region (robust reconstruction), and \emph{Full}, the entire region including unobserved areas (plausible completion). Metrics are IoU, Precision, and Recall.

Results are shown in Table~\ref{tab:csm_bm}, comparing against vectorized mapping, GridMap~\cite{qin2021light}, and classifier guidance~\cite{dhariwal2021diffusion}. Our method attains the highest IoU in both \emph{Visible} and \emph{Full} areas. As in Section~\ref{subsec:single_run_mapping}, GridMap's strict decision rule yields high precision but low recall, whereas vectorized mapping achieves a more balanced precision–recall trade-off. 

Classifier guidance~\cite{dhariwal2021diffusion} (CG) injects observation-likelihood gradients into each denoising step to steer generation, but does not optimize the initial noise latent \(x_T\). In contrast, we perform MAP optimization directly on \(x_T\) while preserving the learned prior. To ensure fair comparison, we strengthen the CG baseline by initializing with inversion, using 25 denoising steps, and adopting adaptive guidance coefficients; even so, our approach yields higher precision and IoU. We attribute this advantage to: (i) Gaussian-basis reparameterization, which enforces stronger Gaussian constraints and improves plausibility; and (ii) iterative gradient-based optimization on the initial latent, which is more effective than step-wise gradient injection for accurate reconstruction. 

We also compare with and without SD Map conditioning: SD maps encode road-level skeleton information that steers generation toward structurally plausible maps, aiding convergence and improving quality in both \emph{Visible} and \emph{Full} areas. Additional analyses on controllability and robustness to invalid SD maps are provided in Section~\ref{subsec:scalability_training}.

\begin{figure*}[t]
    \centering
    \includegraphics[width=0.8\linewidth]{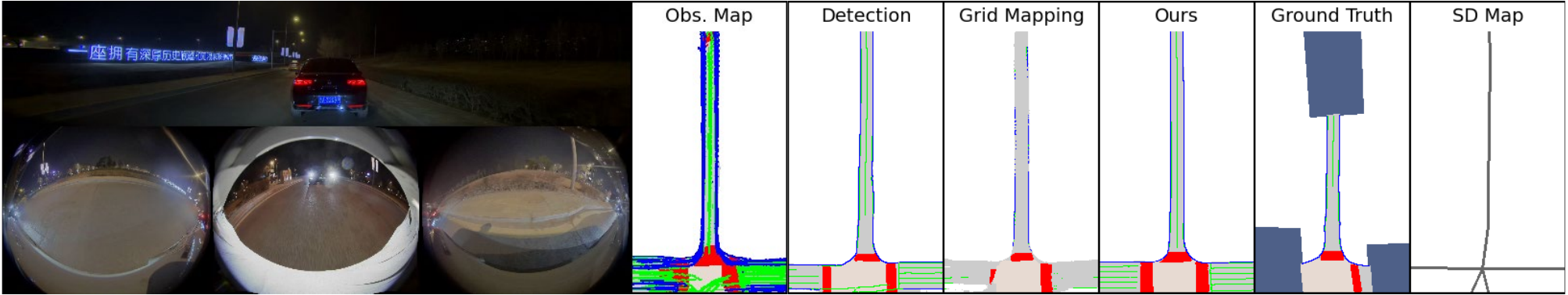}
    \includegraphics[width=0.8\linewidth]{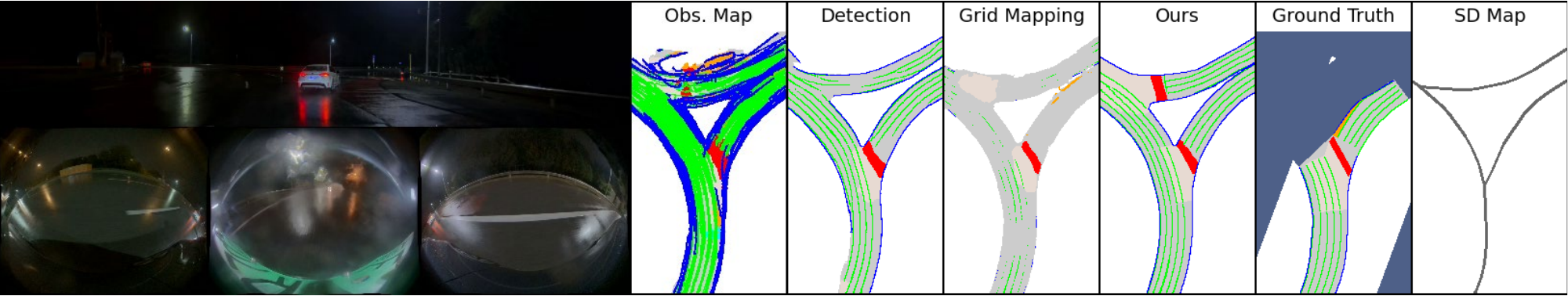}
    \caption{Qualitative results of crowdsourced semantic mapping on the proprietary dataset. \textcolor[RGB]{78,96,136}{Dark blue} areas represent unannotated pixels, which are excluded from evaluation.}
    \label{fig:zyt_exp}
\end{figure*}

\begin{table}[t]
    \centering
    \caption{Evaluation on proprietary datasets (IoU).}
    \label{tab:zyt_bm}
    \setlength{\tabcolsep}{3pt}
    \resizebox{0.9\linewidth}{!}{%
    \begin{tabular}{@{}lcccccc c@{}}
        \toprule
        \textbf{Method} & \textbf{Road} & \textbf{Intersect.} & \textbf{Rev. Lane} & \textbf{Ped. Cross.} & \textbf{Lane} & \textbf{Border} & \textbf{mIoU} \\
        \midrule
        Detection & 88.7 & 51.3 & 47.9 & 37.9 & 31.9 & 25.9 & 47.3 \\
        GridMap~\cite{qin2021light} & \cellcolor{blue!10}{92.7} & \cellcolor{blue!30}{60.9} & \cellcolor{blue!30}{60.9} & \cellcolor{blue!10}{45.3} & \cellcolor{blue!10}{32.9} & \cellcolor{blue!10}{26.6} & \cellcolor{blue!10}{53.2} \\
        Ours & \cellcolor{blue!30}{92.8} & \cellcolor{blue!10}{60.6} & \cellcolor{blue!10}{57.0} & \cellcolor{blue!30}{47.7} & \cellcolor{blue!30}{40.8} & \cellcolor{blue!30}{33.5} & \cellcolor{blue!30}{55.4} \\
        \bottomrule
    \end{tabular}%
    }
\end{table}
\subsubsection{Evaluation on Proprietary Datasets}
\label{subsec:self_proprietary_datasets}

We utilize a large-scale proprietary dataset collected from major Chinese cities for practical validation. The dataset comprises a training set of 400 million HD-map samples and a test set of 25,000 samples, which are geographically disjoint to ensure zero leakage. Each sample is represented as a \(272\times160\times6\) multi-channel raster with a resolution of \(0.6\,\mathrm{m}\)/px. The six semantic channels correspond to pedestrian crossings, dividers, boundaries, intersections, reversal lanes, and road areas. Unlike nuScenes where we use perception outputs from MapTR~\cite{liao2024maptrv2}, here we train a map diffusion model conditioned on surround-view images and SD maps to directly produce map detections, which then serve as the map generation model $G$ for our generative mapping framework (see \appref{app:diffusion_architecture} for implementation details).

Ground-truth maps for the test set were manually annotated based on high-density point clouds captured by LiDAR-equipped survey vehicles. Regions lacking annotation due to occlusion or limited LiDAR sensor field-of-view (FOV) are masked out during evaluation. We benchmark our method against standard detection and grid mapping baselines. Quantitative and qualitative comparisons are presented in Table~\ref{tab:zyt_bm} and Fig.~\ref{fig:zyt_exp}, respectively, with additional examples in \appref{app:proprietary}.

\begin{figure*}[t]
    \centering
    \includegraphics[width=0.9\linewidth]{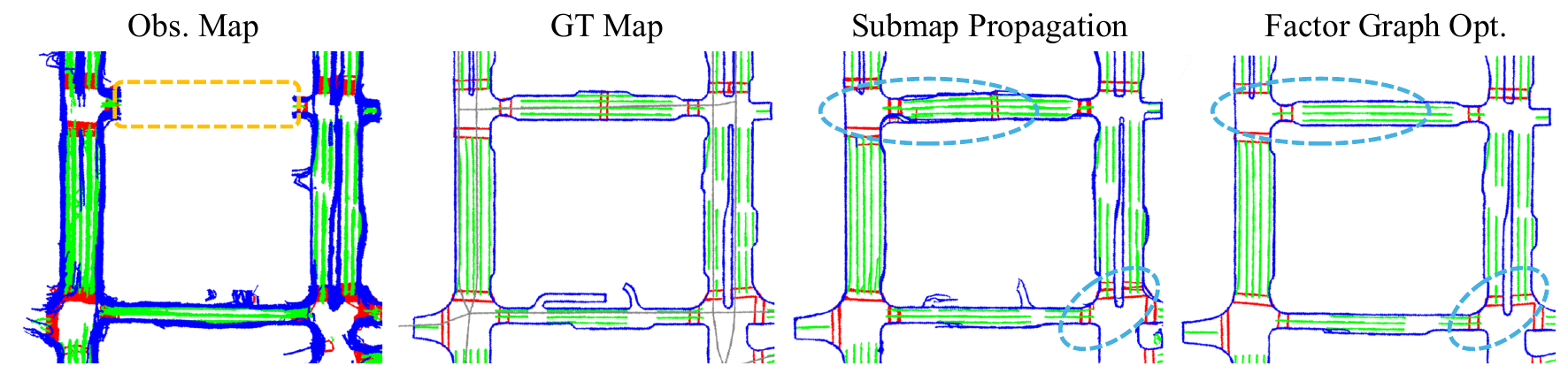}
    \caption{Global map reconstruction. From left to right: input crowdsourced observations; ground-truth map (SD map is colored by \textcolor{gray}{gray}); global map by sequential submap propagation; final result after joint factor graph optimization. \textcolor{orange}{Orange} dashed circles mark unobserved regions, which our method plausibly completes. \textcolor{blue}{Blue} dashed circles indicate improved inter-submap consistency after optimization.}
    \label{fig:global_map}
\end{figure*}

\subsubsection{Global Map Generation}
\label{exp:global_map}
Fig.~\ref{fig:global_map} demonstrates global map construction via factor-graph optimization. In this example, 22 latents represent submaps with \(\sim10\,\mathrm{m}\) spacing between adjacent centers. The optimization runs on one GPU and uses \(\sim13\) GB memory. As shown in the figure, our method plausibly completes unobserved regions while maintaining smooth transitions between generated and reconstructed areas, and improves inter-submap consistency after optimization.

\begin{table}[t]
    \centering
    \caption{Online perception enhancement on HRMapNet with different prior maps. All metrics are in percentage (\%).}
    \label{tab:csm_hrmapnet_enhancement}
    \resizebox{0.90\columnwidth}{!}{%
    \begin{tabular}{@{}l cccc@{}}
        \toprule
        Prior Map & AP\textsubscript{div} $\uparrow$ & AP\textsubscript{ped} $\uparrow$ & AP\textsubscript{bou} $\uparrow$ & mAP $\uparrow$ \\
        \midrule
        HRMap~\cite{zhang2024hrmapnet} (Baseline) & 72.2 & 73.0 & 73.9 & 73.0 \\
        Our CSM & \cellcolor{blue!10}{86.2} & \cellcolor{blue!10}{76.0} & \cellcolor{blue!10}{82.7} & \cellcolor{blue!10}{81.6} \\
        GT Map (Oracle)  & \cellcolor{blue!30}{87.7} & \cellcolor{blue!30}{87.7} & \cellcolor{blue!30}{86.2} & \cellcolor{blue!30}{87.2} \\
        \bottomrule
    \end{tabular}%
    }
\end{table}
\subsubsection{Enhancing Online Detection}
\label{subsec:enhancing_online_perception}
Constructed crowdsourced maps can serve as priors for online detection. We use HRMapNet~\cite{zhang2024hrmapnet}, which accepts rasterized prior maps. We compare three priors on nuScenes: (1) HRMapNet’s original rasterized prior (grid-mapping variant), (2) our generated crowdsourced map, and (3) the ground truth. As shown in Table~\ref{tab:csm_hrmapnet_enhancement}, introducing CSM as a prior improves mean Average Precision (mAP) across all categories and outperforms HRMapNet’s original prior.

\subsection{Ablation Studies for Latent-Space Optimization}
\label{subsec:ablation}

\begin{figure}[t]
    \centering
    \begin{subfigure}[t]{0.49\linewidth}
        \centering
        \includegraphics[width=\linewidth]{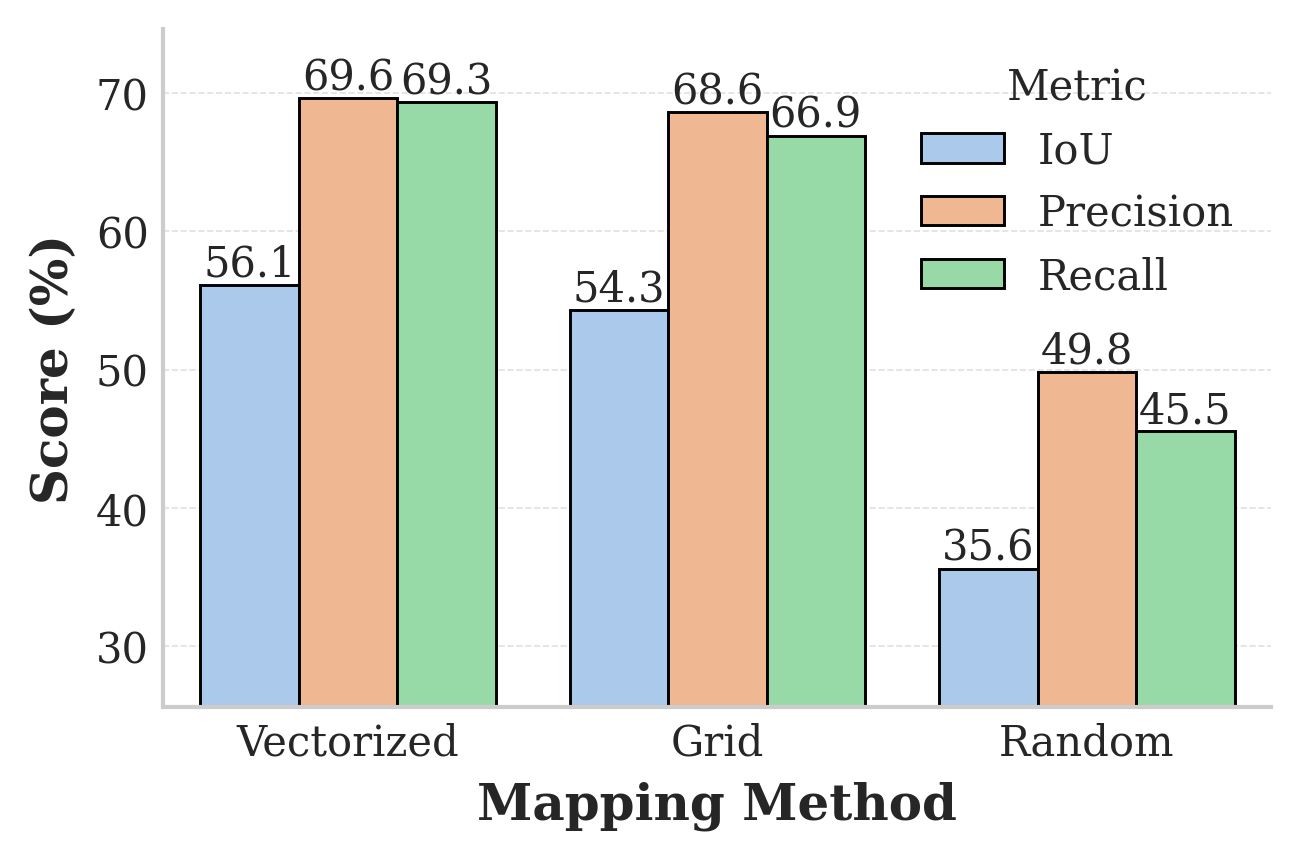}
        \caption{Initial Map}
        \label{fig:ablation_initial_map}
    \end{subfigure}
    \hfill
    \begin{subfigure}[t]{0.49\linewidth}
        \centering
        \includegraphics[width=\linewidth]{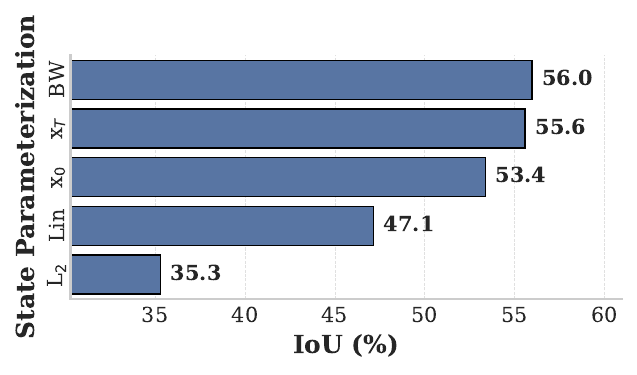}
        \caption{Optimization State}
        \label{fig:ablation_opt_state}
    \end{subfigure}

    \vspace{1em}

    \begin{subfigure}[t]{0.49\linewidth}
        \centering
        \includegraphics[width=\linewidth]{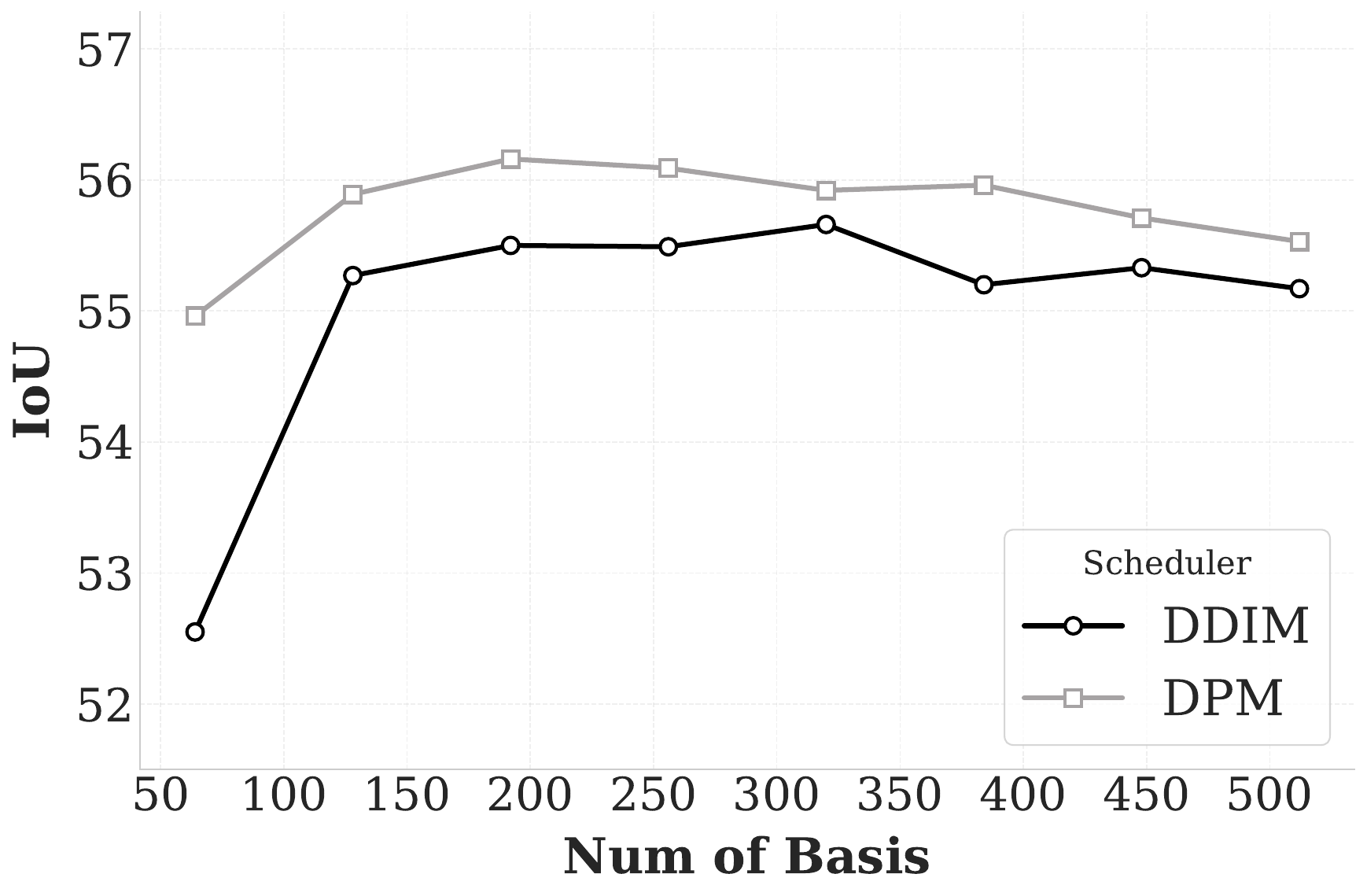}
        \caption{Number of Basis}
        \label{fig:ablation_num_basis}
    \end{subfigure}
    \hfill
    \begin{subfigure}[t]{0.49\linewidth}
        \centering
        \includegraphics[width=\linewidth]{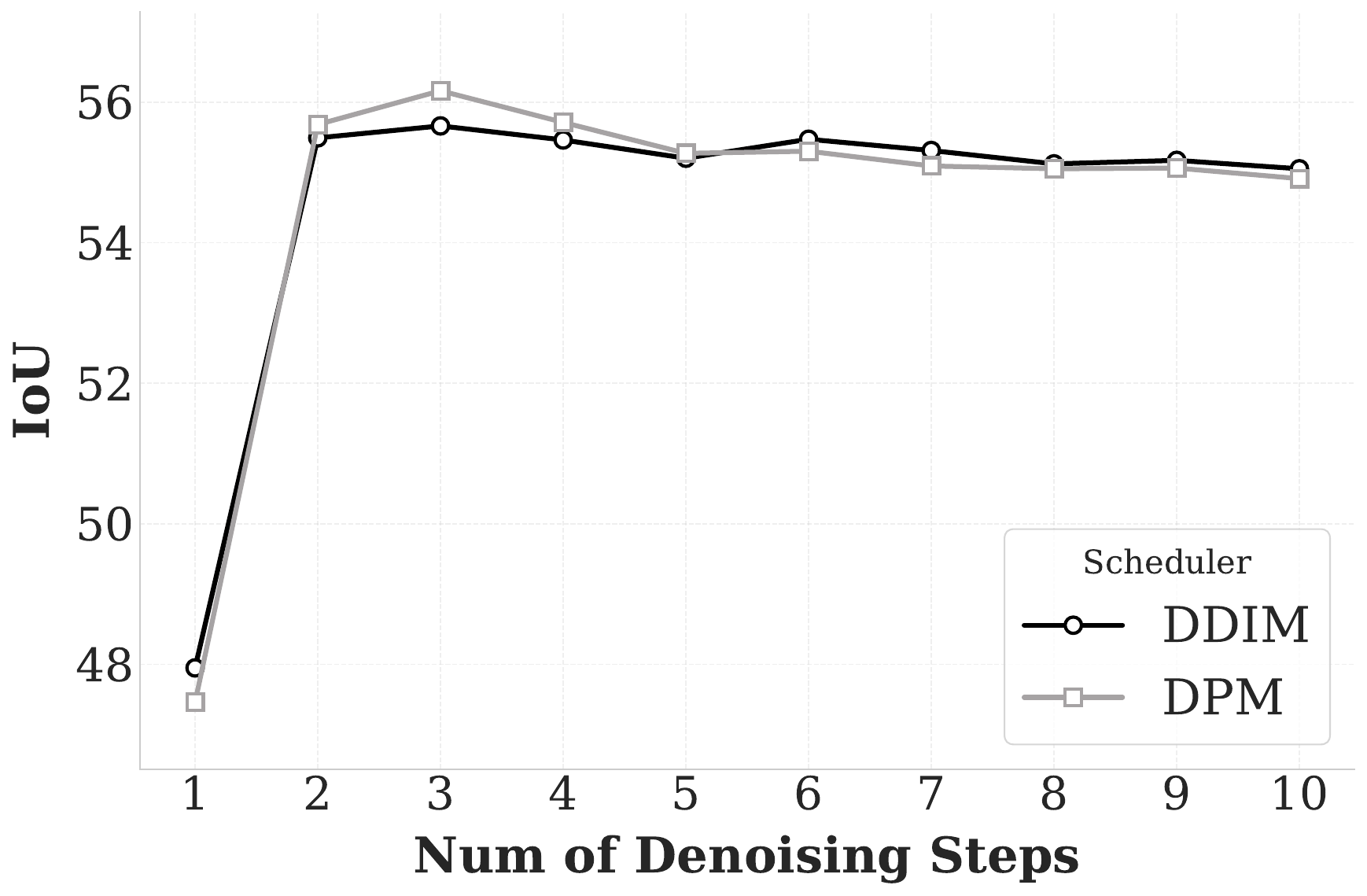}
        \caption{Denoising Steps}
        \label{fig:ablation_denoising_steps}
    \end{subfigure}

    \caption{Ablation results. All metrics are averaged over 3 categories. (a) IoU with different initial map sources; (b) IoU across optimization strategies; (c) impact of basis count; (d) effect of denoising steps.}
    \label{fig:ablation_quantitative}
\end{figure}

\begin{figure}[t]
    \centering
    \includegraphics[width=0.96\linewidth]{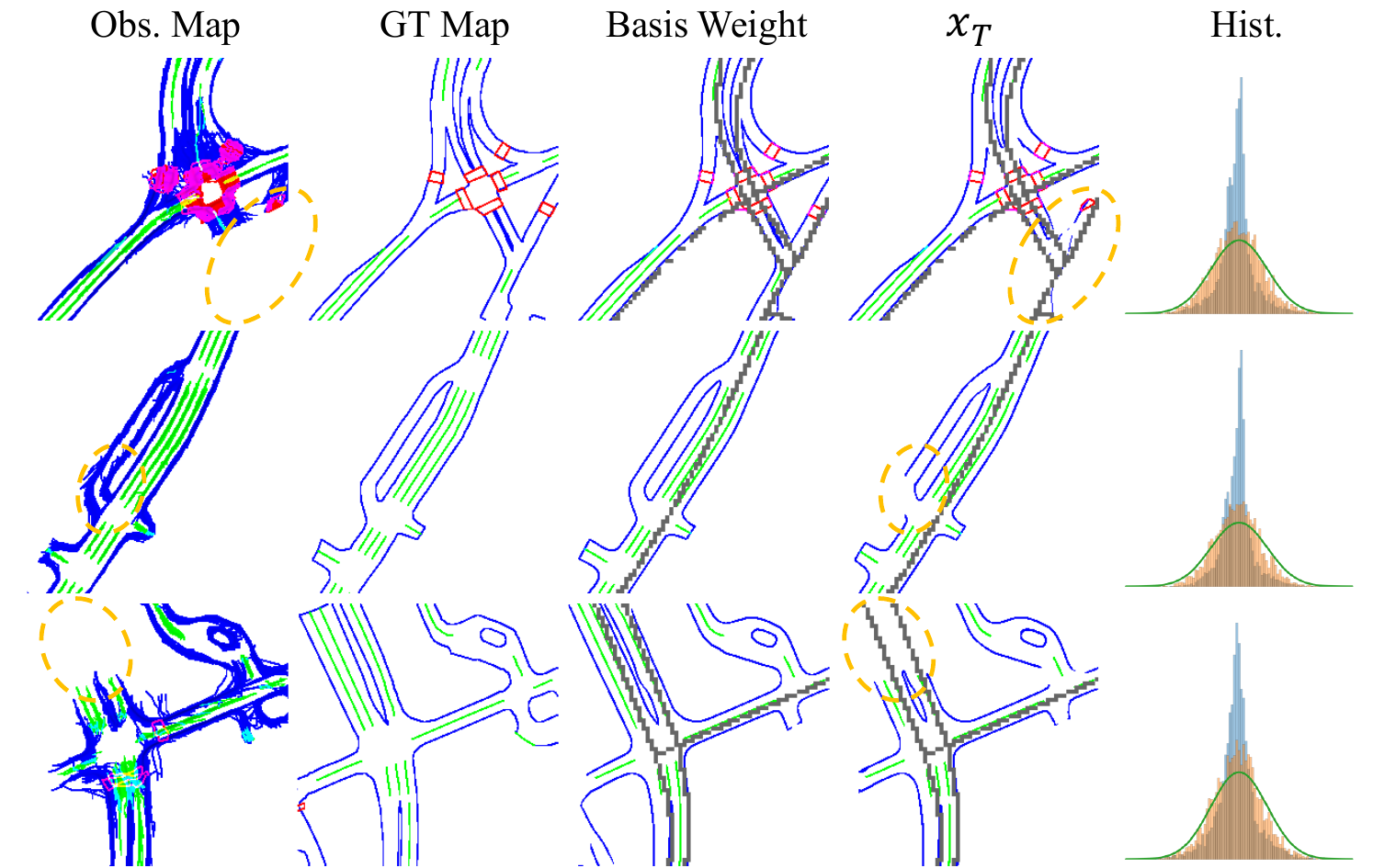}
    \caption{Visualization comparison of optimization strategies: Gaussian-basis reparameterization (ours) and $x_T$ with $\sqrt{d}$ norm constraint on three evaluation samples (No.~189, No.~191, and No.~215). From left to right: input observation, ground-truth map, optimized map by our method, optimized map by $x_T$ (with $\sqrt{d}$ norm constraint), and histogram of optimized latent distribution. \textcolor{orange}{Orange} dashed circles mark unobserved areas where $x_T$ loses generation ability. In the histogram, \textcolor[RGB]{255, 198, 146}{orange}, \textcolor[RGB]{155, 194, 223}{blue}, and \textcolor[RGB]{44, 160, 44}{green} represent optimized latents by our method, $x_T$ (with $\sqrt{d}$ norm constraint), and true Gaussian distribution, respectively.}
    \label{fig:opt_state_ablation}
\end{figure}

\begin{figure}[t]
    \centering
    \includegraphics[width=0.96\linewidth]{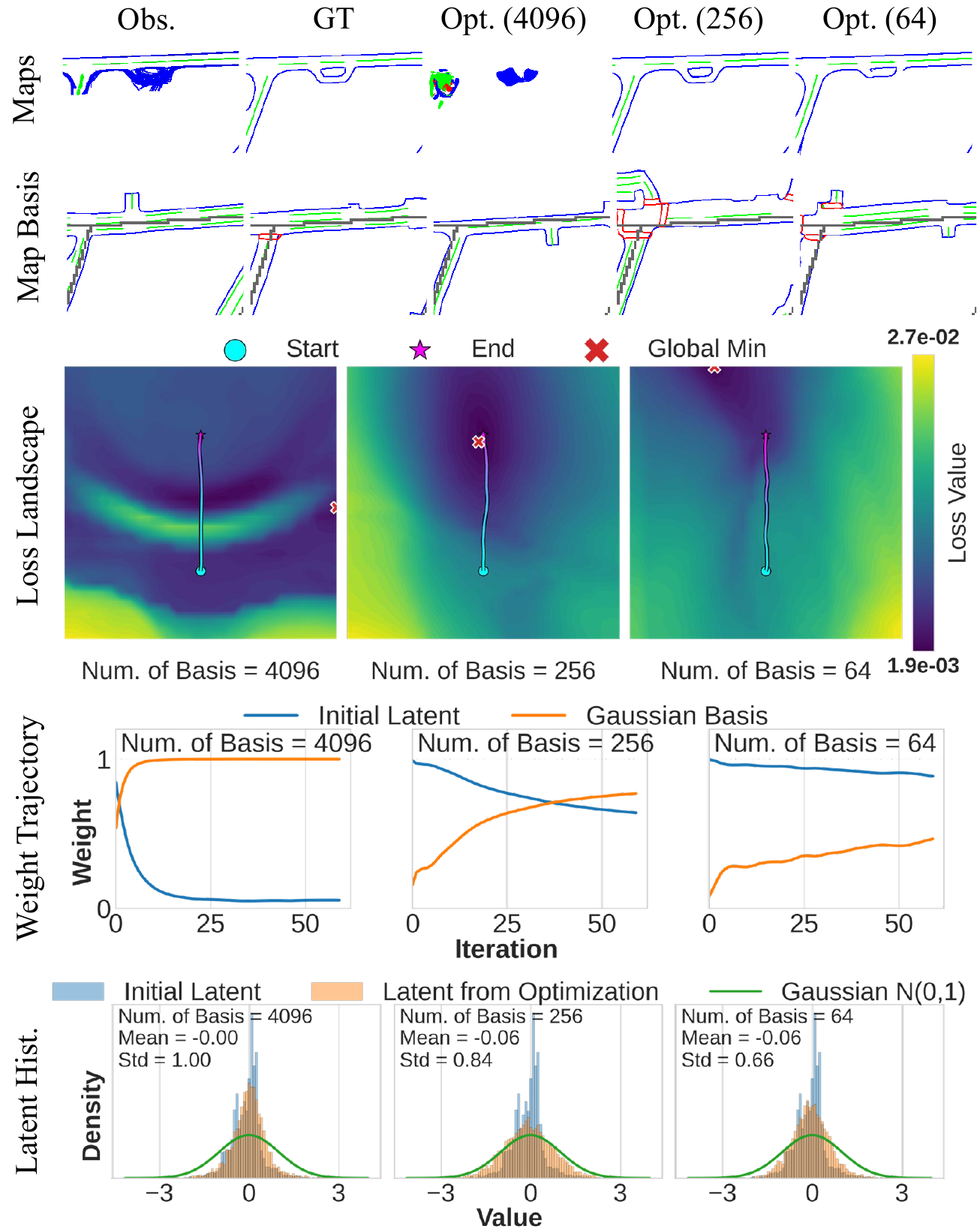}
    \caption{Analysis of basis number in Gaussian-basis reparameterization (sample No.~171). Row 1: observations, ground-truth map, and reconstructions with different basis numbers; Row 2: maps generated from different Gaussian basis sets under the same SD-map condition; Row 3: loss landscapes of different basis numbers along with historical optimization trajectories; Row 4: weight trajectories where \textcolor[RGB]{255, 127, 13}{orange} and \textcolor[RGB]{31, 120, 181}{blue} represent the weight norm of Gaussian bases and the inversion-based initial latent, respectively; Row 5: latent distributions where \textcolor[RGB]{255, 198, 146}{orange}, \textcolor[RGB]{155, 194, 223}{blue}, and \textcolor[RGB]{44, 160, 44}{green} represent histograms of optimized latent distribution, initial latent, and true Gaussian distribution, respectively.}
    \label{fig:num_basis}
\end{figure}

\begin{figure}[t]
    \centering
    \includegraphics[width=0.618\linewidth]{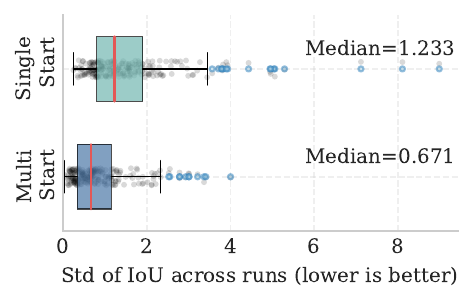}
    \caption{Stability of IoU across different Gaussian-basis groups at fixed \(K\) under \emph{Full}-area evaluation.}
    \label{fig:basis_stable}
\end{figure}

\subsubsection{Map Initialization}
\label{subsec:ab_initialization}
As shown in Fig.~\ref{fig:ablation_initial_map}, we compare initialization strategies: random latent, inversion from grid mapping~\cite{qin2021light}, and inversion from vectorized mapping. Incorporating prior via vectorized-map inversion outperforms random and grid mapping based initializations, consistent with Table~\ref{tab:csm_bm} where vectorized mapping surpasses grid mapping. Both inversion-based initializations outperform random. In summary, classical mapping thus provides an effective warm-start for latent-space optimization, and improvements in classical mapping and inversion further translate to our performance.

\subsubsection{Re-parameterization of Optimization Variables}
\label{subsec:ab_opt_state}
As shown in Fig.~\ref{fig:ablation_opt_state}, we compare four strategies for constraining the latent to a Gaussian distribution: (1) optimize \(x_T\) with \(\lVert x_T\rVert_2=\sqrt{d}\) (latent dimension \(d\); cf. PULSE~\cite{menon2020pulse}); (2) directly optimize \(x_0\); (3) optimize a linear transform applied to \(x_T^{\text{Init}}\) (denoted Lin); and (4) optimize \(x_T\) with an \(\ell_2\) regularizer toward 0 (denoted \(\text{L}_2\)). Our Gaussian-basis reparameterization is denoted as BW (basis weight).

Directly optimizing \(x_0\) discards the prior and relies solely on observation likelihood, reducing robustness to outliers. Although Lin and \(\text{L}_2\) attempt to constrain \(x_T\) to a Gaussian, they are ineffective: Lin's capacity is limited (essentially scale and bias), and \(\text{L}_2\) is inappropriate because high-dimensional Gaussian mass concentrates near radius \(\sqrt{d}\)~\cite{menon2020pulse}. Imposing \(\lVert x_T\rVert_2=\sqrt{d}\) yields acceptable performance overall but, as the orange dashed circles in Fig.~\ref{fig:opt_state_ablation} show, loses generation ability in unobserved areas: matching the norm alone still deviates from the true Gaussian (shown in the histogram in Fig.~\ref{fig:opt_state_ablation}) and harms generation quality.

\begin{figure*}[t]
    \centering
    \includegraphics[width=0.95\linewidth]{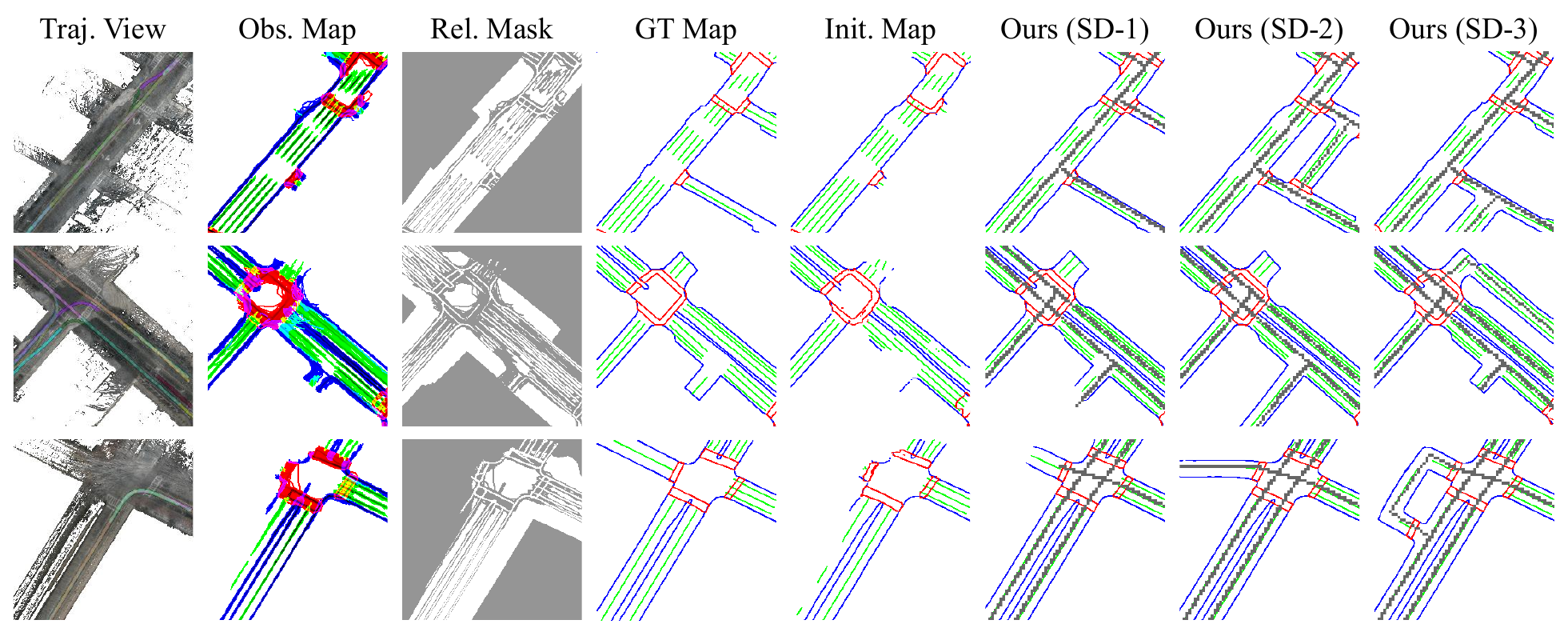}
    \caption{Controllable generation with different SD-Map conditions. Each row represents a sample. Left to right: top-down view of road and vehicle trajectories, observation map, reliability mask (Rel. Mask), ground truth map, initial map from vector mapping, and optimized maps under different SD-Map conditions (gray skeletons).}
    \label{fig:sd_control}
\end{figure*} 

\begin{figure}[t]
    \centering
    \includegraphics[width=0.98\linewidth]{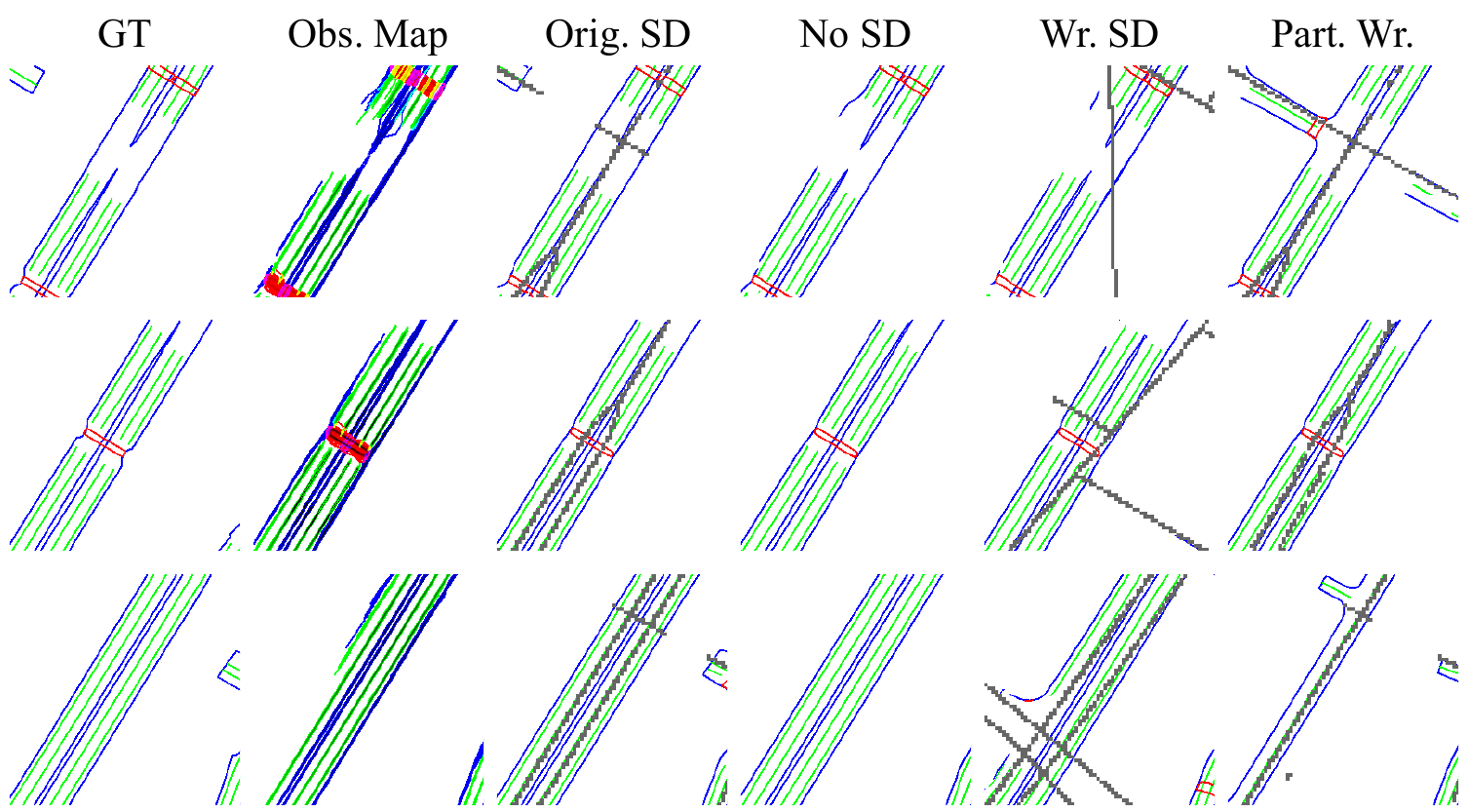}
    \caption{Generation under invalid SD-Map conditions. Each row is a scene. Columns (left to right): ground truth, observations, and our outputs under original SD (Orig. SD), no SD (No SD), entirely wrong SD (Wr. SD), and partially wrong SD (Part. Wr.) guidance.}
    \label{fig:invalid_sd}
\end{figure}

\begin{figure}[t]
    \centering
    \includegraphics[width=0.618\linewidth]{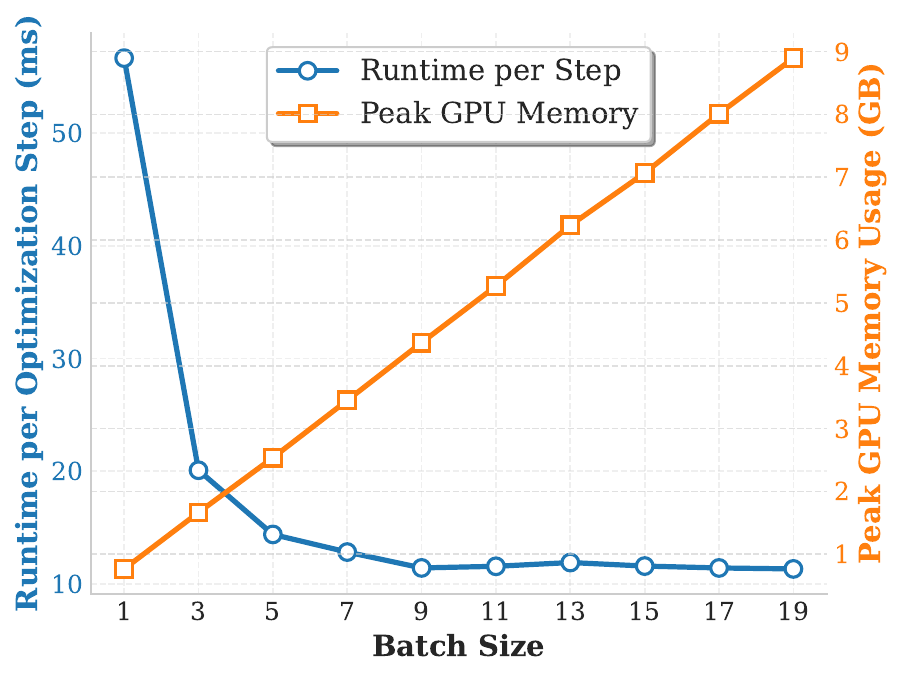}
    \caption{Runtime and GPU memory usage versus batch size.}
    \label{fig:runtime_analysis}
\end{figure}

\subsubsection{Number of Gaussian Bases}
\label{subsec:ab_num_basis}
The number of bases is a key factor in Gaussian-basis reparameterization. As shown in Fig.~\ref{fig:ablation_num_basis}, \(K\in[128,320]\) yields strong results for latent dimension \(d=4096\). Fig.~\ref{fig:num_basis} provides qualitative analysis (sample No.~171). Comparing reconstructions across different basis numbers (row 1) reveals the impact of basis count on reconstruction quality. The second row shows maps generated from different Gaussian basis sets under the same SD-map condition, demonstrating that different bases explore diverse latent directions while the SD map provides structural guidance. 

The loss landscape visualization (row 3) shows that with too many bases (4096), a noisy landscape hinders convergence, while with too few (64), limited control impedes reaching the global optimum, causing reconstruction to degrade in extremely noisy areas. Weight analysis (rows 4--5) reveals that with 4096 bases, the remaining weights dominate (norm near 1) and deviate from a true Gaussian despite preserving mean 0 and unit variance. With 64 bases, the remaining-weight norms stay relatively small, indicating limited influence and reduced expressiveness. Therefore, a moderate \(K\) is necessary to balance optimization stability and expressiveness.

Fig.~\ref{fig:basis_stable} studies result consistency across different Gaussian-basis groups at fixed \(K\). We run 20 seeds to generate different basis sets and compute the per-sample standard deviation of IoU under \emph{Full}-area evaluation. The distribution shows a small median (\(\sim1.2\)) but a long tail (up to \(\sim8\)), attributable to imperfect initialization or weak observation supervision. With a multi-start strategy (\(N_s=20\)), the deviation further reduces (median \(\sim0.67\)), indicating reduced sensitivity to initialization; posterior score-based selection also reduces randomness.

\begin{figure*}[t]
    \centering

    \begin{subfigure}{0.95\linewidth}
        \centering
        \begin{subfigure}[c]{0.05\linewidth}
            \rotatebox{90}{(a) Model Size}
        \end{subfigure}%
        \begin{subfigure}[c]{0.95\linewidth}
            \includegraphics[width=0.33\linewidth]{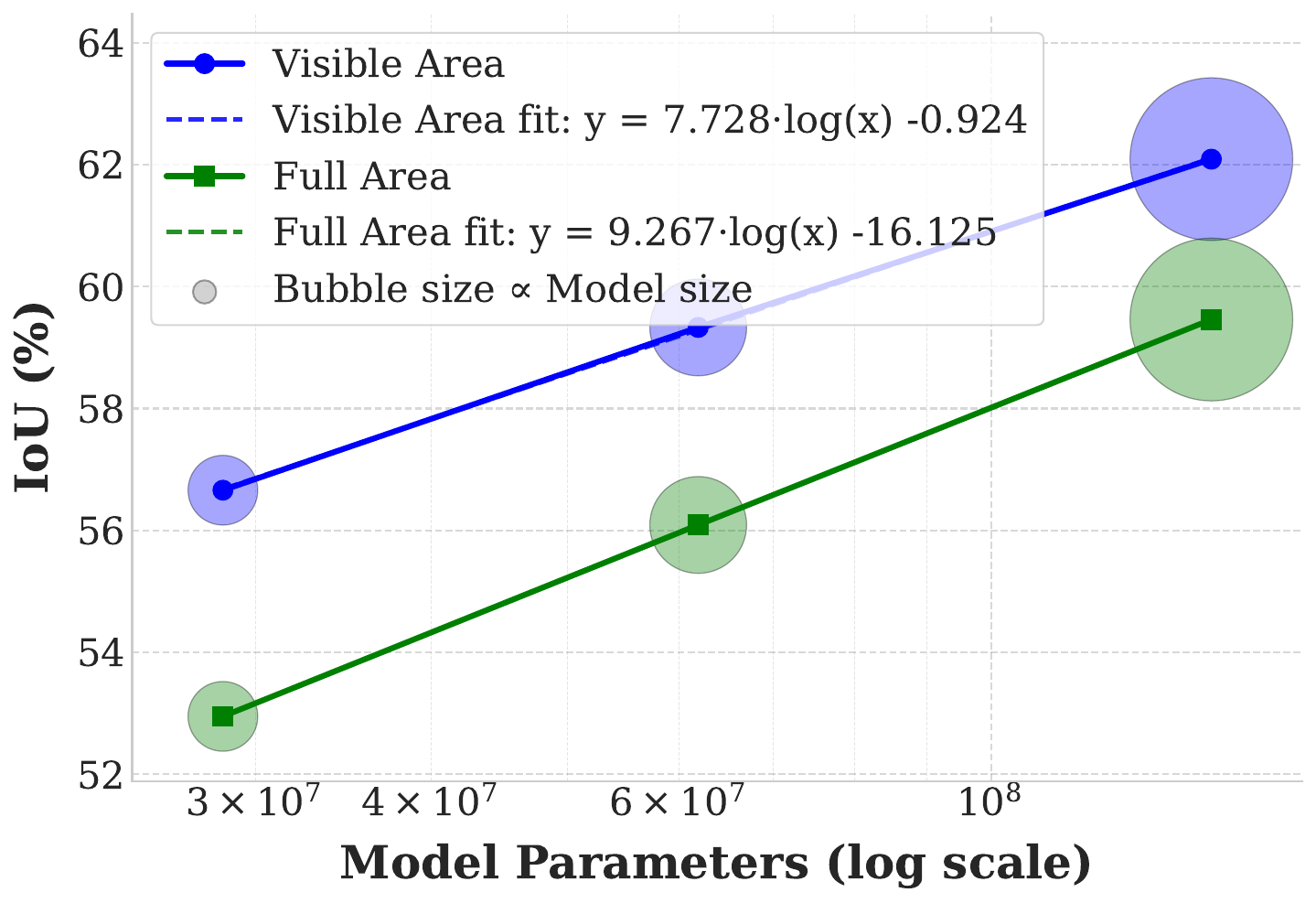}%
            \includegraphics[width=0.33\linewidth]{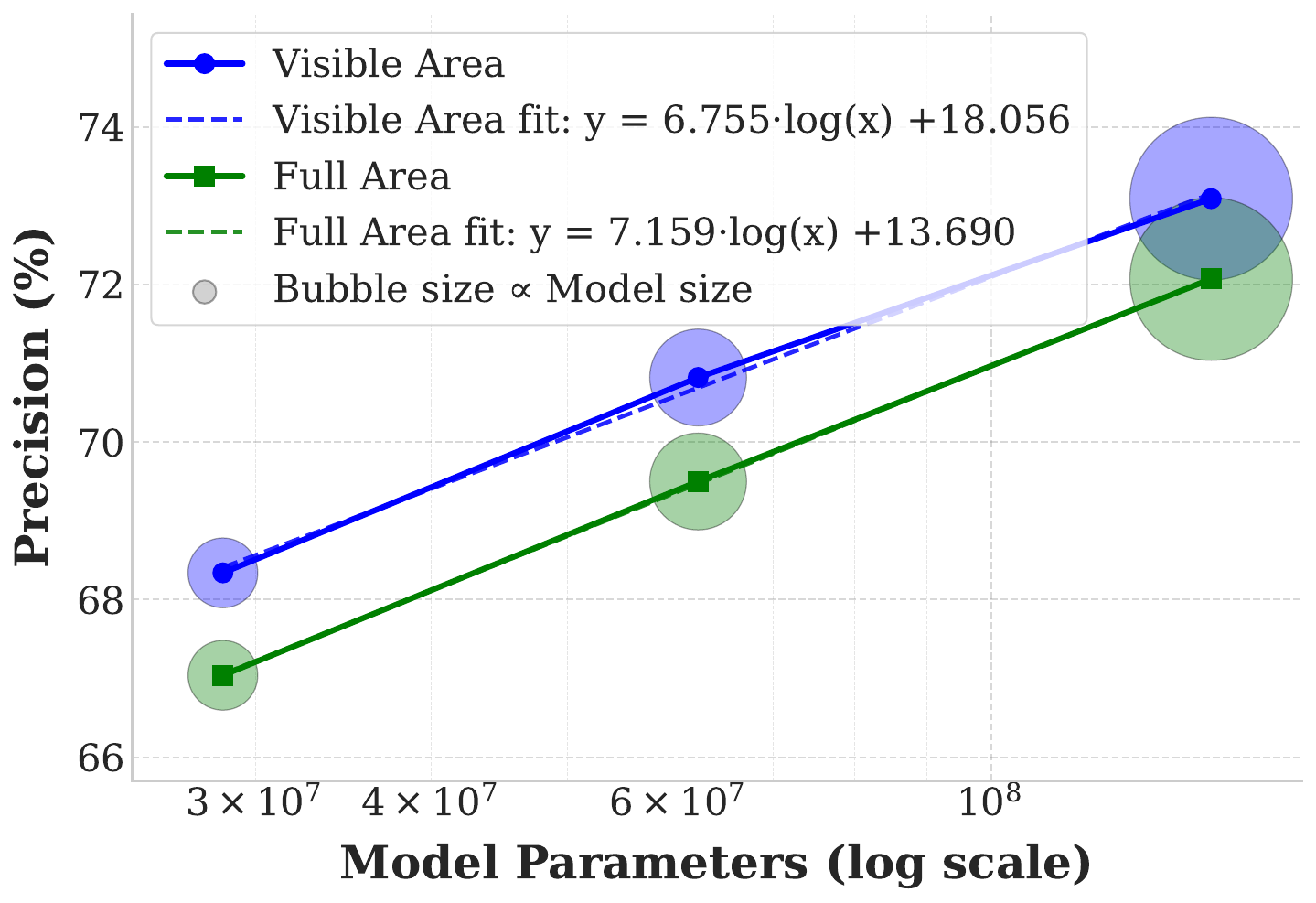}%
            \includegraphics[width=0.33\linewidth]{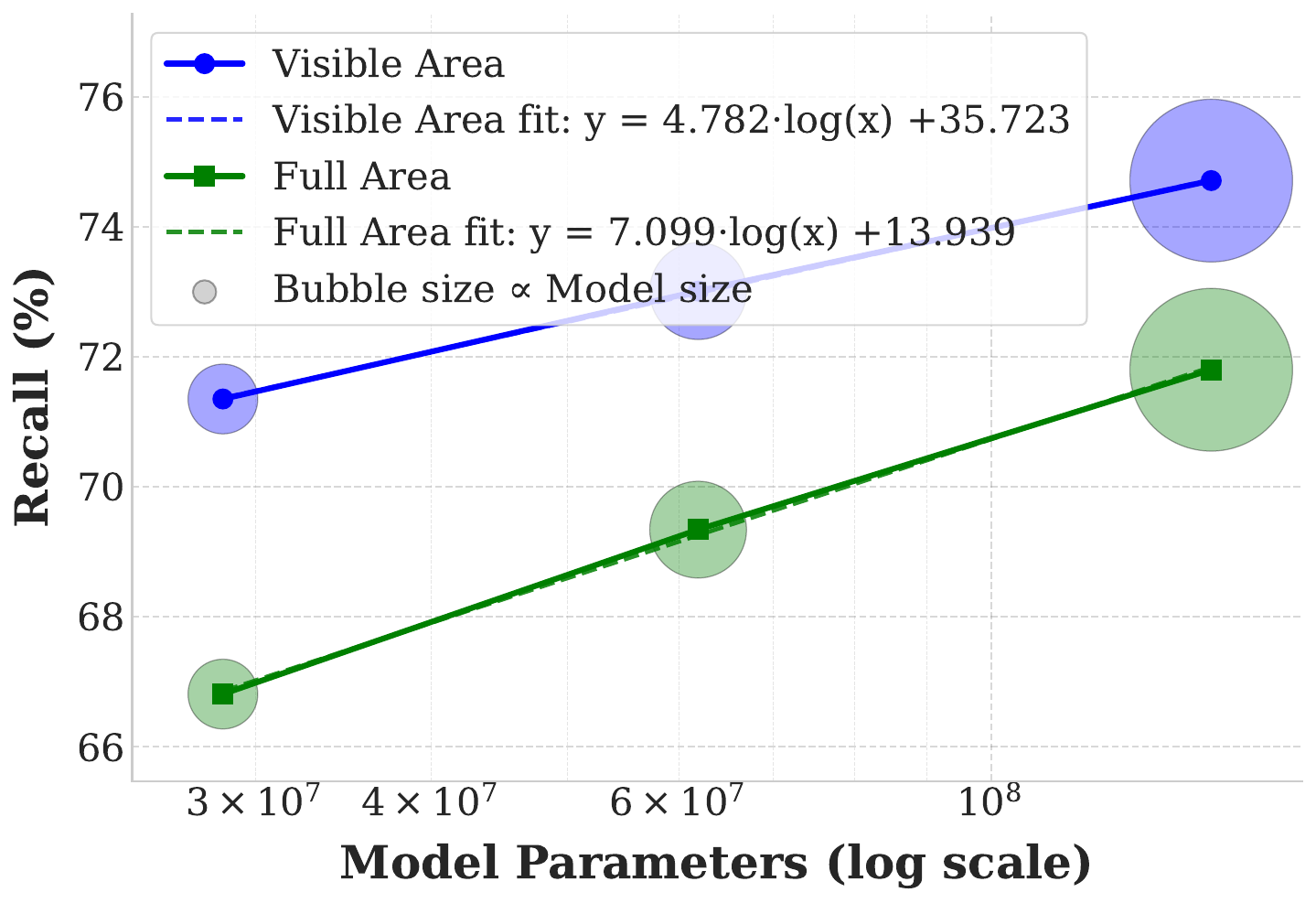}
        \end{subfigure}
        \phantomsubcaption
        \label{subfig:model_size}
    \end{subfigure}

    \vspace{0.5em}

    \begin{subfigure}{0.95\linewidth}
        \centering
        \begin{subfigure}[c]{0.05\linewidth}
            \rotatebox{90}{(b) Data Size}
        \end{subfigure}%
        \begin{subfigure}[c]{0.95\linewidth}
            \includegraphics[width=0.33\linewidth]{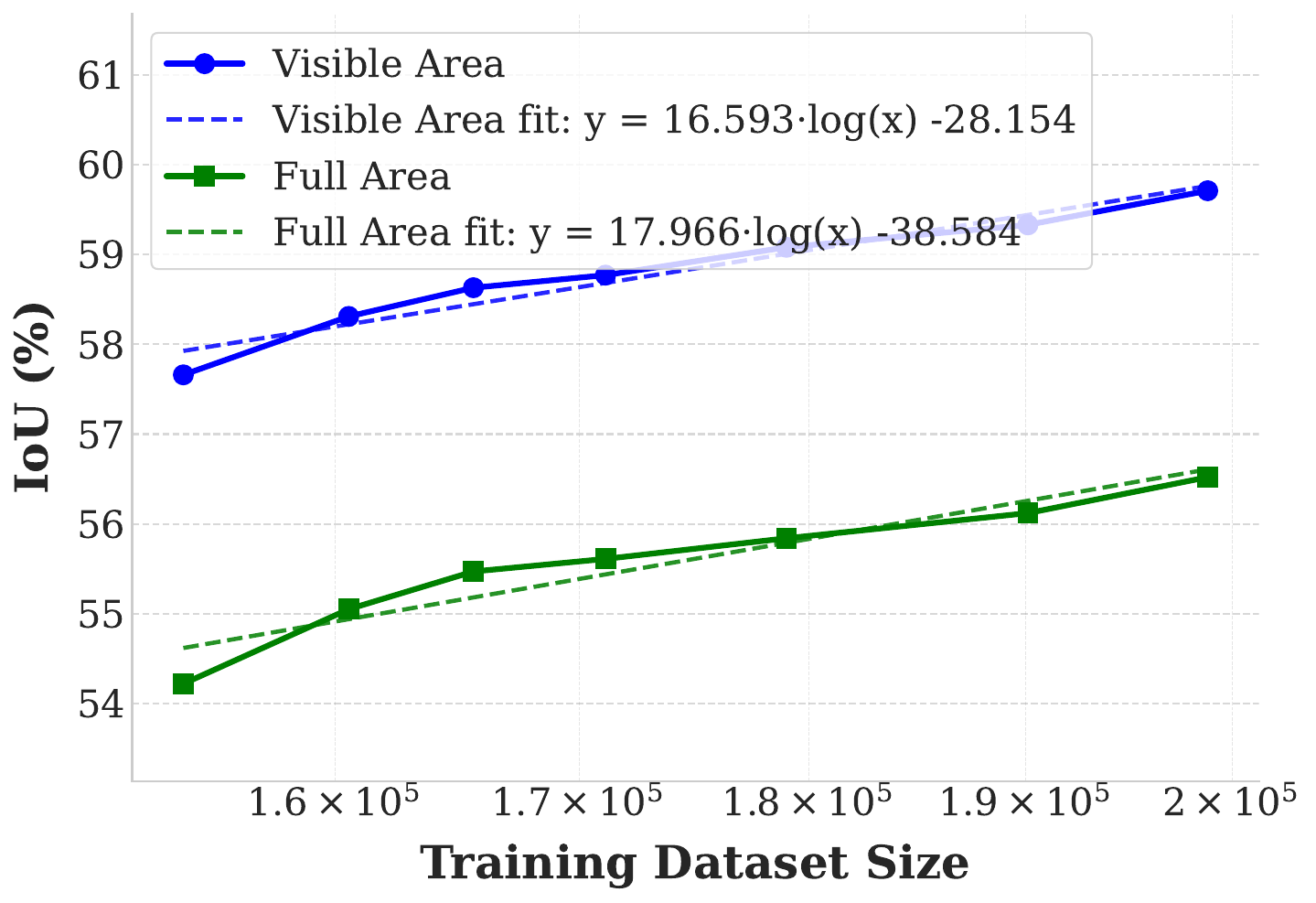}%
            \includegraphics[width=0.33\linewidth]{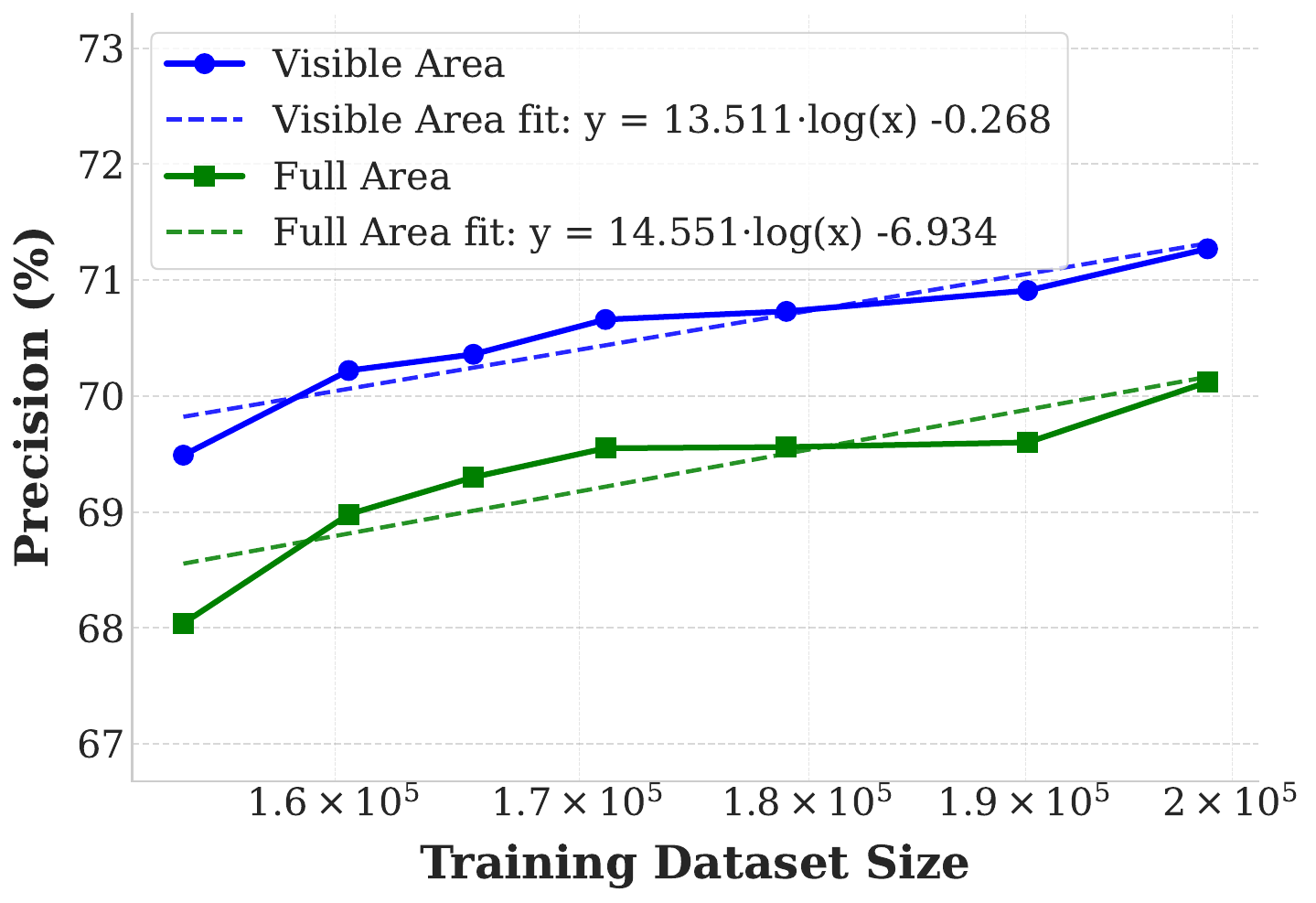}%
            \includegraphics[width=0.33\linewidth]{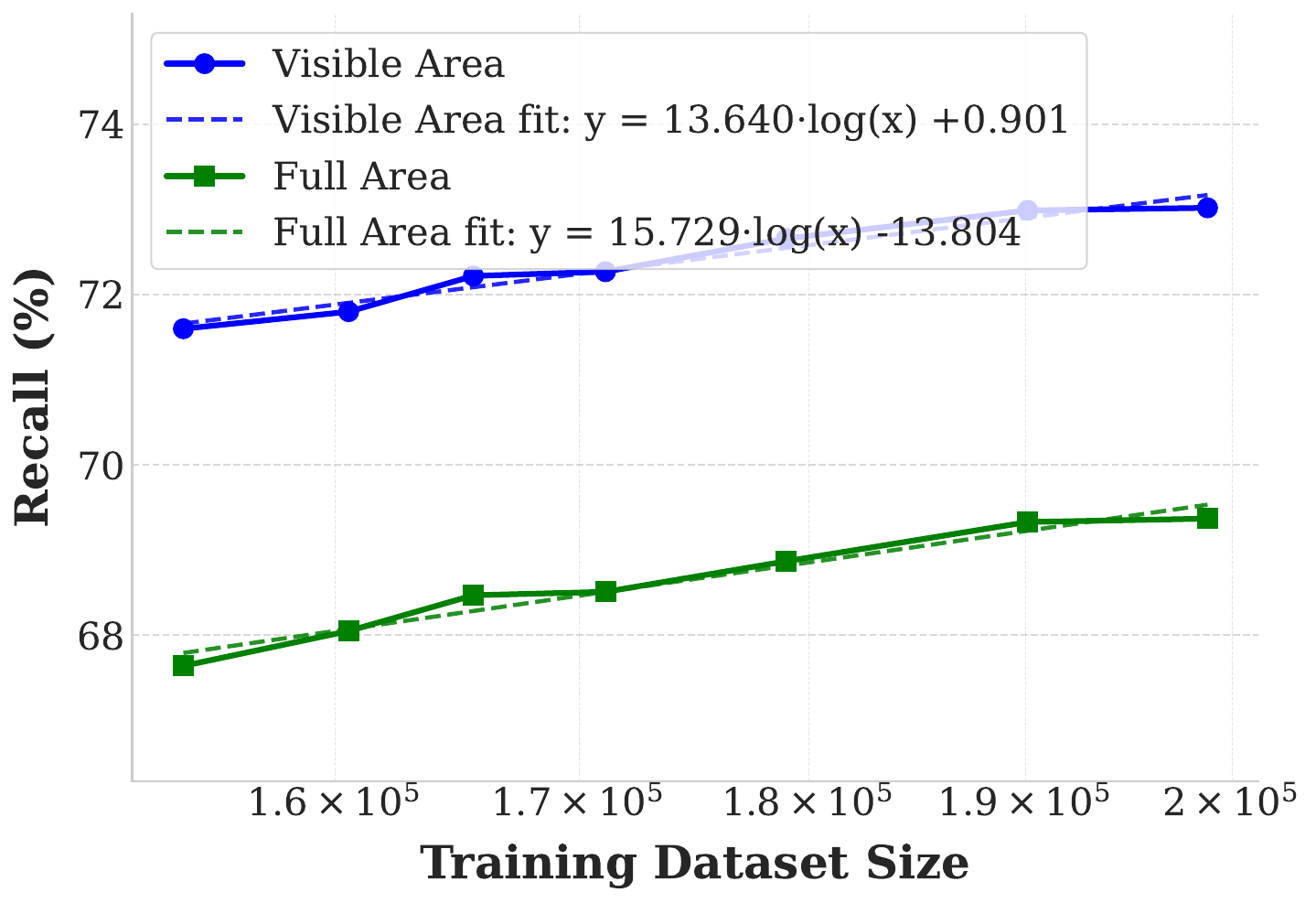}
        \end{subfigure}
        \phantomsubcaption
        \label{subfig:data_size}
    \end{subfigure}

    \vspace{0.5em}
    
    \begin{subfigure}{0.95\linewidth}
        \centering
        \begin{subfigure}[c]{0.05\linewidth}
            \rotatebox{90}{(c) Training Steps}
        \end{subfigure}%
        \begin{subfigure}[c]{0.95\linewidth}
            \includegraphics[width=0.33\linewidth]{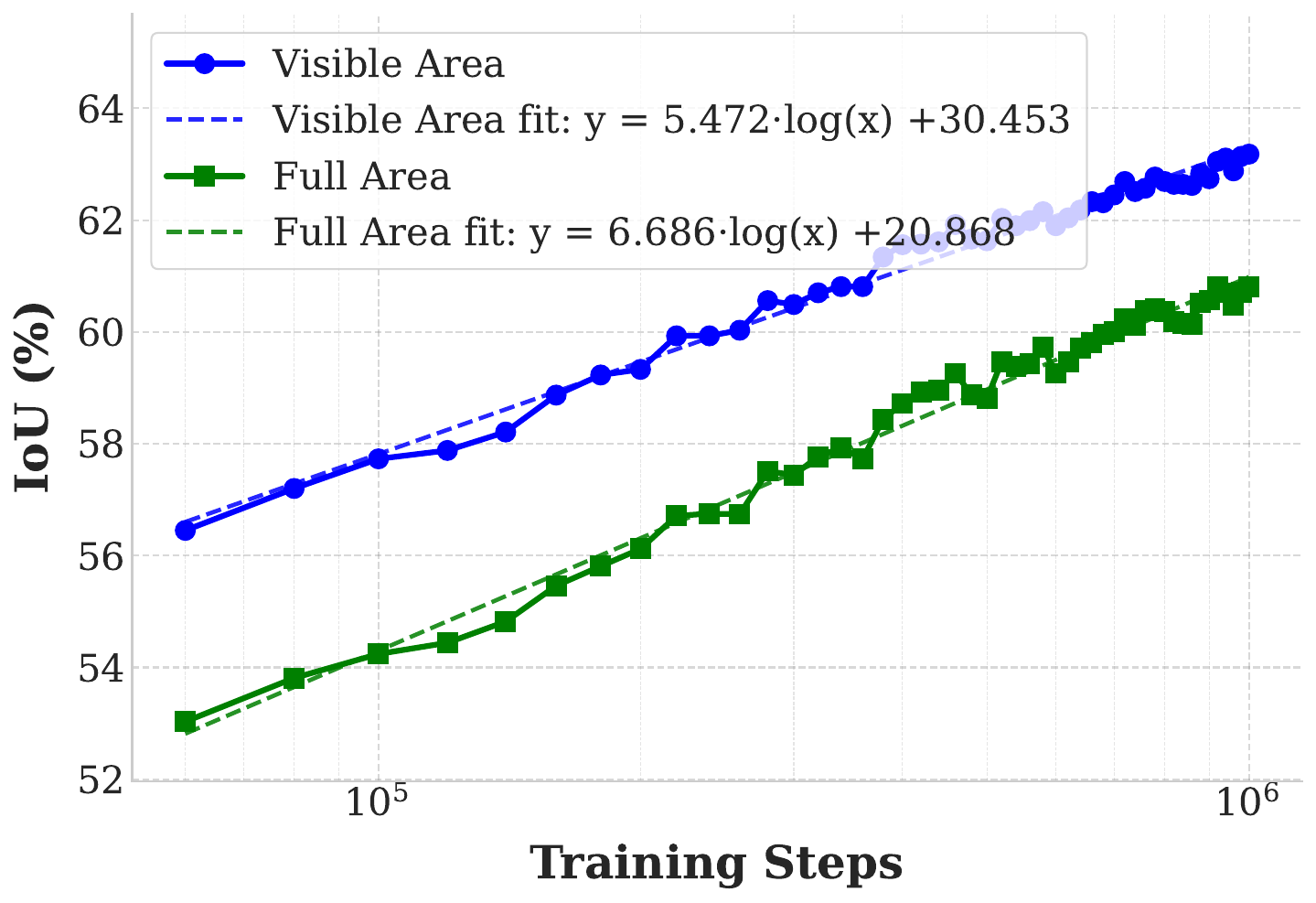}%
            \includegraphics[width=0.33\linewidth]{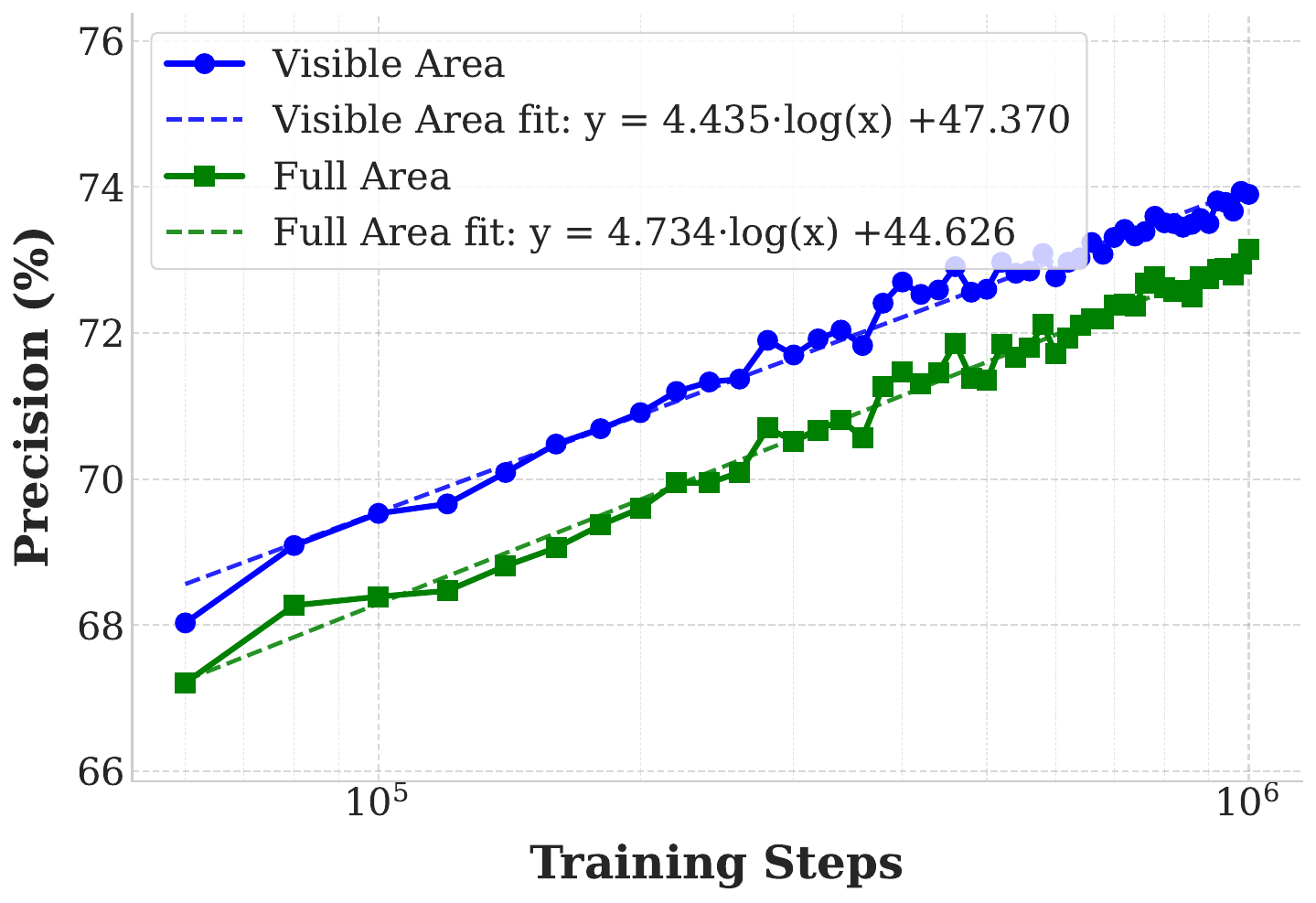}%
            \includegraphics[width=0.33\linewidth]{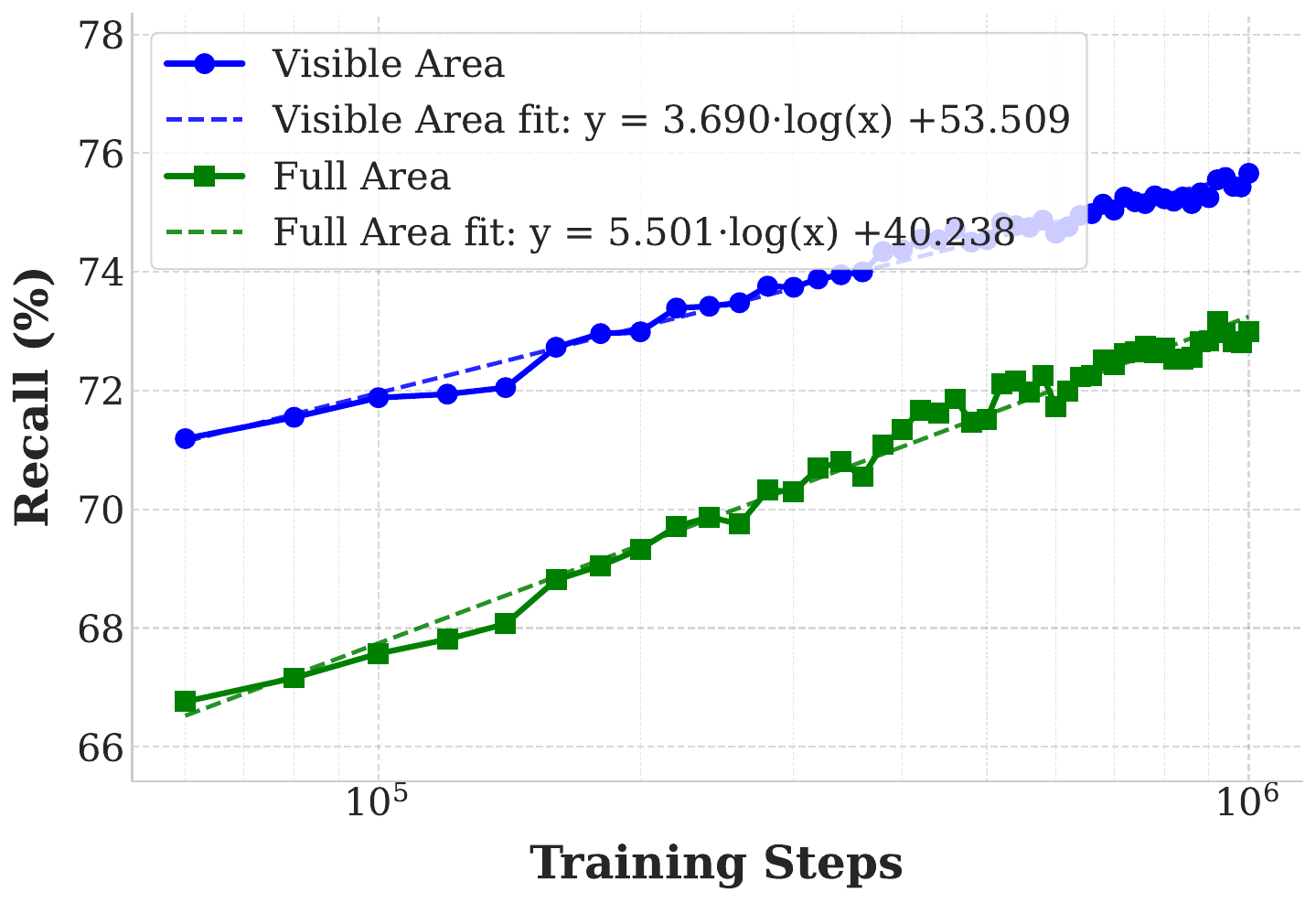}
        \end{subfigure}
        \phantomsubcaption
        \label{subfig:train_steps}
    \end{subfigure}

    \vspace{0.5em}

    \begin{subfigure}{0.95\linewidth}
        \centering
        \begin{subfigure}[c]{0.05\linewidth}
            \rotatebox{90}{(d) Optimization Iterations}
        \end{subfigure}%
        \begin{subfigure}[c]{0.95\linewidth}
            \includegraphics[width=0.33\linewidth]{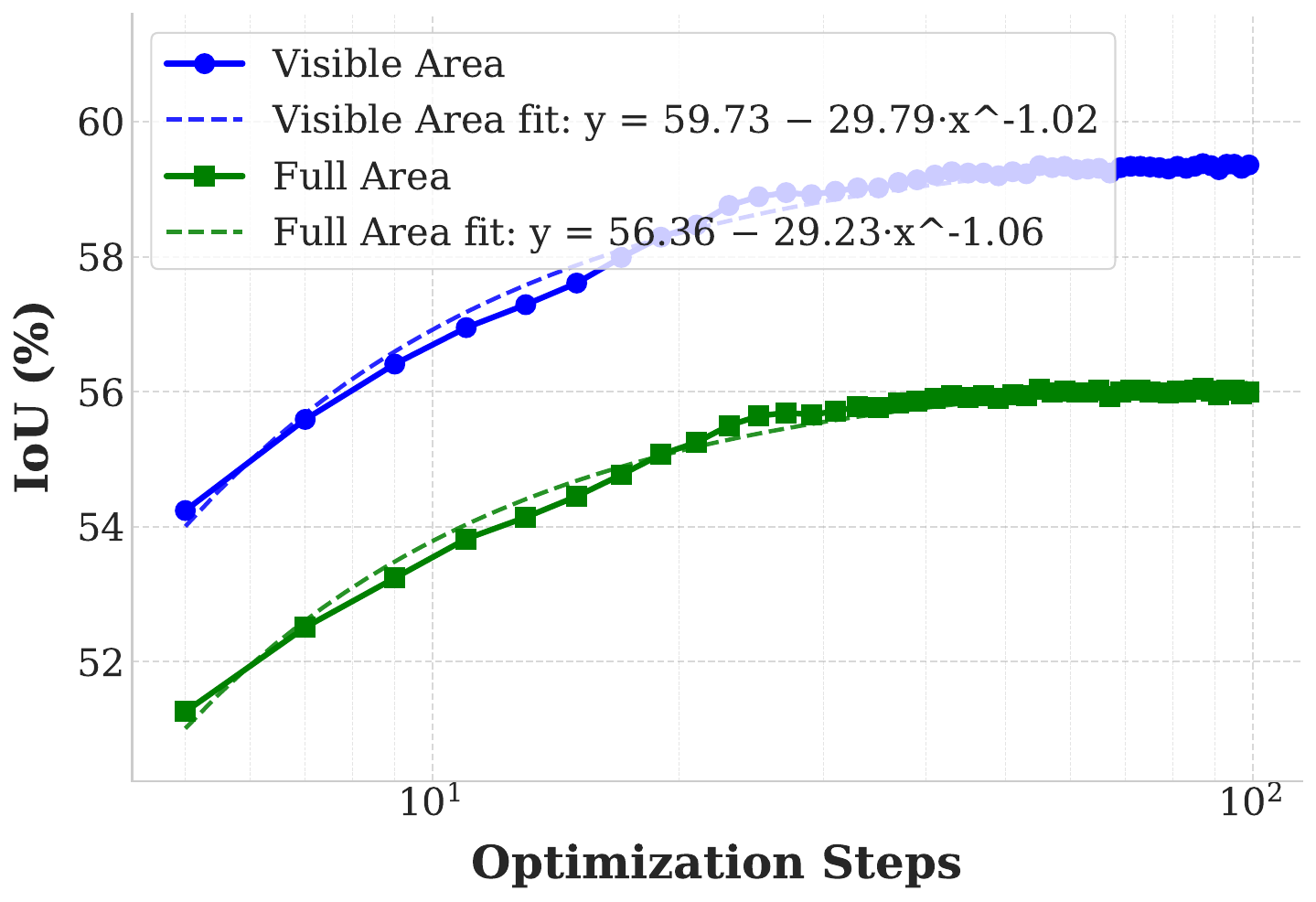}%
            \includegraphics[width=0.33\linewidth]{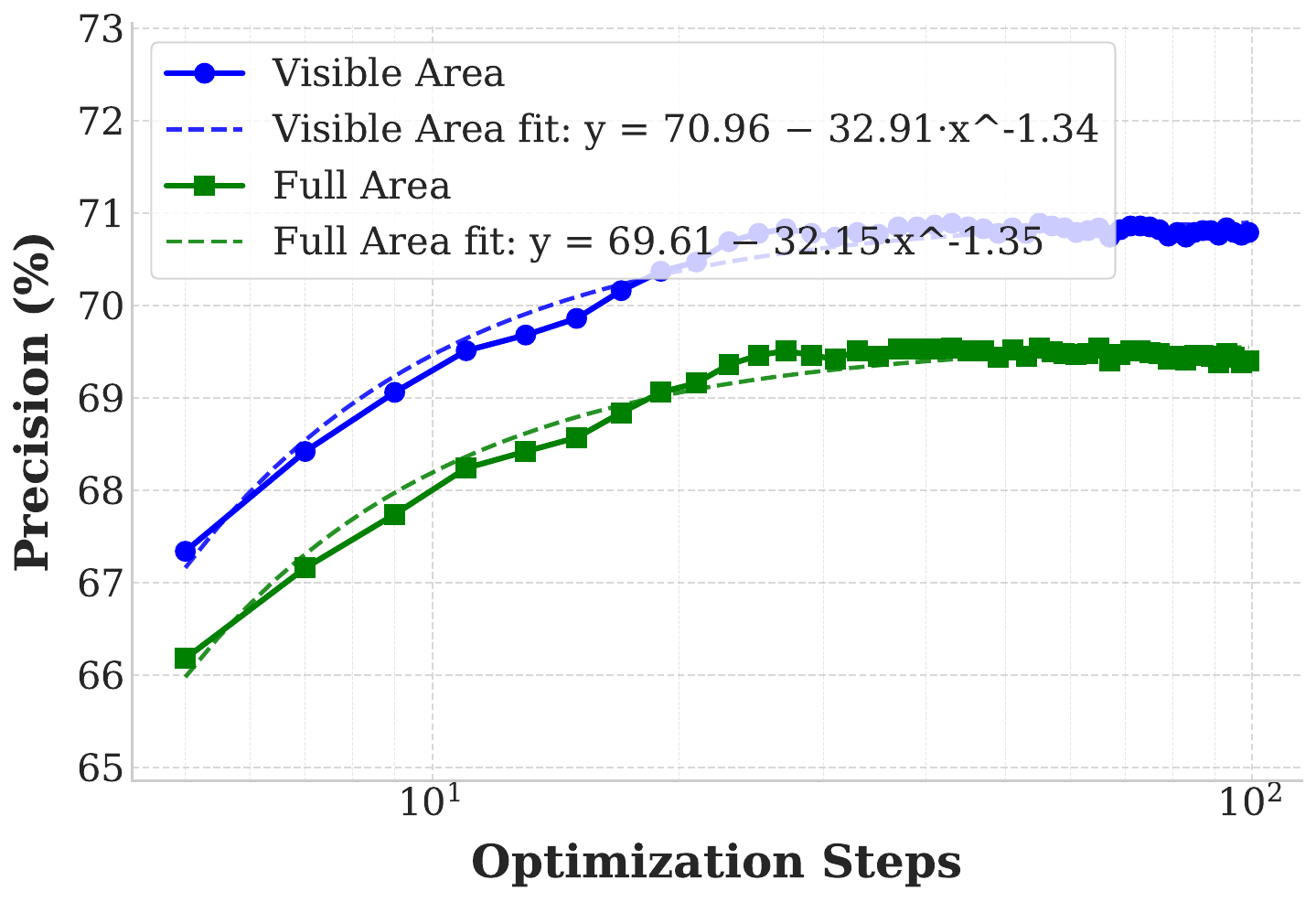}%
            \includegraphics[width=0.33\linewidth]{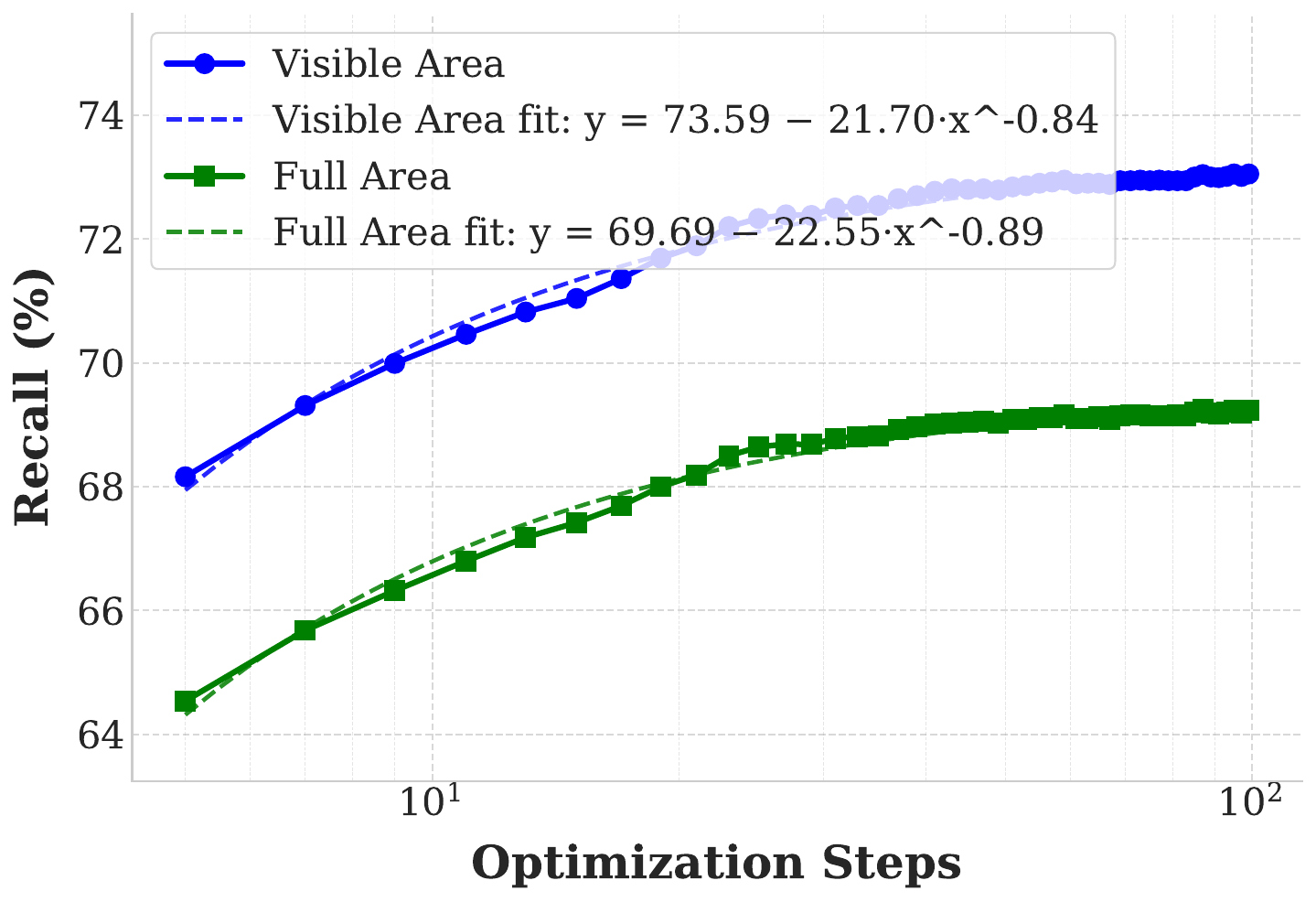}
        \end{subfigure}
        \phantomsubcaption
        \label{subfig:optimization_iterations}
    \end{subfigure}

    \caption{Scaling analysis. Impact of (a) model size, (b) data size, (c) training steps, and (d) optimization iterations on IoU, Precision, and Recall (left to right).}
    \label{fig:scaling_laws}

\end{figure*}

\subsubsection{Denoising Scheduler}
\label{subsec:ab_scheduler}
DDIM~\cite{song2020denoising} and DPM-Solver++~\cite{lu2022dpm} are two popular schedulers. Fig.~\ref{fig:ablation_num_basis} and Fig.~\ref{fig:ablation_denoising_steps} compare them under different \(K\) and step counts. DPM-Solver++ generally performs better, likely due to its higher-order solver yielding more accurate solutions in few steps. For denoising steps, too few steps reduce generation ability. Increasing steps steers the latent more strongly toward high-prior regions (good for generation) but reduces the influence of observation likelihood on the final result (hurting reconstruction), producing a slight degradation (Fig.~\ref{fig:ablation_denoising_steps}). A range of 3--5 steps strikes a good balance.

\subsubsection{Computational Cost Analysis}
Our latent-space optimization provides a unified framework: the optimized variables, generative process variables, and observation-likelihood variables share shapes across samples, enabling batched optimization for multiple samples. Fig.~\ref{fig:runtime_analysis} evaluates efficiency versus batch size, reporting memory and per-map iteration time (total time divided by batch size and iteration number). With batch parallelization, memory grows linearly with batch size, while per-map iteration time decreases and stabilizes at \(11.3\,\mathrm{ms}\) (\(\sim0.66\,\mathrm{s}\) per map in total for 60 iterations).

\subsection{Scalability Analysis for Semantic Mapping}
\label{subsec:scalability}
As discussed in Section~\ref{subsec:paradigms_comp}, our method integrates a data-driven generative model into a classic mapping framework via a prior constraint. This enables improving mapping quality either by strengthening the model during training or by improving optimization during inference.

\begin{figure*}[t]
    \centering
    \includegraphics[width=0.9\linewidth]{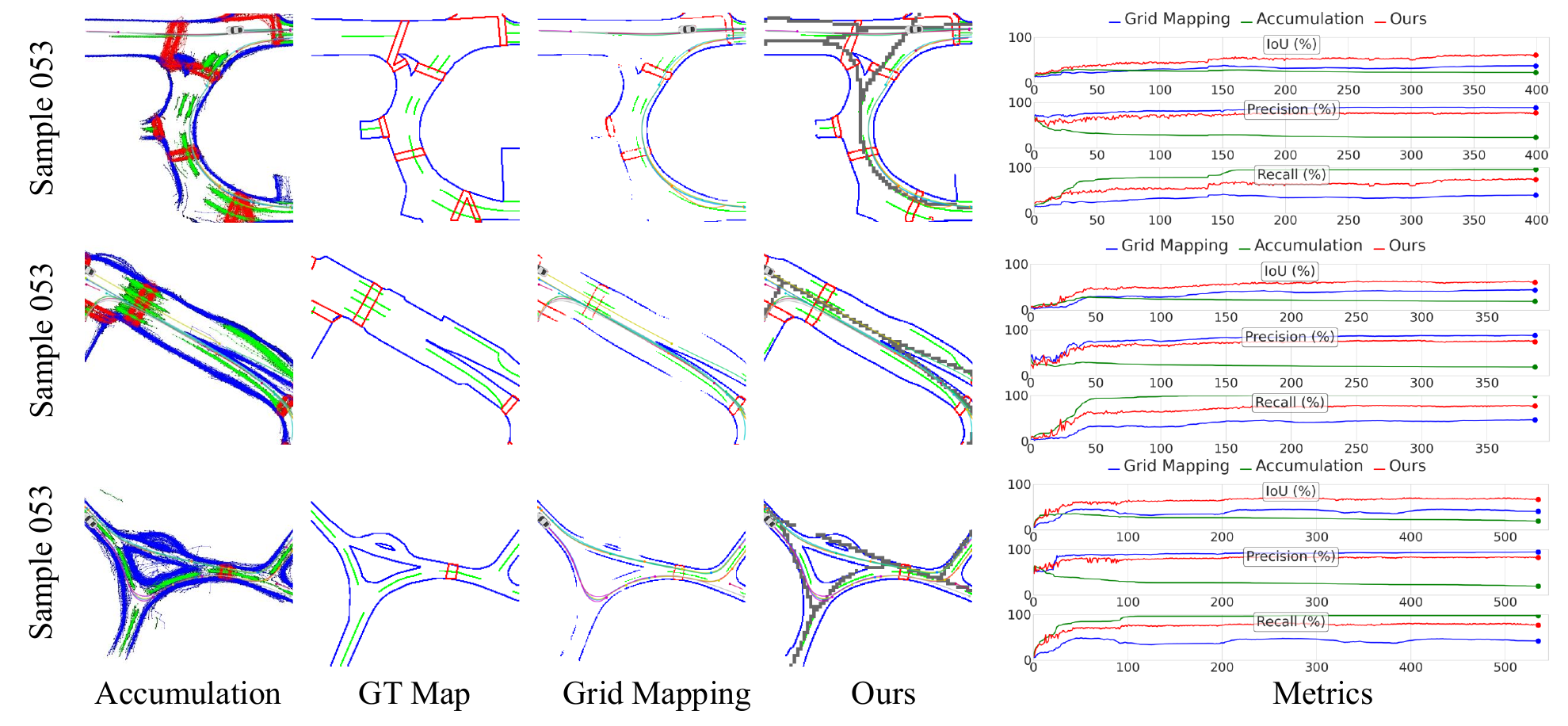} 
    \caption{Performance with crowdsourced data scaling at inference time. Each row shows a sample with five columns (left to right): accumulation of map detections, ground-truth map, grid mapping, our method, and metrics vs observations count. The colorful curves overlaid on the maps represent crowdsourced vehicle trajectories.}\label{fig:scaling_cs_data}
    \vspace{-2mm}
\end{figure*}

\begin{figure}[t]
    \centering
    \includegraphics[width=0.95\linewidth]{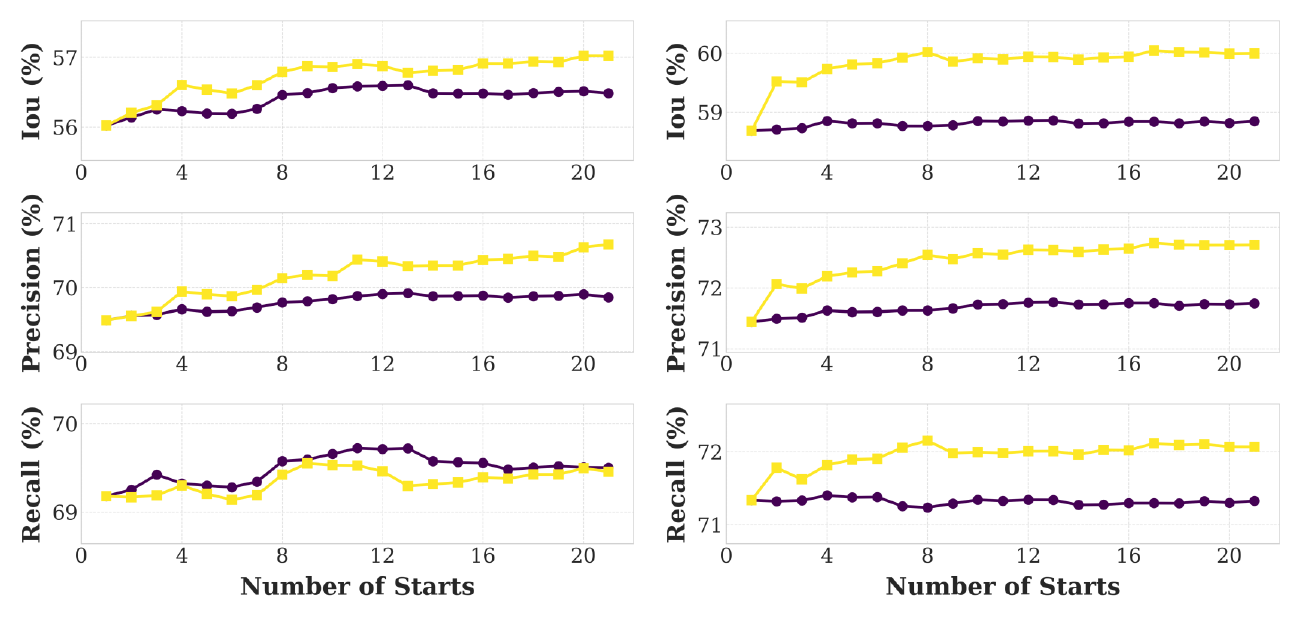}
    \caption{Performance with the number of starts scaling at inference time. \textcolor{yellow}{Yellow} and \textcolor{violet}{purple} represent our proposed posterior-based selection and likelihood-only selection, respectively. Left and right plots represent models with 200{,}000 and 400{,}000 training steps, respectively.}
    \label{fig:scaling_multi_start}
\end{figure} 

\subsubsection{Training-Phase Scalability}
\label{subsec:scalability_training}
We examine how performance scales with dataset size, model size, training duration, and conditioning. To vary dataset size, we adjust the minimum distance between training and test sample centers from 5 to 40 m. Model capacity is varied via the number of transformer blocks and embedding dimension. Figs.~\ref{subfig:model_size}, \ref{subfig:data_size}, and \ref{subfig:train_steps} show steady gains with larger models, larger datasets, and more training—consistent with scaling trends in image generation~\cite{peebles2023} and language models~\cite{kaplan2020scaling}. Although nuScenes limits the absolute number of HD maps, the trend remains clear. Using center distance as a proxy for dataset size reflects our aim: when training data sufficiently covers more HD patterns (ideally including those in test regions), the learned map prior becomes more effective, thereby improving generative mapping performance. Importantly, our training does not require paired crowdsourced/HD maps, improving data accessibility for scaling the dataset up.

Furthermore, while SD maps are intractable for classic mapping, they can condition generation to enhance mapping. Quantitative effects are analyzed in Section~\ref{exp:csm_quality}; here we focus on controllability and practical issues.

For controllability, Fig.~\ref{fig:sd_control} shows three samples under different SD-map conditions. The figure indicates: (i) changes to SD maps outside observed areas do not impair robust reconstruction within observed areas; (ii) generation in unobserved areas remains consistent with reconstruction in connected regions; and (iii) the generated maps are novel (not present in the training data), which means they are not simple retrievals of training samples. Beyond controllability, SD-conditioned generation offers plausible map extension, useful for applications such as traffic-scene simulation~\cite{ren2025cosmos} and simulation of unseen areas~\cite{barsalou2009simulation} when planning. 

On the other hand, in practice, SD maps may misalign with crowdsourced observations due to localization errors or outdated maps. Fig.~\ref{fig:invalid_sd} evaluates reconstruction and generation under three cases: missing (No SD), entirely wrong (Wr. SD), and partially wrong (Part. Wr.). When missing or entirely wrong, the method relies more on observations—maintaining reconstruction in observed areas but degrading on the generation in unobserved regions. Partially wrong SD maps yield different behaviors: (row 1) an extra, non-existent road link—non-conflicting SD segments still guide generation; (row 2) small rotation/translation—minor geometric deviations have limited impact because the SD map serves as abstract topological condition; (row 3) missing links—the model cannot disambiguate SD incompleteness from observation faults, yielding errors. Given such variability, removing SD conditioning may be preferable when SD reliability is questionable in practice.

\begin{figure*}[t]
    \centering
    \includegraphics[width=\linewidth]{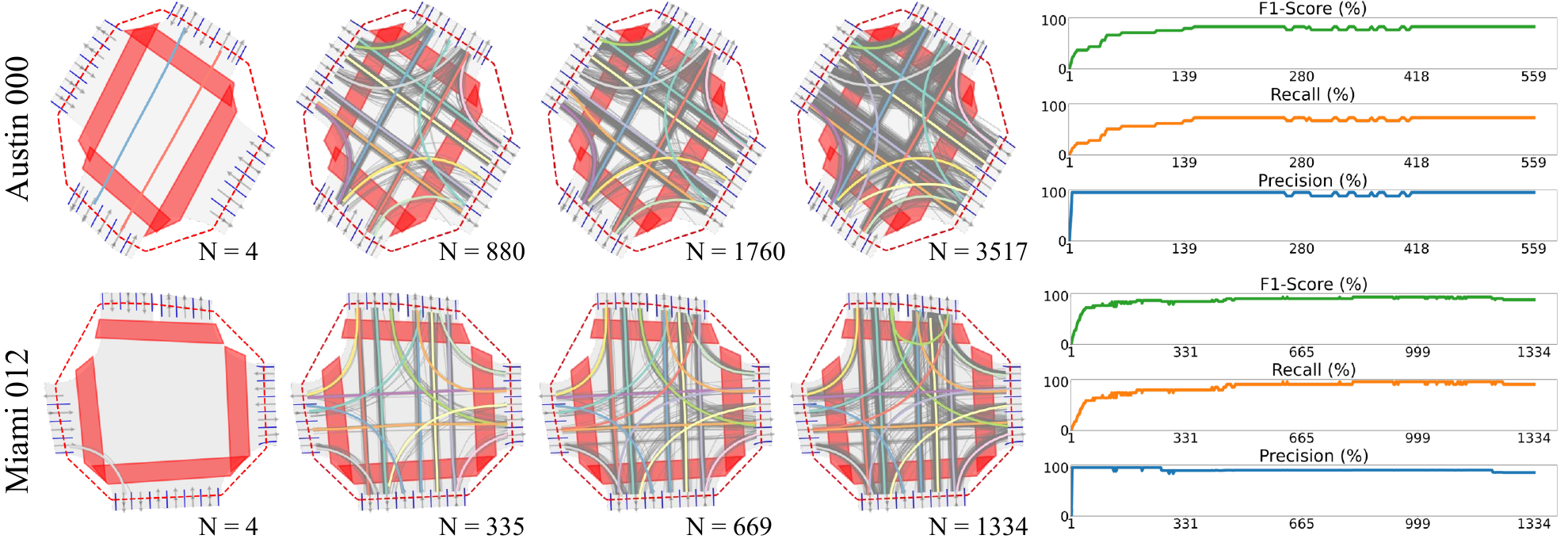}
    \caption{Scaling on trajectory data. Each row shows a region with extracted centerlines under four different trajectory counts (left to right: increasing trajectory counts) and metrics.}
    \label{fig:scaling_traj}
\end{figure*}

\subsubsection{Inference-Phase Scalability}
\label{subsec:scalability_inference}
We evaluate inference-time scalability with respect to the number of optimization iterations, the amount of observation data, and multi-start strategies. 

Fig.~\ref{subfig:optimization_iterations} shows quality improves with more iterations. Iterative refinement is standard in classic mapping but absent in feed-forward mapping~\cite{wang2025vggt,keetha2025mapanything}; our approach enables projected gradient descent in latent space to achieve iterative map refinement.

Fig.~\ref{fig:scaling_cs_data} shows performance versus crowdsourced coverage. Direct accumulation highlights the severe noise in crowdsourced data, posing a significant challenge for classic mapping. Methods like grid mapping have limited performance and no generation ability in unobserved areas. Despite increased noise with more data, our prior constraint sustains performance. At very sparse coverage, for example at the initial stage, metrics fluctuate due to stochastic generation in unobserved regions; as coverage grows, metrics stabilize.

Fig.~\ref{fig:scaling_multi_start} validates the multi-start strategy: performance improves with more initial latents (we test \(N_s\) from 1 to 20). Posterior-based selection outperforms likelihood-only selection, highlighting that the plausibility indicated by the prior score is important for evaluating mapping quality. Furthermore, with more training steps, posterior scoring improves as noise prediction accuracy increases.

\begin{table}[t]
    \centering
    \caption{Comparison with HD-map centerlines. All metrics are in percentage (\%). P: Precision, R: Recall, F1: F1-score. Higher is better ($\uparrow$).}
    \label{tab:topo-hd}
    \resizebox{\linewidth}{!}{%
    \begin{tabular}{@{}ll ccc ccc ccc@{}}
    \toprule
    \multirow{2}{*}{Type} &
      \multirow{2}{*}{Method} &
      \multicolumn{3}{c}{Threshold@1 m} &
      \multicolumn{3}{c}{Threshold@2 m} &
      \multicolumn{3}{c}{Threshold@3 m} \\ 
    \cmidrule(lr){3-5} \cmidrule(lr){6-8} \cmidrule(l){9-11}
     &
       &
      \textbf{P} $\uparrow$ &
      \textbf{R} $\uparrow$ &
      \textbf{F1} $\uparrow$ &
      \textbf{P} $\uparrow$ &
      \textbf{R} $\uparrow$ &
      \textbf{F1} $\uparrow$ &
      \textbf{P} $\uparrow$ &
      \textbf{R} $\uparrow$ &
      \textbf{F1} $\uparrow$ \\
    \midrule
    \multirow{3}{*}{Straight} & FlowMap~\cite{ding2023flowmap} & 64.5 & \cellcolor{blue!30}{79.4} & 71.2 & 69.2 & \cellcolor{blue!30}{85.1} & 76.3 & 70.1 & \cellcolor{blue!30}{86.2} & 77.3 \\
                              & K-MEANS & \cellcolor{blue!10}{78.0} & \cellcolor{blue!10}{77.4} & \cellcolor{blue!10}{77.7} & \cellcolor{blue!10}{83.8} & \cellcolor{blue!10}{83.1} & \cellcolor{blue!10}{83.4} & \cellcolor{blue!10}{84.6} & \cellcolor{blue!10}{84.0} & \cellcolor{blue!10}{84.3} \\
                              & Ours  & \cellcolor{blue!30}{82.8} & 77.0 & \cellcolor{blue!30}{79.8} & \cellcolor{blue!30}{88.0} & 81.9 & \cellcolor{blue!30}{84.8} & \cellcolor{blue!30}{88.9} & 82.7 & \cellcolor{blue!30}{85.7} \\
    \midrule
    \multirow{3}{*}{Left}     & FlowMap~\cite{ding2023flowmap} & 39.4 & \cellcolor{blue!30}{45.3} & 42.2 & 59.4 & \cellcolor{blue!30}{68.3} & 63.5 & 63.8 & \cellcolor{blue!30}{73.3} & 68.2 \\
                              & K-MEANS & \cellcolor{blue!10}{47.8} & \cellcolor{blue!10}{42.3} & \cellcolor{blue!30}{44.8} & \cellcolor{blue!10}{74.7} & \cellcolor{blue!10}{66.1} & \cellcolor{blue!10}{70.1} & \cellcolor{blue!10}{81.2} & \cellcolor{blue!10}{71.8} & \cellcolor{blue!10}{76.2} \\
                              & Ours  & \cellcolor{blue!30}{50.0} & 40.2 & \cellcolor{blue!10}{44.5} & \cellcolor{blue!30}{79.4} & 63.7 & \cellcolor{blue!30}{70.7} & \cellcolor{blue!30}{86.7} & 69.6 & \cellcolor{blue!30}{77.2} \\
    \midrule
    \multirow{3}{*}{Right}    & FlowMap~\cite{ding2023flowmap} & 62.4 & \cellcolor{blue!30}{67.3} & 64.8 & 74.5 & \cellcolor{blue!30}{80.3} & 77.3 & 76.0 & \cellcolor{blue!30}{82.0} & 78.9 \\
                              & K-MEANS & \cellcolor{blue!10}{70.2} & \cellcolor{blue!10}{64.2} & \cellcolor{blue!10}{67.0} & \cellcolor{blue!10}{85.5} & \cellcolor{blue!10}{78.2} & \cellcolor{blue!10}{81.7} & \cellcolor{blue!10}{87.5} & \cellcolor{blue!10}{80.0} & \cellcolor{blue!10}{83.6} \\
                              & Ours  & \cellcolor{blue!30}{73.0} & 62.0 & \cellcolor{blue!30}{67.1} & \cellcolor{blue!30}{89.0} & 75.6 & \cellcolor{blue!30}{81.8} & \cellcolor{blue!30}{91.8} & 78.0 & \cellcolor{blue!30}{84.4} \\
    \midrule
    \multirow{3}{*}{U-Turn}   & FlowMap~\cite{ding2023flowmap} & 25.0 & \cellcolor{blue!30}{16.7} & 20.0 & \cellcolor{blue!10}{75.0} & \cellcolor{blue!30}{50.0} & \cellcolor{blue!30}{60.0} & \cellcolor{blue!10}{91.7} & \cellcolor{blue!30}{61.1} & \cellcolor{blue!30}{73.3} \\
                              & K-MEANS & \cellcolor{blue!10}{33.3} & \cellcolor{blue!30}{16.7} & \cellcolor{blue!10}{22.2} & \cellcolor{blue!30}{77.8} & 38.9 & \cellcolor{blue!10}{51.9} & \cellcolor{blue!30}{100.0} & 50.0 & \cellcolor{blue!10}{66.7} \\
                              & Ours  & \cellcolor{blue!30}{42.9} & \cellcolor{blue!30}{16.7} & \cellcolor{blue!30}{24.0} & 71.4 & 27.8 & 40.0 & \cellcolor{blue!30}{100.0} & 38.9 & 56.0 \\
    \midrule
    \multirow{3}{*}{Overall} & FlowMap~\cite{ding2023flowmap} & 57.3 & \cellcolor{blue!30}{66.9} & 61.8 & 67.9 & \cellcolor{blue!30}{79.2} & 73.1 & 69.9 & \cellcolor{blue!30}{81.6} & 75.3 \\
                              & K-MEANS & \cellcolor{blue!10}{68.4} & \cellcolor{blue!10}{64.3} & \cellcolor{blue!10}{66.3} & \cellcolor{blue!10}{81.9} & \cellcolor{blue!10}{77.1} & \cellcolor{blue!10}{79.4} & \cellcolor{blue!10}{84.5} & \cellcolor{blue!10}{79.6} & \cellcolor{blue!10}{82.0} \\
                              & Ours & \cellcolor{blue!30}{72.2} & 63.0 & \cellcolor{blue!30}{67.3} & \cellcolor{blue!30}{86.1} & 75.2 & \cellcolor{blue!30}{80.3} & \cellcolor{blue!30}{89.2} & 77.8 & \cellcolor{blue!30}{83.1} \\
    \bottomrule
    \end{tabular}%
    }
    \end{table}
\begin{figure}[t]
    \centering
    \includegraphics[width=0.8\linewidth]{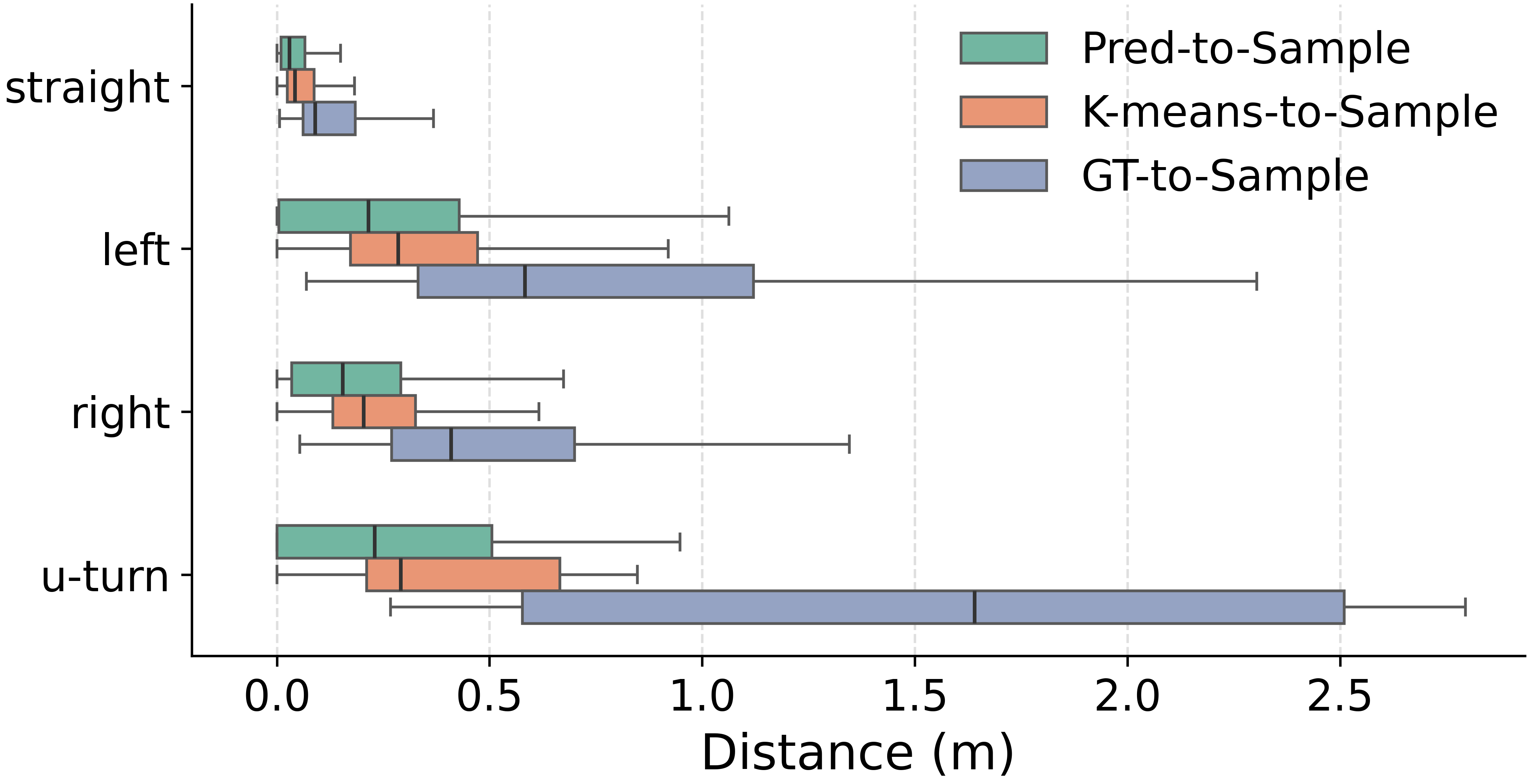}
    \caption{CDTW distance to human trajectories.}
    \label{fig:dist_to_human}
\end{figure}

\begin{figure}[t]
    \centering
    \includegraphics[width=0.8\linewidth]{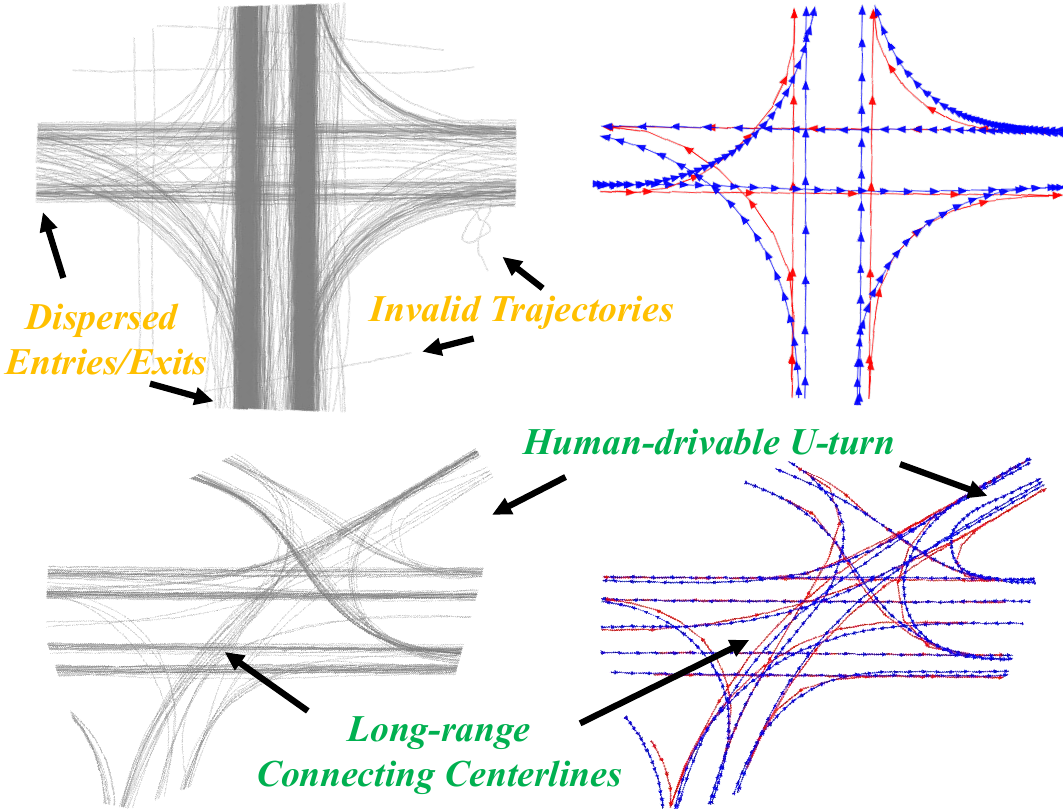}
    \caption{Topology mapping visualization across two scenes. Top row: raw trajectories (\textcolor{gray}{gray}) and centerline comparison in Pittsburgh 001 (HD annotations: \textcolor{red}{red}; ours: \textcolor{blue}{blue}). Bottom row: the same for Pittsburgh 038.}
    \label{fig:topo_gt_pred}
\end{figure}

\subsection{Topological Mapping Experiments}
\label{subsec:topo_exp}

\subsubsection{Task Setup and Evaluation}
We select 365 complex intersections from the \emph{Argoverse 2} motion-forecasting dataset. On average, there are 718.9 human driving trajectories and 12.2 centerlines per intersection; the average trajectory length is \(27.9\,\mathrm{m}\). We categorize topologies into four types—straight, left turn, right turn, and U-turn—to measure performance across maneuvers. The k-medoids cluster count is \(k=25\). The CDTW threshold for Non-Maximum Suppression (NMS), \(\tau_{\text{nms}}\), is \(3\,\mathrm{m}\). For trajectory confidence weighting, we set \(\lambda_{\text{obs}}=\text{0.2}\), \(\lambda_{\text{head}}=\text{0.3}\), \(\lambda_{\text{smooth}}=\text{0.5}\), heading threshold \(\tau_{\theta}=\text{4.0}^\circ\), and smoothness parameter \(\sigma_r=\text{0.2}\). For kinematic refinement, we use time step \(dt=\text{0.1}\) and control regularization \(\lambda_{\text{ctrl}}=\text{10}\).

\subsubsection{Topology Performance Evaluation}

We use manually annotated HD-map centerlines as ground truth (GT). Continuous Dynamic Time Warping (CDTW) measures similarity between predictions and GT. A detection is a true positive if its CDTW distance to the corresponding GT is below a threshold (1/2/3 m). We report Precision, Recall, and F1-score. 

Table~\ref{tab:topo-hd} compares against FlowMap~\cite{ding2023flowmap} (self-implementation) and k-means. Our method achieves the best F1 across almost all topology types and thresholds, with a few exceptions (e.g., Left@1m and U-Turn@2m/3m). FlowMap attains higher recall but lower precision: it groups by entry/exit and then clusters within groups, which can yield redundant adjacent centerlines and sensitivity to dispersed entry/exit as trajectories increase (see Fig.~\ref{fig:topo_gt_pred}). In contrast, we cluster all trajectories directly and use NMS to prune redundancy, facilitating convergence to a global optimum. K-means averages trajectories to form centerlines, making it sensitive to outliers. Moreover, Fig.~\ref{fig:scaling_traj} shows that increasing trajectory data consistently improves mapping quality, highlighting the scalability of our topological mapping system. 


In addition to GT comparison, we assess human-driving consistency: for each true positive, we compute CDTW distances from both predicted and annotated centerlines to all human trajectories in the cluster. Fig.~\ref{fig:dist_to_human} shows our centerlines are closer to actual driving than k-means or manual annotation because our clustering centers are derived directly from real trajectory data rather than averaging or rules. Fig.~\ref{fig:topo_gt_pred} also shows human-drivable U-turns and long-range connectors extracted by our method alongside human labels.


\section{Conclusion and Future Work}
\label{sec:conclusion}
We presented CSMapping, a scalable crowdsourced mapping system that leverages perceived road semantics and vehicle trajectories to construct both semantic and topological maps. For semantic mapping, an HD-map diffusion model provides a strong prior while vectorized mapping supplies high-quality initialization; constraining latent-space optimization to a Gaussian manifold yields accurate and complete maps. Extensive experiments demonstrate scalability along both training axes (larger datasets, models, training steps, and conditional control) and inference axes (optimization iterations, observation data, and multi-start strategy), and show practical benefits for enhancing online detection. For topological mapping, we introduce confidence-weighted k-medoids clustering over raw trajectories with kinematic refinement, achieving high-quality centerlines that remain robust as trajectory volume and noise increase.

Looking ahead, we identify four promising directions. (i) Computational efficiency: reduce optimization and denoising costs via consistency models~\cite{song2023consistency} or distillation toward few-/single-step generation~\cite{sauer2024adversarial}. (ii) Representation: transition from raster to vectorized instance representations~\cite{monninger2025mapdiffusion} to shrink the state space and improve geometric precision, while handling variable-size instance sets during optimization. (iii) Conditioning: exploit additional modalities (e.g., perspective segmentation, satellite imagery) and develop reliability-aware conditioning to mitigate misaligned or outdated SD maps. (iv) Self-improving autonomy: close the loop so that validated maps harvested during operation continually refine the generative model, which in turn enhances downstream mapping and perception.

\bibliographystyle{IEEEtran}
\bibliography{root}

\clearpage
\label{appendix:startpage}
\appendix
This supplementary material provides detailed derivations, architectural specifications, and additional experimental results that complement the main paper. If accepted, this document will be made publicly available via arXiv or the project's open-source repository.

\subsection{Continuous Dynamic Time Warping (CDTW)}
\label{app:cdtw}
\begin{figure}[h]
    \centering
    \includegraphics[width=\columnwidth]{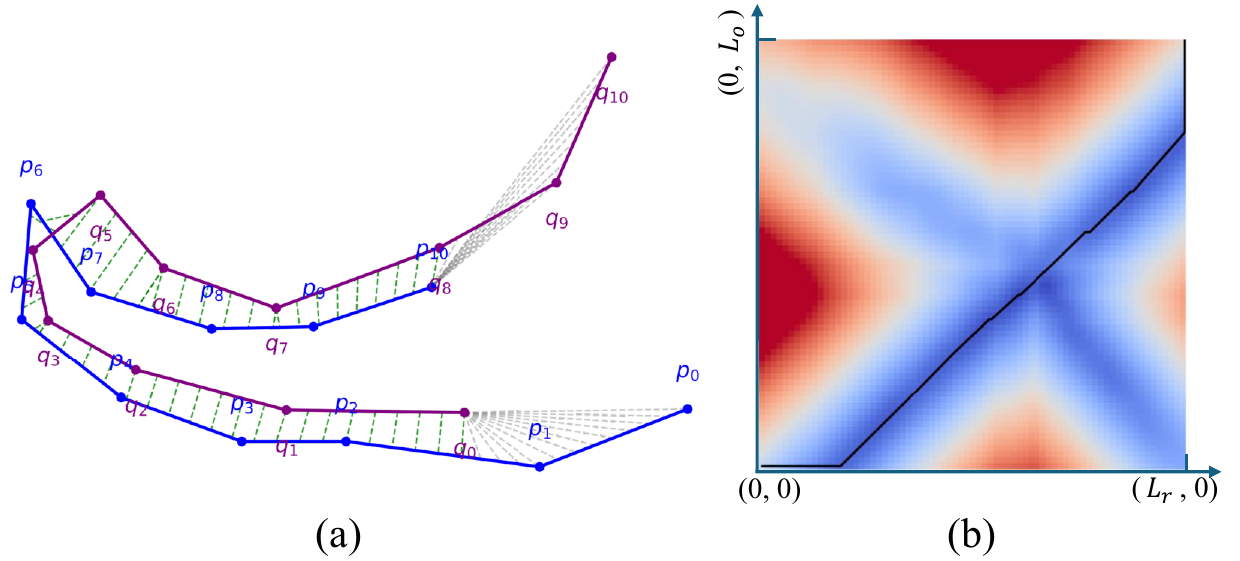}
    \caption{(a) An example of Continuous Dynamic Time Warping (CDTW) between two curves. The dashed lines indicate the optimal alignment, or warping path, found by CDTW. (b) The corresponding fine-grained cost field, with black lines denoting the optimal warping path.}
    \label{fig:cdtw_curves}
\end{figure}

Continuous Dynamic Time Warping (CDTW) is a distance measure between two curves that is robust to differences in their sampling rates and local stretching. It generalizes the well-known discrete Dynamic Time Warping (DTW) by finding an optimal continuous alignment, or ``warping path'', between the curves, as illustrated in Fig.~\ref{fig:cdtw_curves}.

\subsubsection{Problem Formulation}
Given two continuous curves, $c_r: [0, L_r] \to \mathbb{R}^d$ and $c_o: [0, L_o] \to \mathbb{R}^d$, where $d$ is the dimension of the curve, parameterized by their respective arc lengths, the joint parameter space $\mathcal{P} = [0, L_r] \times [0, L_o]$ is the product of the two curve parameter spaces, as shown in Fig.~\ref{fig:cdtw_param_space}.

The goal of CDTW is to find a path through their joint parameter space that minimizes a total cost, defined as the line integral of the squared Euclidean distance between corresponding points on the two curves, as follows:

\begin{equation}
    \gamma(\pi) = \int_0^1 \| c_r(\pi(t)) - c_o(\pi(t)) \|^2_2 \cdot \|\pi'(t)\|_1 \, dt
\end{equation}
where a valid warping path $\pi: [0, 1] \to \mathcal{P}$ must be continuous and monotonic, starting at $\pi(0) = (0,0)$ and ending at $\pi(1) = (L_r, L_o)$. Intuitively, $\pi$ starts from the lower-left corner and ends at the upper-right corner of the parameter space, moving only in nondecreasing directions (right, up, or right-up), as shown in Fig.~\ref{fig:cdtw_curves}\textcolor{red}{b}.

\begin{figure}[t]
    \centering
    \subfloat[]{
        \includegraphics[width=0.75\linewidth]{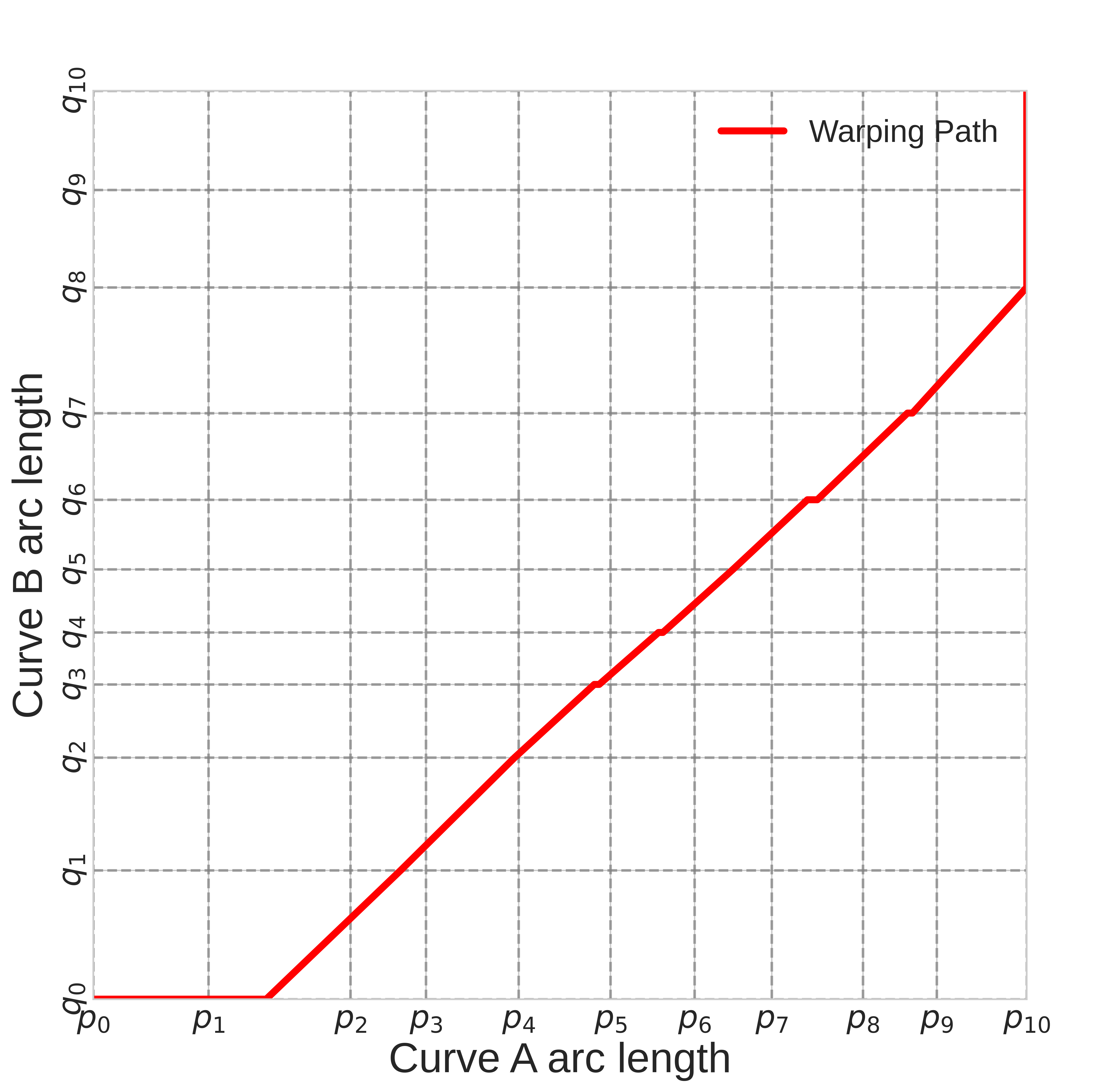}
        \label{fig:cdtw_param_space}
    }
    \\
    \subfloat[]{
        \includegraphics[width=\linewidth]{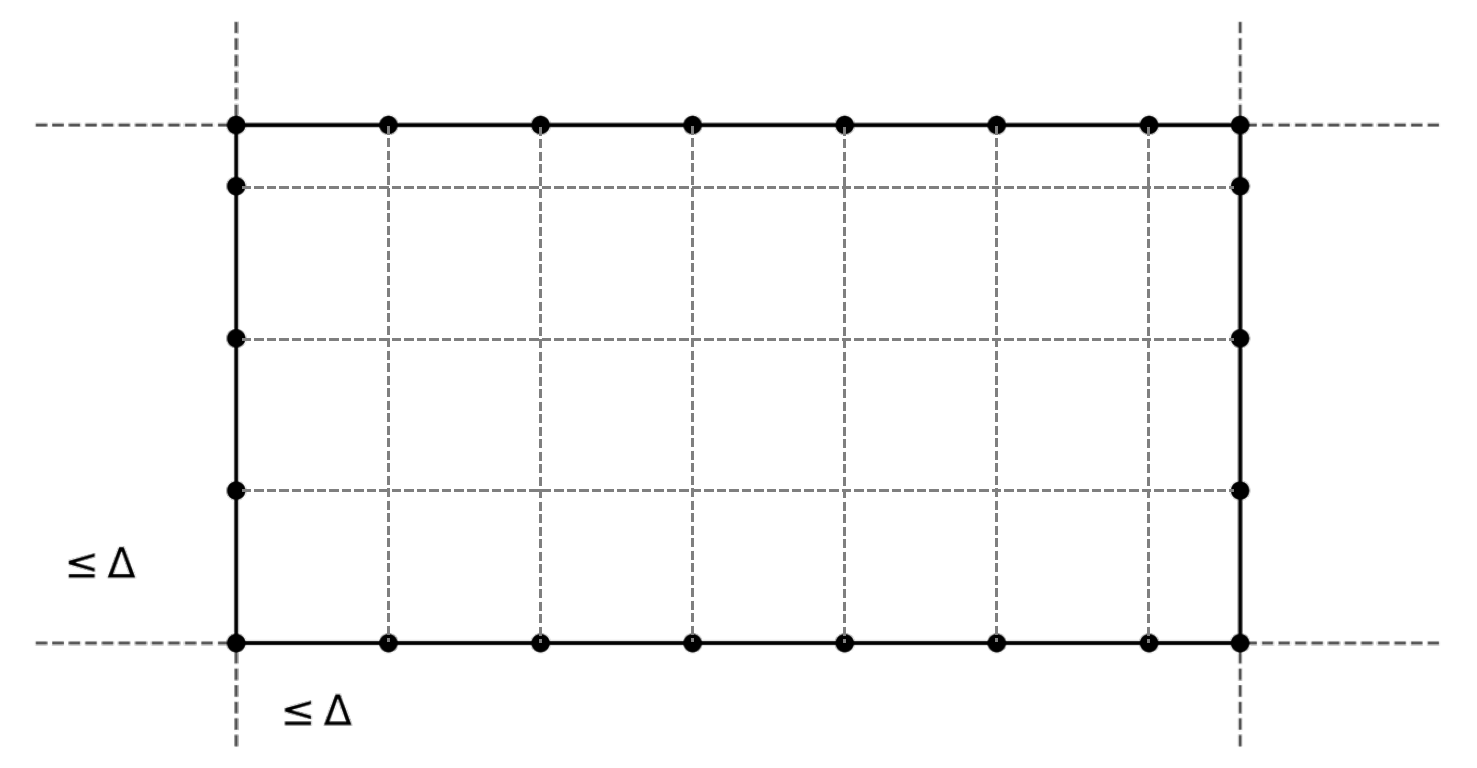}
        \label{fig:cdtw_cell}
    }
    \caption{Graph-based approximation of CDTW. (a) The joint parameter space is partitioned into a grid along the curve vertices $q_{0,1,\ldots, 10}$ and $p_{0,1,\ldots, 10}$ in $c_r$ and $c_o$ respectively, where $(q_0, p_0)$ correspond to $\pi(0)$ and $(q_{10}, p_{10})$ correspond to $\pi(1)$. (b) Within each cell, the cell is partitioned into a grid with a resolution $\Delta$ to obtain higher approximation accuracy, where any two adjacent vertices' CDTW can be computed analytically.}
    \label{fig:cdtw_approx}
\end{figure}

\subsubsection{Practical Implementation}
Practical implementations of CDTW rely on discretizing the continuous parameter space into a directed graph and finding the shortest path. The simplest approach partitions the parameter space into a fine-grained grid, computes the cost (Euclidean distance between curves) at each grid point to construct a cost field, as shown in Fig.~\ref{fig:cdtw_curves}\textcolor{red}{b}. Every vertex in the grid is connected to its right, up and right-up neighbors, and the cost of each edge is given by the average cost of its two incident vertices, which yields a directed graph. CDTW can then be approximated by finding the shortest path from $(0,0)$ to $(L_r, L_o)$ on this directed graph.

However, although conceptually simple, this method is also computationally expensive, because a fine-grained grid is required to ensure an accurate cost field, and the cost must be evaluated over the entire grid. Brankovic et al.~\cite{brankovic2020} proposed a more efficient approximation by partitioning the parameter space into a grid aligned with the curve vertices, as shown in Fig.~\ref{fig:cdtw_param_space}. The key observation is that, while the CDTW between two full curves is difficult to compute directly, the CDTW between two line segments can be computed analytically. This allows us to efficiently evaluate the CDTW between any two adjacent vertices in the grid of Fig.~\ref{fig:cdtw_param_space}. To obtain higher approximation accuracy, each grid cell is further partitioned with resolution $\Delta$, as shown in Fig.~\ref{fig:cdtw_cell}. Each subgrid again corresponds to a parameter subspace of two line segments. We then use a bidirectional Dijkstra algorithm to find the shortest path, where each edge cost is the analytical CDTW between two line segments and is computed on demand when the edge is explored, rather than precomputed for the entire grid.

\subsubsection{Optimal Substructure for Partial Matching}
\label{app:cdtw_optimal_substructure}
A key property of CDTW is that optimal warping paths exhibit optimal substructure. This means that any subpath of an optimal path is itself an optimal path between its endpoints. This property is crucial when two curves are partially overlapping, as it guarantees that we can find the optimal alignment for the overlapping segments of two partially observed curves by computing the CDTW on the full curves. The proof of Lemma~\ref{lemma:cdtw_substructure} follows \cite{klaren2020continuous}.

\begin{lemma}
    \label{lemma:cdtw_substructure}
    Consider a minimum-cost warping path $\pi$ connecting two points $s$ and $t$ in the parameter space. For any segment $\sigma \subseteq \pi$ bounded by intermediate points $s'$ and $t'$, the segment $\sigma$ necessarily achieves the minimum cost among all valid paths from $s'$ to $t'$.
\end{lemma}
    
\begin{proof}
    We establish this result via contradiction. Decompose $\pi$ into three consecutive segments: $\sigma_s$ (from $s$ to $s'$), $\sigma$ (from $s'$ to $t'$), and $\sigma_t$ (from $t'$ to $t$). By the additive nature of path costs, we have $\gamma(\pi) = \gamma(\sigma_s) + \gamma(\sigma) + \gamma(\sigma_t)$.
    
    Suppose $\sigma$ does not minimize the cost between $s'$ and $t'$. Then there must exist an alternative path $\sigma^*$ satisfying $\gamma(\sigma^*) < \gamma(\sigma)$. Replacing $\sigma$ with $\sigma^*$ yields a composite path $\pi^* = \sigma_s \circ \sigma^* \circ \sigma_t$ connecting $s$ to $t$, whose total cost equals $\gamma(\pi^*) = \gamma(\sigma_s) + \gamma(\sigma^*) + \gamma(\sigma_t)$.
    
    The strict inequality $\gamma(\sigma^*) < \gamma(\sigma)$ directly implies $\gamma(\pi^*) < \gamma(\pi)$, contradicting the optimality of $\pi$. Hence, $\sigma$ must be a minimum-cost path between $s'$ and $t'$.
\end{proof} 

In practice, we directly compute the CDTW between two curves, regardless of whether they are fully or only partially overlapping. We then check whether an endpoint of one curve is matched to a segment starting from an endpoint of the other curve (rather than matching exactly to that endpoint). If so, we find the nearest matching point on that segment, discard the portion from the endpoint to this nearest point, and keep only the remaining part, as shown by the gray dashed lines in Fig.~\ref{fig:cdtw_curves}\textcolor{red}{a}.

\subsection{Graduated Non-Convexity (GNC) Implementation}
\label{app:gnc}

This section provides the detailed implementation of the Graduated Non-Convexity (GNC) optimization used in \mainref{subsec:robust_estimation} for robust Chebyshev-curve parameter estimation. Our derivation follows \cite{yang2020graduated}.

\subsubsection{Black-Rangarajan Duality Formulation}
The robust optimization problem (Truncated Least Squares, TLS) in \maineqref{eq:tls_objective} is:
\begin{equation}
\mathbf{a}^* = \argmin_{\mathbf{a}} \sum_{j=1}^N \min\left( \|r_j(\mathbf{a})\|^2_2, c^2 \right)
\end{equation}
where $c^2$ is the truncation threshold and $r_j(\mathbf{a}) = p_j - \sum_{i=0}^n a_i \cos(i \arccos(t_j))$ is the linear residual function with respect to the Chebyshev coefficients $\mathbf{a}$.

It can be reformulated using the Black-Rangarajan duality principle \cite{black1996unification} to jointly optimize over curve parameters $\mathbf{a}$ and auxiliary weights $w_j \in [0,1]$:
\begin{equation}
\label{eq:gnc_dual}
\mathbf{a}^*, \hat{\mathbf{W}} = \argmin_{\mathbf{a}, \{w_j\}} \sum_{j=1}^N \left[ w_j \|r_j(\mathbf{a})\|^2_2 + \Phi_{\rho_\mu}(w_j) \right]
\end{equation}
where $\mathbf{W} = \text{diag}(w_1, \ldots, w_N)$ is the weight matrix, and $\Phi_{\rho_\mu}(w_j)$ is a penalty term derived from the robust loss function $\rho_\mu(\cdot)$. For the truncated least-squares (TLS) loss, the penalty is:
\begin{equation}
\Phi_{\rho_\mu}(w_j) = \frac{\mu(1-w_j)}{\mu+w_j}c^2
\end{equation}
where $\mu$ controls the sharpness of the robust loss. As $\mu$ increases from $0$ to $\infty$, the robust loss function $\rho_\mu(\cdot)$ transitions from a convex quadratic loss to the original truncated loss.

\subsubsection{Alternating Optimization Strategy}

Since Eq.~\ref{eq:gnc_dual} is non-convex in the joint variables $(\mathbf{a}, \{w_j\})$, we adopt an alternating optimization strategy that iteratively updates the parameters and weights:

\paragraph{Step 1: Parameter Update (Fixed Weights)} Given weights $\hat{w}_j^{(t-1)}$ from the previous iteration, solve the weighted least squares problem:
\begin{equation}
\label{eq:gnc_param_update}
\mathbf{a}^{(t)} = \argmin_{\mathbf{a}} \sum_{j=1}^N \hat{w}_j^{(t-1)} \|r_j(\mathbf{a})\|^2_2
\end{equation}
This is a standard weighted least-squares problem that can be solved efficiently in closed form, since the residual is quadratic in the Chebyshev coefficients $\mathbf{a}$. In practice, this step can be further accelerated using GPU parallelization.

\paragraph{Step 2: Weight Update (Fixed Parameters)} Given updated parameters $\mathbf{a}^{(t)}$, update each weight $w_j$ independently by minimizing:
\begin{equation}
\hat{w}_j^{(t)} = \argmin_{w_j \in [0,1]} \left[ w_j \|r_j(\mathbf{a}^{(t)})\|^2_2 + \Phi_{\rho_\mu}(w_j) \right]
\end{equation}
Taking the derivative with respect to $w_j$ and setting it to zero:
\begin{equation}
\frac{\partial}{\partial w_j} \left( w_j \|r_j(\mathbf{a}^{(t)})\|^2_2 + \Phi_{\rho_\mu}(w_j) \right) = 0
\end{equation}

The derivative of the TLS penalty function is:
\begin{equation}
\frac{\partial \Phi_{\rho_\mu}(w_j)}{\partial w_j} = \frac{\partial}{\partial w_j} \left( \frac{\mu(1-w_j)}{\mu+w_j} c^2 \right) = -c^2 \cdot \frac{\mu(\mu+1)}{(\mu+w_j)^2}
\end{equation}

Thus, the optimality condition becomes:
\begin{equation}
\|r_j(\mathbf{a}^{(t)})\|^2_2 - c^2 \cdot \frac{\mu(\mu+1)}{(\mu+w_j)^2} = 0
\end{equation}

Let $\bar{r}_j = \|r_j(\mathbf{a}^{(t)})\|^2_2$, then:
\begin{equation}
\bar{r}_j = c^2 \cdot \frac{\mu(\mu+1)}{(\mu+w_j)^2}
\end{equation}

Solving for $w_j$:
\begin{equation}
w_j = c \sqrt{\frac{\mu(\mu+1)}{\bar{r}_j}} - \mu
\end{equation}

However, since $w_j \in [0,1]$ and the penalty function is designed for convex--concave optimization, the solution must be clamped to the valid range. The unconstrained solution $w_j = c \sqrt{\frac{\mu(\mu+1)}{\bar{r}_j}} - \mu$ must satisfy:

- When $\bar{r}_j \geq \frac{\mu+1}{\mu} c^2$, the unconstrained solution gives $w_j \leq 0$, but since $w_j \geq 0$, it is clamped to $w_j = 0$ (outlier case).
- When $\bar{r}_j \leq \frac{\mu}{\mu+1} c^2$, the unconstrained solution gives $w_j \geq 1$, but since $w_j \leq 1$, it is clamped to $w_j = 1$ (inlier case).
- When $\frac{\mu}{\mu+1} c^2 < \bar{r}_j < \frac{\mu+1}{\mu} c^2$, the unconstrained solution gives $0 < w_j < 1$ (uncertain case).

This leads to the following closed-form weight update:
\begin{equation}
\label{eq:gnc_weight_update}
\hat{w}_j^{(t)} =
\begin{cases}
0 & \text{if } \bar{r}_j \geq \frac{\mu+1}{\mu}c^2 \quad \text{(outlier)} \\
c\sqrt{\frac{\mu(\mu+1)}{\bar{r}_j}} - \mu & \text{if } \frac{\mu}{\mu+1}c^2 < \bar{r}_j < \frac{\mu+1}{\mu}c^2 \quad \text{(uncertain)} \\
1 & \text{if } \bar{r}_j \leq \frac{\mu}{\mu+1}c^2 \quad \text{(inlier)}
\end{cases}
\end{equation}
where $\bar{r}_j = \|r_j(\mathbf{a}^{(t)})\|^2_2$ is the squared residual for observation $j$ given current parameters $\mathbf{a}^{(t)}$.

\subsubsection{Parameter Scheduling and Initialization}

The parameter $\mu$ controls the shape of the robust loss function. GNC gradually increases $\mu$ across iterations to transition from a convex (quadratic) loss to the non-convex truncated loss:

\paragraph{Initialization} At the first iteration $t=0$, set:
\begin{equation}
\mu^{(0)} = \frac{c^2}{2\max_j \|r_j(\mathbf{a}^{(0)})\|^2_2 + c^2}
\end{equation}
where $\mathbf{a}^{(0)}$ is an initial parameter estimate (e.g., from a simple least-squares fit). This ensures that the initial loss is nearly convex.

\paragraph{Progression Schedule} At each iteration $t$, update:
\begin{equation}
\mu^{(t)} = \kappa \cdot \mu^{(t-1)}
\end{equation}
where $\kappa > 1$ is a constant (typically $\kappa = 1.4$). This gradually sharpens the robust loss, allowing GNC to avoid poor local minima in early iterations while converging to an accurate solution.

\paragraph{Termination} The optimization terminates when either:
\begin{itemize}
    \item The gradient norm $\left\|\nabla_{\mathbf{a}} \sum_j \hat{w}_j^{(t)} \|r_j(\mathbf{a}^{(t)})\|^2_2\right\|$ falls below a threshold (e.g., $10^{-6}$), or
    \item The maximum number of iterations $N_{\text{iter}}$ (e.g., 50) is reached.
\end{itemize}

\subsection{Denoising Diffusion Probabilistic Models (DDPM)}
\label{app:ddpm_discrete}

The original Denoising Diffusion Probabilistic Model (DDPM) \cite{ho2020denoising} is formulated with discrete time steps. It consists of two Markov processes: a forward process and a reverse process.

\subsubsection{Forward Process} 
\label{app:forward_process}
The forward process gradually adds Gaussian noise to a data sample $\x_0$ over $T$ discrete time steps (e.g., $T=1000$). The transition at each step $t \in \{1, \dots, T\}$ is defined as a Markov chain with a fixed variance schedule $\{\beta_t\}_{t=1}^T$:
\begin{equation}
    q(\x_t|\x_{t-1}) := \mathcal{N}(\x_t; \sqrt{1-\beta_t} \x_{t-1}, \beta_t\mathbf{I})
\end{equation}
Letting $\alpha_t = 1 - \beta_t$ and $\bar{\alpha}_t = \prod_{s=1}^t \alpha_s$, a key property of this process is that we can sample $\x_t$ directly from $\x_0$ in a closed form:
\begin{equation}
    q(\x_t|\x_0) := \mathcal{N}(\x_t; \sqrt{\bar{\alpha}_t}\x_0, (1-\bar{\alpha}_t)\mathbf{I})
\end{equation}
This allows for efficient training, as we can sample $\x_t$ by the following formula
\begin{equation}
    \label{eq:direct_sampling}
    \x_t = \sqrt{\bar{\alpha}_t}\x_0 + \sqrt{1-\bar{\alpha}_t}\epsilon
\end{equation}
for any $t$, where $\epsilon \sim \mathcal{N}(0, \mathbf{I})$.

\subsubsection{Reverse Process} 
\label{app:reverse_process}
The reverse process is a learned generative process that starts from a standard Gaussian noise $\x_T \sim \mathcal{N}(\mathbf{0}, \mathbf{I})$ and denoises it back to a clean sample $\x_0$. This process is also a Markov chain defined as:
\begin{equation}
    p_\theta(\x_{t-1}|\x_t) := \mathcal{N}(\x_{t-1}; \mu_\theta(\x_t, t), \Sigma_\theta(\x_t, t))
\end{equation}
The goal of training is to learn the parameters $\theta$ of the mean $\mu_\theta$ and variance $\Sigma_\theta$. In DDPM~\cite{ho2020denoising}, the mean $\mu_\theta$ and variance $\Sigma_\theta$ are defined as:
\begin{equation}
    \mu_\theta(\x_t, t) = \frac{1}{\sqrt{\alpha_t}} \left( \x_t - \frac{1-\alpha_t}{\sqrt{1-\bar{\alpha}_t}} \epsilon_\theta(\x_t, t) \right)
    \label{eq:ddpm_mu}
\end{equation}
\begin{equation}
    \Sigma_\theta(\x_t, t) = \sigma_t^2 \mathbf{I} = \frac{1-\bar{\alpha}_{t-1}}{1-\bar{\alpha}_t} \beta_t \mathbf{I}
    \label{eq:ddpm_sigma}
\end{equation}
where $\epsilon_\theta(\x_t, t)$ is the predicted noise by the neural network.

While the standard DDPM reverse process requires $T=1000$ steps for high-quality generation, more efficient solvers have been developed to accelerate sampling. DDIM~\cite{song2020denoising} provides a deterministic sampling method that treats the diffusion process as an Ordinary Differential Equation (ODE), enabling high-quality generation with significantly fewer steps. DPM-Solver++~\cite{lu2022dpm++} further improves efficiency through high-order ODE solvers, achieving comparable quality with even fewer steps. Both solvers maintain the same learned model $\epsilon_\theta$ but use different numerical methods to solve the reverse process, significantly reducing computational cost during inference.

\subsection{Derivation of the DDPM Training Objective}
\label{app:ddpm_loss_derivation}

This section provides the theoretical justification for the DDPM training objective, which simplifies to minimizing the mean-squared error between the true and predicted noise. The derivation follows the principles in \cite{ho2020denoising, chan2024tutorial}.

The training goal is to maximize the log-likelihood of the data, $\log p_\theta(\x_0)$. Since this is intractable, we instead maximize a lower bound, the Evidence Lower Bound (ELBO), derived via Jensen's inequality:
\begin{align}
\begin{split}
\log p_\theta(\x_0) \geq \mathbb{E}_{q(\x_{1:T}|\x_0)}\left[\log \frac{p_\theta(\x_{0:T})}{q(\x_{1:T}|\x_0)}\right] = \mathcal{L}_{\text{ELBO}}
\end{split}
\end{align}
Using the chain rule of probability, this ELBO can be decomposed into three main components:
\begin{align}
\mathcal{L}_{\text{ELBO}} ={}& \underbrace{\mathbb{E}_q[\log p_\theta(\x_0|\x_1)]}_{\text{Reconstruction Term}} \nonumber \\
& \underbrace{- D_{KL}(q(\x_T|\x_0) || p(\x_T))}_{\text{Prior Matching Term}} \nonumber \\
& \underbrace{- \sum_{t=2}^T \mathbb{E}_q[D_{KL}(q(\x_{t-1}|\x_t, \x_0) || p_\theta(\x_{t-1}|\x_t))]}_{\text{Denoising Matching Terms}}
\label{eq:elbo_decomposed}
\end{align}
The three components can be interpreted as follows:
\begin{itemize}
    \item \textbf{Reconstruction Term ($L_0$)}: This measures how well the model can reconstruct the original data $\x_0$ from its first noised version $\x_1$. As $p_\theta(\x_0|\x_1)$ is a Gaussian with mean $\mu_\theta(\x_1, 1)$, maximizing this log-likelihood term is equivalent to minimizing the squared error $\|\x_0 - \mu_\theta(\x_1, 1)\|^2$. As we will see, this also simplifies to a noise-prediction objective, consistent with the other terms.
    \item \textbf{Prior Matching Term ($L_T$)}: This term encourages the distribution of the final latent variable after the forward process, $q(\x_T|\x_0)$, to match the standard Gaussian prior $p(\x_T)$. Since the forward process $q$ is fixed and has no learnable parameters, this term is treated as a constant during training and can be ignored for optimization.
    \item \textbf{Denoising Matching Terms ($L_{t-1}$ for $t \in [2, T]$)}: Forming the core of the optimization, this is a sum of KL divergences that measure the mismatch between the true posterior $q(\x_{t-1}|\x_t, \x_0)$ and our learned reverse step $p_\theta(\x_{t-1}|\x_t)$.
\end{itemize}
Our primary goal is to simplify the Denoising Matching and Reconstruction terms. To make the KL divergence tractable, we first need a closed-form expression for the true posterior $q(\x_{t-1}|\x_t, \x_0)$. Using Bayes' theorem, it can be shown that this posterior is also a Gaussian, $q(\x_{t-1}|\x_t, \x_0) = \mathcal{N}(\x_{t-1}; \tilde{\mu}_t, \tilde{\beta}_t\mathbf{I})$, where the mean $\tilde{\mu}_t$ and variance $\tilde{\beta}_t$ are given by:
\begin{align}
    \tilde{\mu}_t(\x_t, \x_0) &= \frac{\sqrt{\bar{\alpha}_{t-1}}\beta_t}{1-\bar{\alpha}_t}\x_0 + \frac{\sqrt{\alpha_t}(1-\bar{\alpha}_{t-1})}{1-\bar{\alpha}_t}\x_t \label{eq:posterior_mean} \\
    \tilde{\beta}_t &= \frac{1-\bar{\alpha}_{t-1}}{1-\bar{\alpha}_t}\beta_t \label{eq:posterior_var}
\end{align}
The detailed derivation of this posterior distribution using Bayes' theorem can be found in \cite{chan2024tutorial}. Since the true posterior is Gaussian, we set our model's reverse process $p_\theta(\x_{t-1}|\x_t)$ to also be a Gaussian, with a learnable mean $\mu_\theta(\x_t, t)$ and a fixed variance $\Sigma_\theta(\x_t, t) = \tilde{\beta}_t\mathbf{I}$. With matching variances, the KL divergence simplifies to the squared difference between the means:
\begin{equation}
\begin{split}
L_{t-1} &= D_{KL}(q(\x_{t-1}|\x_t, \x_0) || p_\theta(\x_{t-1}|\x_t)) \\
&= \mathbb{E}_{\x_0, \epsilon} \left[ \frac{1}{2\tilde{\beta}_t} \|\tilde{\mu}_t(\x_t, \x_0) - \mu_\theta(\x_t, t)\|^2 \right]
\end{split}
\end{equation}
The key insight is how to parameterize $\mu_\theta$. By substituting $\x_0 = (\x_t - \sqrt{1-\bar{\alpha}_t}\epsilon)/\sqrt{\bar{\alpha}_t}$ (recalling Eq.~(\ref{eq:direct_sampling})) into the equation for $\tilde{\mu}_t$, we can express the true posterior mean as a function of the noisy image $\x_t$ and the true noise $\epsilon$:
\begin{equation}
\tilde{\mu}_t(\x_t, \x_0) = \frac{1}{\sqrt{\alpha_t}}\left(\x_t - \frac{\beta_t}{\sqrt{1-\bar{\alpha}_t}}\epsilon\right)
\end{equation}
This motivates parameterizing our model's mean $\mu_\theta$ with the same functional form, but replacing the true noise $\epsilon$ with a prediction from a neural network $\epsilon_\theta(\x_t, t)$:
\begin{equation}
\mu_\theta(\x_t, t) = \frac{1}{\sqrt{\alpha_t}}\left(\x_t - \frac{\beta_t}{\sqrt{1-\bar{\alpha}_t}}\epsilon_\theta(\x_t, t)\right)
\label{eq:mu_theta_param_elbo}
\end{equation}
Plugging these two forms for the mean into the KL divergence term, the $\x_t$ terms cancel out, leaving a direct comparison between the true and predicted noise:
\begin{align}
\label{eq:mu_diff_squared}
\begin{split}
\|\tilde{\mu}_t - \mu_\theta\|^2 &= \left\| \frac{1}{\sqrt{\alpha_t}}\left(\x_t - \frac{\beta_t}{\sqrt{1-\bar{\alpha}_t}}\epsilon\right) \right. \\
&\qquad \left. - \frac{1}{\sqrt{\alpha_t}}\left(\x_t - \frac{\beta_t}{\sqrt{1-\bar{\alpha}_t}}\epsilon_\theta(\x_t, t)\right) \right\|^2 \\
&= \left\| \frac{\beta_t}{\sqrt{\alpha_t(1-\bar{\alpha}_t)}} (\epsilon_\theta(\x_t, t) - \epsilon) \right\|^2 \\
&= \frac{\beta_t^2}{\alpha_t(1-\bar{\alpha}_t)} \|\epsilon - \epsilon_\theta(\x_t, t)\|^2
\end{split}
\end{align}
This shows that all trainable parts of the ELBO ($L_0, L_1, \dots, L_{T-1}$) can be unified under the same objective of minimizing the noise prediction error. The authors of \cite{ho2020denoising} found empirically that ignoring the complex, time-dependent weighting term that arises from this derivation and optimizing a simplified, unweighted objective yields better sample quality. This leads to the final simplified training objective for DDPM, which sums the noise prediction error over all timesteps:
\begin{equation}
    L_{\text{simple}}(\theta) = \mathbb{E}_{t \sim [1,T], \x_0, \epsilon} \left[ \left\| \epsilon - \epsilon_\theta(\sqrt{\bar{\alpha}_t}\x_0 + \sqrt{1-\bar{\alpha}_t}\epsilon, t) \right\|^2 \right]
    \label{eq:ddpm_loss}
\end{equation}

This objective is intuitive: the model serves as a noise predictor trained across various noise levels, with the goal of reversing the forward noising process. Alternatively, predicting the clean sample $\mathbf{x}_0$ is also viable, as it can be converted back to the noise prediction via Eq.~(\ref{eq:direct_sampling}). The choice between these two prediction targets typically depends on experimental results.

\begin{figure}[t]
    \centering
    \includegraphics[width=0.618\linewidth]{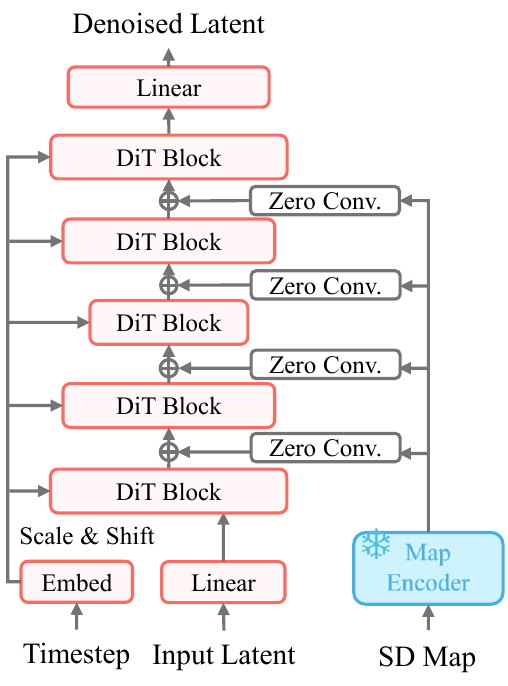}
    \caption{Noise-prediction network: the diffused latent $\x_t$ is mapped to the predicted noise. Timestep and SD-map conditioning are injected into the DiT via adaLN and zero-initialized convolution layers, respectively, enabling controlled denoising.}
    \label{fig:gen_model}
\end{figure}

\subsection{Diffusion Model Architecture}
\label{app:diffusion_architecture}

This section details the latent diffusion architecture used in \mainref{sec:diff_mapping}. It follows a VAE + DiT design: a VAE first encodes the raster HD map into a compact latent representation, and a conditional DiT-based denoiser operates in latent space to predict the diffusion noise.


\paragraph{Map Variational-Autoencoder} In nuScenes experiments, we use raster maps of resolution $256 \times 256$ with three semantic channels (pedestrian crossings, lane dividers, and road boundaries). The encoder $E$ and decoder $D$ are implemented as an AutoencoderKL~\cite{diffusers} with five downsampling and five upsampling blocks: \texttt{DownEncoderBlock2D} / \texttt{UpDecoderBlock2D} with block widths $(32, 32, 64, 128, 256)$ and one layer per block. This produces a latent tensor $\x_0 \in \mathbb{R}^{16 \times 16 \times 16}$ with $16$ channels and spatial size $16 \times 16$ (i.e., $H' = W' = 16$, $C' = 16$). To ensure that the latent space possesses regularity to facilitate the learning of the diffusion model while maintaining the reconstruction quality as much as possible, we adopt a relatively low KL loss weight of 1e-8. 

\paragraph{DiT Denoiser for nuScenes} The architecture of the proposed noise-prediction network is illustrated in Fig.~\ref{fig:gen_model}. The latent denoiser performs denoising operations directly on the $16 \times 16 \times 16$ latent tensor. We maintain the latent spatial resolution unchanged and adopt a patch size of $1$, such that each latent pixel corresponds to a single token. The DiT is configured with the embedding dimension of $384$, the depth of $16$, and $6$ self-attention heads. The architecture follows the DiT with skip connections: the encoder and decoder each comprise $8$ DiT blocks, where every block integrates self-attention with rotary position embedding \cite{su2024roformer} and feed-forward layers, along with skip connections that mirror the U-Net structure (omitted in the figure for clarity). A learnable timestep embedding is processed by an MLP and injected into every DiT block through adaLN-style modulation, thereby enabling time-dependent noise prediction. For SD-Map conditioning, the SD map $c_{\text{sd}}$ is rasterized and aligned with the latent. The SD map is first encoded into a latent feature grid and subsequently projected through zero-initialized linear layers before being fused into every DiT block, similar in spirit to ControlNet~\cite{zhang2023adding}.


\paragraph{DiT Denoiser for Proprietary Dataset} For the larger proprietary dataset, we use a higher-capacity DiT denoiser with the same latent resolution ($16 \times 16 \times 16$) but a deeper and wider backbone to better exploit the larger training set. The DiT uses an embedding dimension of $512$, depth $24$, and $8$ self-attention heads. The architecture follows the same design as the nuScenes variant. For SD-Map conditioning, the SD map is first transformed into bird's-eye-view features using LSS~\cite{philion2020lift}, then encoded into the denoiser via adaLN~\cite{peebles2023}. Training uses 110{,}000 steps with batch size 64 on 32 NVIDIA RTX 3090 GPUs.

\paragraph{Diffusion Scheduler} During training and inference we use a discrete DDPM~\cite{ho2020denoising} scheduler with $T = 1000$ timesteps and a scaled linear $\beta$ schedule from $\beta_{\text{start}} = 8.5 \times 10^{-4}$ to $\beta_{\text{end}} = 1.2 \times 10^{-2}$. For nuScenes, we use $\epsilon$-prediction (noise prediction) as the training objective, while for the proprietary dataset, we directly predict the clean sample $\x_0$ (sample prediction).

\begin{figure}[t]
    \centering
    \includegraphics[width=0.99\linewidth]{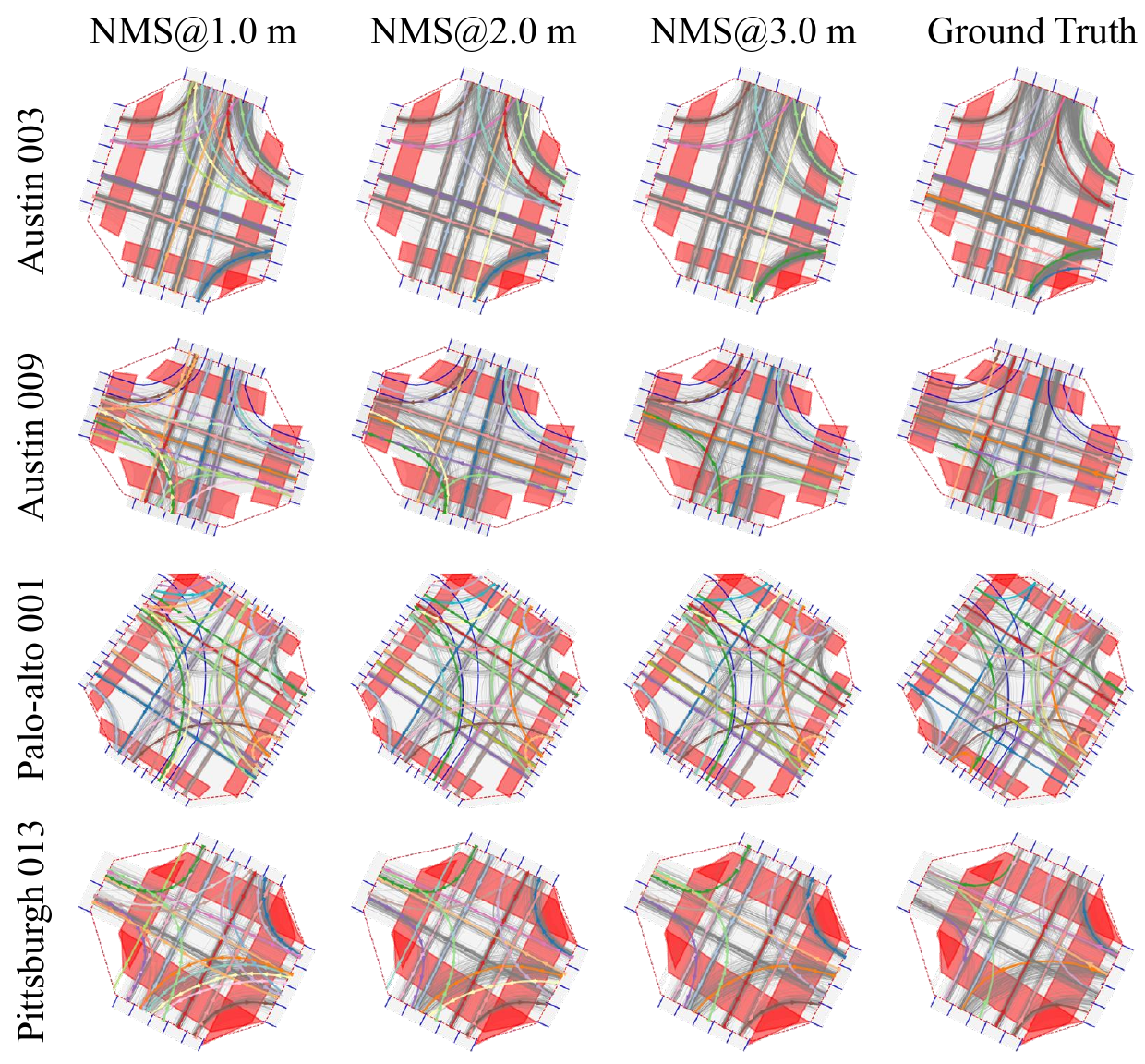}
    \caption{Comparison of NMS thresholds at 1\,m, 2\,m, and 3\,m. For each setting, we visualize four sample clusters together with the human-labeled ground-truth centerlines to illustrate the influence of the NMS threshold.}
    \label{fig:topo_nms}
\end{figure}

\subsection{Observation Aggregation Equivalence}
\label{app:obs_aggregation}

This section proves that the $\ell_2$ loss between the generated map $\mathcal{M}_s$ and the probability map $\mathcal{M}_{prob}$ is mathematically equivalent (up to a constant) to averaging the losses over all individual observations for each pixel and then summing across all pixels and classes.

Consider a scene with multiple crowdsourced observations. For each pixel $p \in \Omega$ and semantic class $c \in \{1,\dots,C\}$, we have $N_p$ binary observations $m_i^c(p) \in \{0,1\}$. The total loss across all pixels and classes is:
\begin{equation}
\label{eq:total_loss_individual}
\mathcal{L}_{\text{total}} = \sum_{p \in \Omega} \sum_{c=1}^C \frac{1}{N_p} \sum_{i=1}^{N_p} (m_i^c(p) - \mathcal{M}_s^c(p))^2
\end{equation}

Without loss of generality, we focus on a single pixel-class pair $(p,c)$ and derive the loss for that pair. The total loss becomes
\begin{equation}
\label{eq:single_loss_individual}
\mathcal{L}_{p,c} = \frac{1}{N_p} \sum_{i=1}^{N_p} (m_i^c(p) - \mathcal{M}_s^c(p))^2
\end{equation}

Expanding this expression:
\begin{align}
\mathcal{L}_{p,c} &= \frac{1}{N_p} \sum_{i=1}^{N_p} \left[ m_i^c(p)^2 - 2m_i^c(p)\mathcal{M}_s^c(p) + \mathcal{M}_s^c(p)^2 \right] \nonumber \\
&= \frac{1}{N_p} \sum_{i=1}^{N_p} m_i^c(p)^2 - \frac{2\mathcal{M}_s^c(p)}{N_p} \sum_{i=1}^{N_p} m_i^c(p) + \mathcal{M}_s^c(p)^2
\end{align}

Since $m_i^c(p) \in \{0,1\}$, we have $m_i^c(p)^2 = m_i^c(p)$. Define the average observation as $\bar{m}^c(p) = \frac{1}{N_p} \sum_{i=1}^{N_p} m_i^c(p)$, which equals the probability map value $\mathcal{M}_{prob}^c(p)$. This gives:
\begin{equation}
\begin{aligned}
    \mathcal{L}_{p,c} &= \bar{m}^c(p) - 2\mathcal{M}_s^c(p) \bar{m}^c(p) + \mathcal{M}_s^c(p)^2 \\
    &= (\bar{m}^c(p) - \mathcal{M}_s^c(p))^2 + \bar{m}^c(p)(1 - \bar{m}^c(p))
\end{aligned}
\end{equation}

The last term $\bar{m}^c(p)(1 - \bar{m}^c(p))$ is constant with respect to $\mathcal{M}_s^c(p)$, so:
\begin{equation}
\label{eq:equivalence_single}
\mathcal{L}_{p,c} = (\bar{m}^c(p) - \mathcal{M}_s^c(p))^2 + \text{const}
\end{equation}

For the probability map loss, we compute the squared error against $\bar{m}^c(p) = \mathcal{M}_{prob}^c(p)$:
\begin{equation}
\label{eq:prob_map_loss}
\mathcal{L}_{\text{prob},p,c} = (\mathcal{M}_{prob}^c(p) - \mathcal{M}_s^c(p))^2 = (\bar{m}^c(p) - \mathcal{M}_s^c(p))^2
\end{equation}

Comparing Eq.~\ref{eq:equivalence_single} and Eq.~\ref{eq:prob_map_loss}, the loss against the probability map equals the loss against individual observations up to a constant that depends only on the observations and not on the generated map $\mathcal{M}_s$.

Since this holds for each pixel-class pair, the total loss against the probability map:
\begin{equation}
\mathcal{L}_{\text{prob}} = \sum_{p \in \Omega} \sum_{c=1}^C (\mathcal{M}_{prob}^c(p) - \mathcal{M}_s^c(p))^2 = \sum_{p \in \Omega} \sum_{c=1}^C (\bar{m}^c(p) - \mathcal{M}_s^c(p))^2
\end{equation}
equals the total loss against individual observations (Eq.~\ref{eq:total_loss_individual}) up to constants that are independent of $\mathcal{M}_s$.

\subsection{Ablation of NMS Threshold}
\label{app:topo_nms}
In \mainref{subsec:traj_cluster}, to avoid manually selecting the number of clusters $k$, we first initialize with a larger value and then apply Non-Maximum Suppression (NMS). Specifically, we sort medoids by their cluster sizes (number of trajectories in each cluster) in descending order, and remove any medoid whose trajectory distance $d_{\tau}$ to a medoid of a larger cluster (appearing earlier in the sorted list) falls below the threshold $\tau_{\text{nms}}$. 

Fig.~\ref{fig:topo_nms} compares the results under NMS thresholds of 1\,m, 2\,m, and 3\,m together with the corresponding human-labeled ground-truth centerlines. The main differences appear at large turns, where vehicle trajectories become more dispersed: although drivers tend to start from similar entry points, they exhibit diverse exit points and follow trajectories with different curvatures. With a small NMS threshold (e.g., 1\,m), multiple reference centerlines corresponding to a single GT centerline are preserved, capturing different driving styles. In contrast, a large threshold (e.g., 3\,m) suppresses these variations, favoring a single centerline that reflects the maximum consensus among trajectories.

\subsection{More Qualitative Results on Proprietary Dataset}
\label{app:proprietary}

This section provides additional qualitative results on the proprietary dataset, complementing the examples shown in \mainref{subsec:self_proprietary_datasets}.

\begin{figure*}[t]
    \centering
    \begin{subfigure}[b]{0.95\linewidth}
        \centering
        \includegraphics[width=\linewidth]{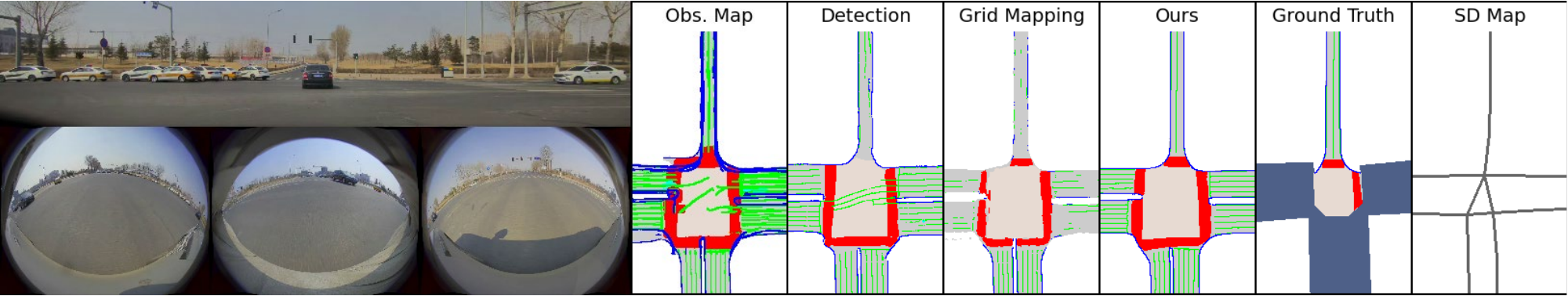}
    \end{subfigure}
    \hfill
    \begin{subfigure}[b]{0.95\linewidth}
        \centering
        \includegraphics[width=\linewidth]{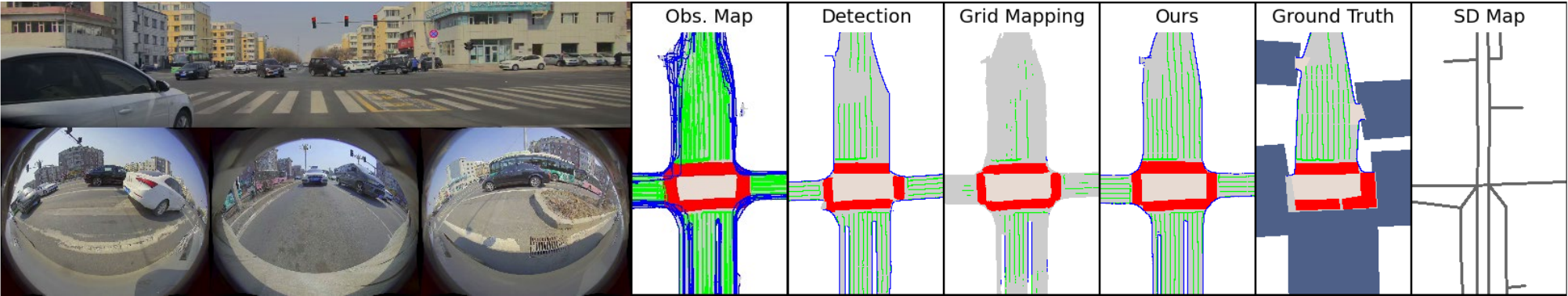}
    \end{subfigure}
    
    \vspace{0.5em}
    
    \begin{subfigure}[b]{0.95\linewidth}
        \centering
        \includegraphics[width=\linewidth]{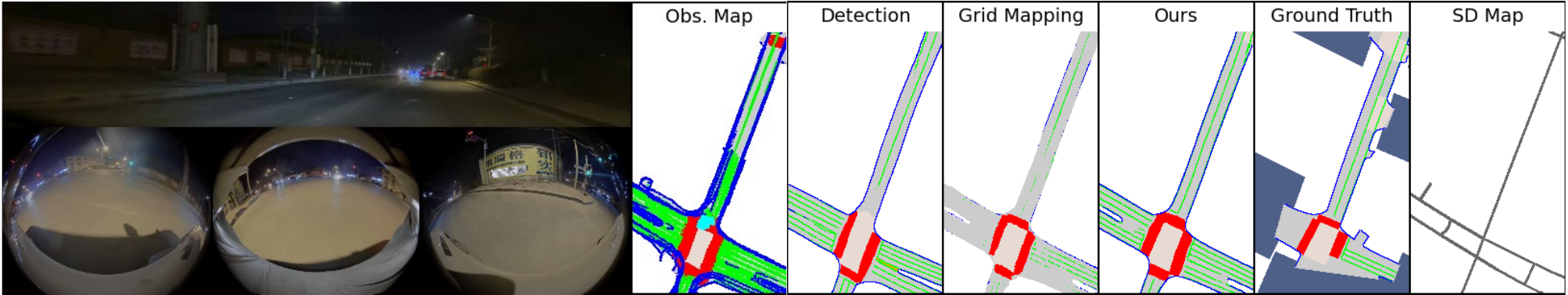}
    \end{subfigure}
    \hfill
    \begin{subfigure}[b]{0.95\linewidth}
        \centering
        \includegraphics[width=\linewidth]{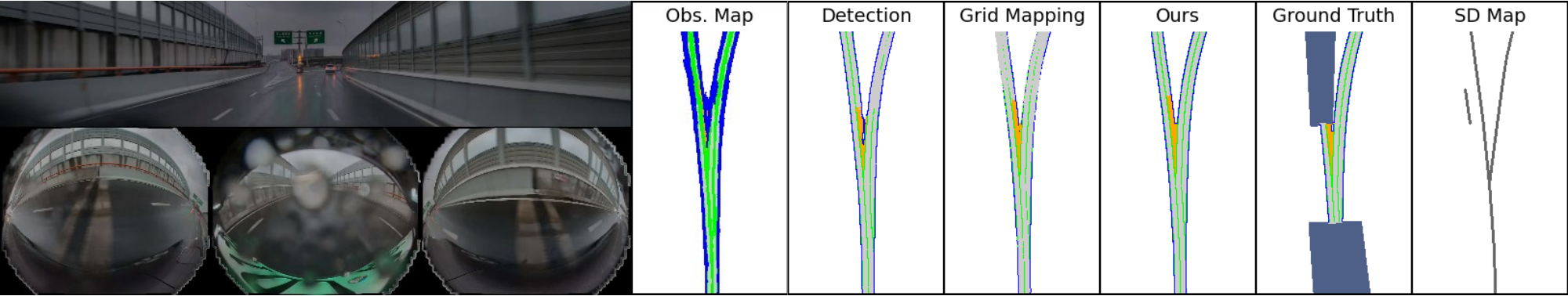}
    \end{subfigure}
    
    \vspace{0.5em}
    
    \begin{subfigure}[b]{0.95\linewidth}
        \centering
        \includegraphics[width=\linewidth]{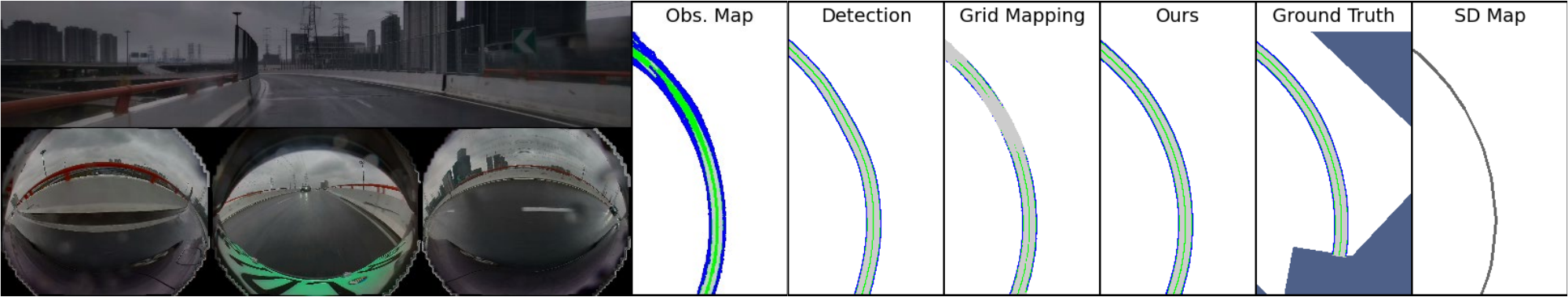}
    \end{subfigure}
    \hfill
    \begin{subfigure}[b]{0.95\linewidth}
        \centering
        \includegraphics[width=\linewidth]{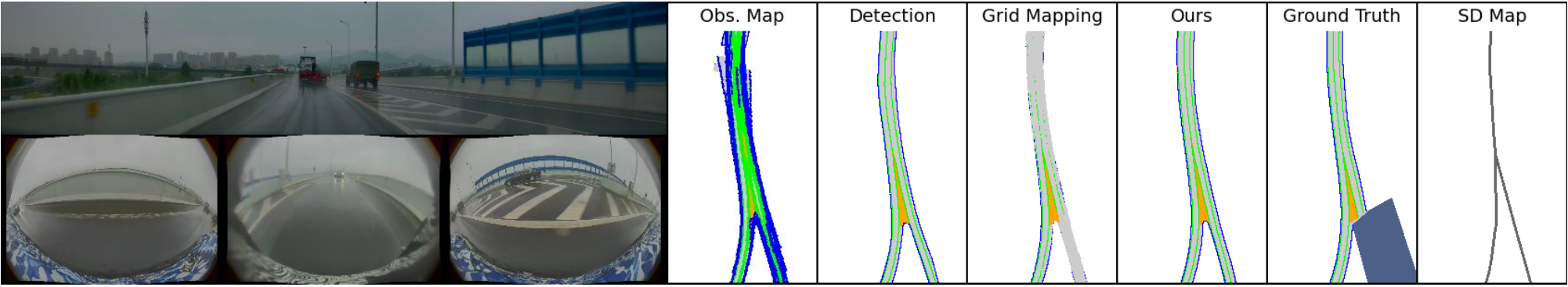}
    \end{subfigure}
    \caption{Additional qualitative results of crowdsourced semantic mapping on the proprietary dataset. \textcolor[RGB]{78,96,136}{Dark blue} areas represent unannotated pixels, which are excluded from evaluation.}
    \label{fig:proprietary_examples}
\end{figure*}

\end{document}